\documentclass{article}
\usepackage{iclr2027_conference,times}

\usepackage{amsmath,amsfonts,bm}

\def\eqref#1{equation~\ref{#1}}

\def\1{\bm{1}}

\DeclareMathAlphabet{\mathsfit}{\encodingdefault}{\sfdefault}{m}{sl}
\SetMathAlphabet{\mathsfit}{bold}{\encodingdefault}{\sfdefault}{bx}{n}

\usepackage{amsthm}
\usepackage{hyperref}
\usepackage{url}
\usepackage{graphicx}   
\usepackage{booktabs}   
\usepackage{amsmath}    
\usepackage{amssymb}    
\newtheorem{odproposition}{Proposition}
\newtheorem{odlemma}{Lemma}
\newtheorem{odcorollary}{Corollary}
\usepackage{multirow}   
\usepackage{array}      
\usepackage{makecell}   
\usepackage{xcolor}
\usepackage{adjustbox}
\usepackage{float}
\usepackage{algorithm}
\usepackage{algpseudocode}

\usepackage{makecell}
\usepackage{multirow}
\usepackage{bigstrut}
\usepackage{colortbl}

\definecolor{ODGroupOne}{HTML}{E4EFF9}
\definecolor{ODGroupTwo}{HTML}{FFF7DE}
\definecolor{ODGroupThree}{HTML}{FAEAEE}
\definecolor{ODMixGroupNew}{HTML}{E8F3E5}
\definecolor{ODOursHeader}{HTML}{D6EAEB}
\definecolor{ODOursBody}{HTML}{DAF5FB}
\definecolor{ODOursSummary}{HTML}{E5F2F3}
\definecolor{ODNeutralHeader}{HTML}{F0F1F3}
\definecolor{ODSummary}{HTML}{F5F6F7}
\definecolor{ODRankBest}{HTML}{D6EAEB}
\definecolor{ODRule}{HTML}{58616B}

\definecolor{ODOverviewOurs}{HTML}{DAF5FB}
\definecolor{ODOverviewScenario}{HTML}{EFEFEF}
\newlength{\ODOverviewDataWidth}

\usepackage{wrapfig}
\usepackage{capt-of}
\usepackage{needspace}
\definecolor{ODMixHalfOurs}{HTML}{DAF5FB}
\newlength{\ODMixHalfDataWidth}

\title{Online Automated Algorithm Design with Large Language Models}

\author{
Zhiyao Zhang$^{1}$\thanks{Equal Contribution.}\quad
Yichen Li$^{1}$\footnotemark[1]\quad
Xingyu Wu$^{1}$\thanks{Corresponding Author.}\quad
Liang Feng$^{2}$\quad
Kay Chen Tan$^{1}$\\
$^{1}$Department of Data Science and Artificial Intelligence\\
The Hong Kong Polytechnic University, Hong Kong SAR 999077, China\\
$^{2}$College of Computer Science, Chongqing University,
Chongqing 400044, China\\
\texttt{\{zhiyzhang, xingy.wu\}@polyu.edu.hk}
}

\iclrfinalcopy 
\begin{document}

\maketitle
\ificlrfinal
\lhead{Under review}
\fi

\begin{abstract}
Large language models (LLMs) enable automated algorithm design (AAD) through reasoning and code synthesis. However, most existing LLM-based AAD methods separate algorithm design from target optimization, deploying a fixed design even as the optimization state evolves. Conventional adaptive optimizers can respond to such changes, but their adjustments remain confined to predefined parameters, operators, or strategies. To address these limitations, we introduce \textbf{online LLM-based AAD}, a novel optimization paradigm that treats the algorithm itself as a state-dependent decision variable. At each stage, LLM agents synthesize an algorithm with new behavior logic from the current optimization state. Executing the generated algorithm advances the search and provides feedback for subsequent designs, coupling algorithm design with target optimization without requiring a separate offline algorithm pretraining stage. To implement this paradigm, we propose \textbf{OnDesign}, a multi-agent framework that reconciles competing design perspectives to synthesize executable algorithms and uses execution feedback to refine how runtime evidence is interpreted for subsequent designs. We evaluate OnDesign across two mainstream black-box optimization paradigms on three scenarios: Bayesian optimization, evolutionary continuous optimization, and evolutionary mixed-variable optimization. Extensive experiments on six benchmark suites and one engineering problem across multiple problem dimensions demonstrate superior overall performance over conventional optimizers and offline LLM-based AAD methods.
\end{abstract}
\section{Introduction}

Designing an effective optimization algorithm requires substantial expertise in constructing search strategies and decision rules~\citep{guo2025configx}. Recent advances in large language models (LLMs) have opened a new avenue for \emph{Automated Algorithm Design} (AAD), enabling algorithms to be generated and refined through natural-language reasoning and code synthesis~\citep{romera2024mathematical}. Most existing LLM-based AAD methods, however, follow an offline design paradigm: candidate algorithms are first generated and evaluated on training optimization tasks, and a selected algorithm is then deployed to the target optimization task~\citep{zheng2026llm}. Although LLMs greatly expand \emph{what} can be designed and \emph{how} designs are explored, algorithm design remains temporally separated from target-task execution. This separation has two important consequences. First, candidate evaluation may require a separate objective-evaluation budget, introducing additional costs when objective evaluations or repeated optimization runs are expensive~\citep{yu2026exploring}.
Second, the deployed algorithm remains fixed rather than being redesigned in response to the evolving optimization state~\citep{antipov2026switching}. In this offline workflow, LLMs change how algorithms are designed, but design still precedes the target optimization run~\citep{shi2026generalizable}.

This limitation particularly  matters because optimization is inherently dynamic~\citep{guo2024deep}. As observations accumulate, the optimization state evolves, and the search behavior needed to make progress may change with it~\citep{yu2026autopso}. Conventional adaptive optimizers respond by adjusting parameters~\citep{hansen2001completely,zhang2009jade}, selecting operators~\citep{xie2025co}, or switching among strategies~\citep{ngo2026adaptive}, but these choices remain confined to a design space specified in advance~\citep{guo2026designx}. When the search requires mechanisms beyond this predefined space, such adaptations may be insufficient to sustain progress~\citep{wu2025learning}.


\begin{figure}[!t]
\begin{center}
\includegraphics[width=1\linewidth]{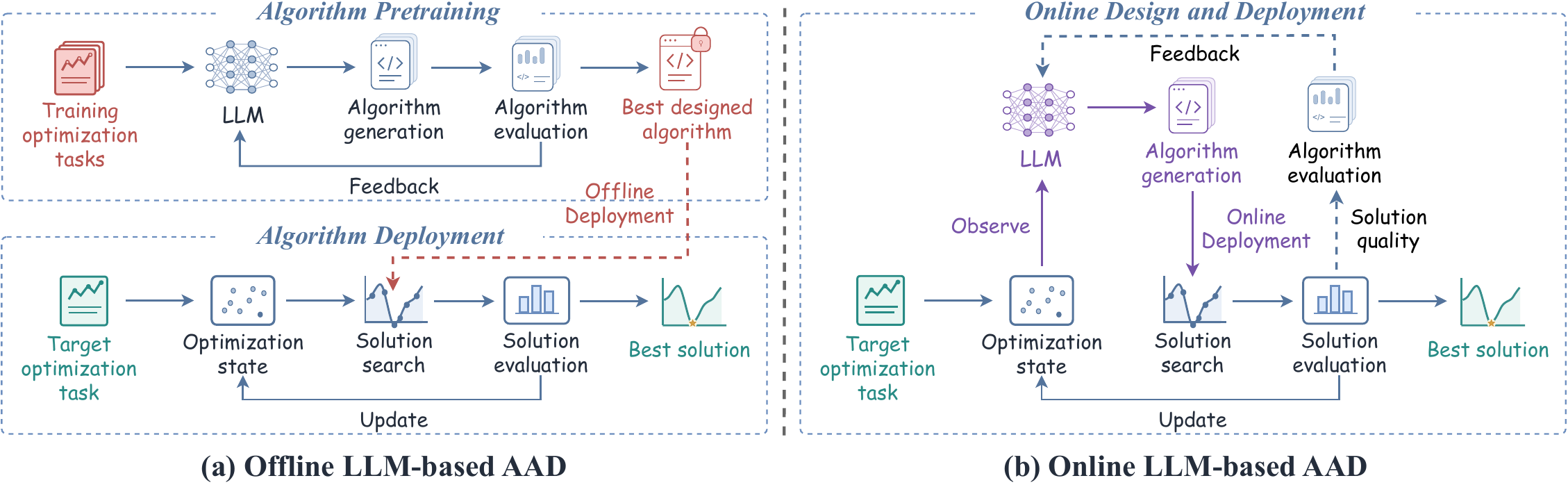}
\end{center}
\caption{Offline vs. online LLM-based AAD.
(a) Offline algorithm design and evaluation form a
pre-deployment loop, followed by target execution.
(b) Online algorithm design and target execution form a
coupled loop, with optimization states and execution gains
informing subsequent designs.}
\label{fig:offline_online}
\end{figure}

We therefore introduce {\emph{Online LLM-based AAD}}, which brings LLM-based algorithm design into the target optimization process. As shown in Figure~\ref{fig:offline_online}, instead of using LLMs to produce a single algorithm before execution, online LLM-based AAD repeatedly synthesizes an algorithm with state-conditioned behavior logic from the current optimization state. The LLMs can introduce new search mechanisms, rather than merely adjusting parameters or selecting among predefined operators or strategies. The algorithm itself thus becomes a state-dependent decision variable that evolves alongside the search. This interaction can be viewed as a meta-level sequential decision process: at each generation, the current optimization state conditions how LLMs design an algorithm; the resulting algorithm is then executed on the target task, where its execution both advances optimization and yields feedback that shapes the next algorithm design decision. Thus,
offline LLM-based AAD asks {what algorithm should be designed before optimization}, whereas online LLM-based AAD asks {what algorithm should LLMs generate next, given the search so far}.

To instantiate online LLM-based AAD, we propose OnDesign, a multi-agent framework that addresses two challenges: identifying runtime observations relevant to the next design and translating that evidence into executable algorithm logics. OnDesign comprises three interacting modules: (1) {State Analysis}, using a State Analysis Guideline (SAG) to organize heterogeneous runtime observations into a structured, evidence-grounded state report; (2) {Multi-Agent Deliberation}, bringing together role-specialized Proposers with competing search priorities and an Arbiter that reconciles their proposals to synthesize an algorithm from the state report and algorithm design history; and (3) {SAG Evolution}, periodically updating the SAG using performance feedback, adapting how runtime evidence is interpreted in subsequent state reports. Execution of the generated algorithm advances optimization and supplies feedback for subsequent designs and SAG updates.  Feedback therefore informs both what algorithm to generate and how to interpret the evidence used to generate it. We instantiate OnDesign within two major black-box optimization paradigms: Bayesian optimization~\citep{jones1998efficient} and evolutionary computation~\citep{10509608}. For the former, OnDesign generates acquisition functions; for the latter, it generates offspring-generation operators for continuous and mixed-variable optimization. Our main contributions are as follows:

\begin{itemize}
\item We formulate {Online LLM-based AAD}, extending algorithm design from offline discovery to state-dependent program synthesis during optimization. Under this formulation, the algorithm itself becomes a dynamic decision variable, synthesized by an LLM from the current optimization state and algorithm design history.

\item We propose {OnDesign}, a multi-agent framework that integrates State Analysis, Multi-Agent Deliberation, and SAG Evolution. Performance feedback informs both subsequent algorithm synthesis and updates to the guideline used to interpret runtime evidence.

\item We conduct extensive experiments in Bayesian optimization for expensive black-box problems~\citep{11311712} and in evolutionary computation for continuous and mixed-variable problems~\citep{9464165}. The results demonstrate the empirical superiority of OnDesign over conventional optimizers and existing offline LLM-based AAD methods across all three scenarios.
\end{itemize}

\section{Online LLM-Based AAD}
\label{sec:oaad}

We formalize online LLM-based AAD by casting target optimization and the meta-level algorithm-design process as two coupled sequential decision processes.

\paragraph{A sequential-decision view of optimization.}
We view target optimization as a sequential decision process in which each search action changes the information available for subsequent decisions~\cite{chen2022learning}. At stage $t$, the optimization state $s_t \in \mathcal{S}$ summarizes evaluated solutions, recent progress, the remaining evaluation budget, and other relevant information. Based on this state, the algorithm produces a search action $a_t \in \mathcal{A}$, specifying the candidate solutions to evaluate next.

Evaluating these candidates and performing the associated optimization updates produces the next state $s_{t+1}$, which informs the subsequent search action. Search actions are therefore coupled over time: each action affects both the progress of the search and the information on which later actions are based. Assuming that $s_t$ sufficiently summarizes the relevant history, this process admits the Markovian representation~\cite{sutton1998reinforcement}
\begin{equation}
    P(s_{t+1}\mid s_0,\ldots,s_t,a_0,\ldots,a_t)
    =
    P(s_{t+1}\mid s_t,a_t),
\end{equation}
where $P$ denotes the state-transition distribution.

The deployed algorithm determines how search actions are produced from the optimization state. Under a fixed algorithm $\pi$, the search action is $a_t=\operatorname{Execute}(\pi,s_t)$, where $\pi$ is represented as an executable program and $\operatorname{Execute}$ denotes the task-specific procedure for applying it. The resulting optimization trajectory can be summarized as
\begin{equation}
    s_t
    \xrightarrow{\operatorname{Execute}(\pi,\cdot)}
    a_t
    \xrightarrow{P}
    s_{t+1}
    \xrightarrow{\operatorname{Execute}(\pi,\cdot)}
    a_{t+1}
    \xrightarrow{P}
    s_{t+2}
    \longrightarrow \cdots.
\end{equation}
Each transition under $P$ includes solution evaluation and the associated optimization updates. The optimization state and search actions evolve over successive stages, while the deployed algorithm $\pi$ remains unchanged.

\paragraph{A sequential-decision view of online LLM-based AAD.}
Online LLM-based AAD makes the algorithm itself a state-dependent decision variable in the meta-level algorithm-design process. At each optimization stage $t$, LLM agents synthesize an algorithm $\pi_t$ from the current design state:
\begin{equation}
    \pi_t
    \sim
    \mathcal{D}_{\mathrm{LLM}}
    (\,\cdot\mid z_t,\iota\,),
    \qquad
    \pi_t \in \Pi,
\end{equation}
where $\mathcal{D}_{\mathrm{LLM}}$ denotes the algorithm-design policy instantiated by LLM agents, $z_t \in \mathcal{Z}$ is the design state of the meta-level process, and $\iota$ specifies the task instructions and program interface. The algorithm space $\Pi$ consists of executable algorithms compatible with this interface. The design state $z_t=(s_t,h_t)$ combines the current optimization state $s_t$ with the accumulated algorithm design history $h_t$.

Applying the generated algorithm produces the search action $a_t=\operatorname{Execute}(\pi_t,s_t)$. Evaluating its candidate solutions and performing the associated optimization updates produces $s_{t+1}\sim P(\,\cdot\mid s_t,a_t\,)$. The execution gain $r_t=\mathcal{R}(s_t,\pi_t,s_{t+1})$ quantifies progress in terms of the target objective, such as improvement in the best observed objective value.

The algorithm design history is updated as $h_{t+1}=h_t\oplus(\pi_t,r_t)$, where $\oplus$ appends the generated algorithm and its execution gain. The next design state is $z_{t+1}=(s_{t+1},h_{t+1})$. Historical optimization states are not included in the history entries; the current optimization state is supplied separately through $s_t$.

The optimization and algorithm-design processes together form the closed loop
\begin{equation}
    z_t
    \xrightarrow{\mathcal{D}_{\mathrm{LLM}}(\cdot\mid\iota)}
    \pi_t
    \xrightarrow{\operatorname{Execute}}
    a_t
    \xrightarrow{P,\mathcal{R}}
    (s_{t+1},r_t)
    \longrightarrow
    z_{t+1}.
\end{equation}
The updated design state conditions algorithm synthesis at the next optimization stage. The two processes are therefore coupled: algorithm design determines how the search proceeds, while target optimization supplies the state information and performance feedback for subsequent algorithm designs. This closed loop seeks to improve the best observed objective value within a fixed fitness-evaluation budget $FE_{\max}$. Algorithm design and target optimization take place within the same run, without requiring a separate offline algorithm pretraining stage.


\begin{figure}[!t]
\begin{center}
\includegraphics[width=1\linewidth]{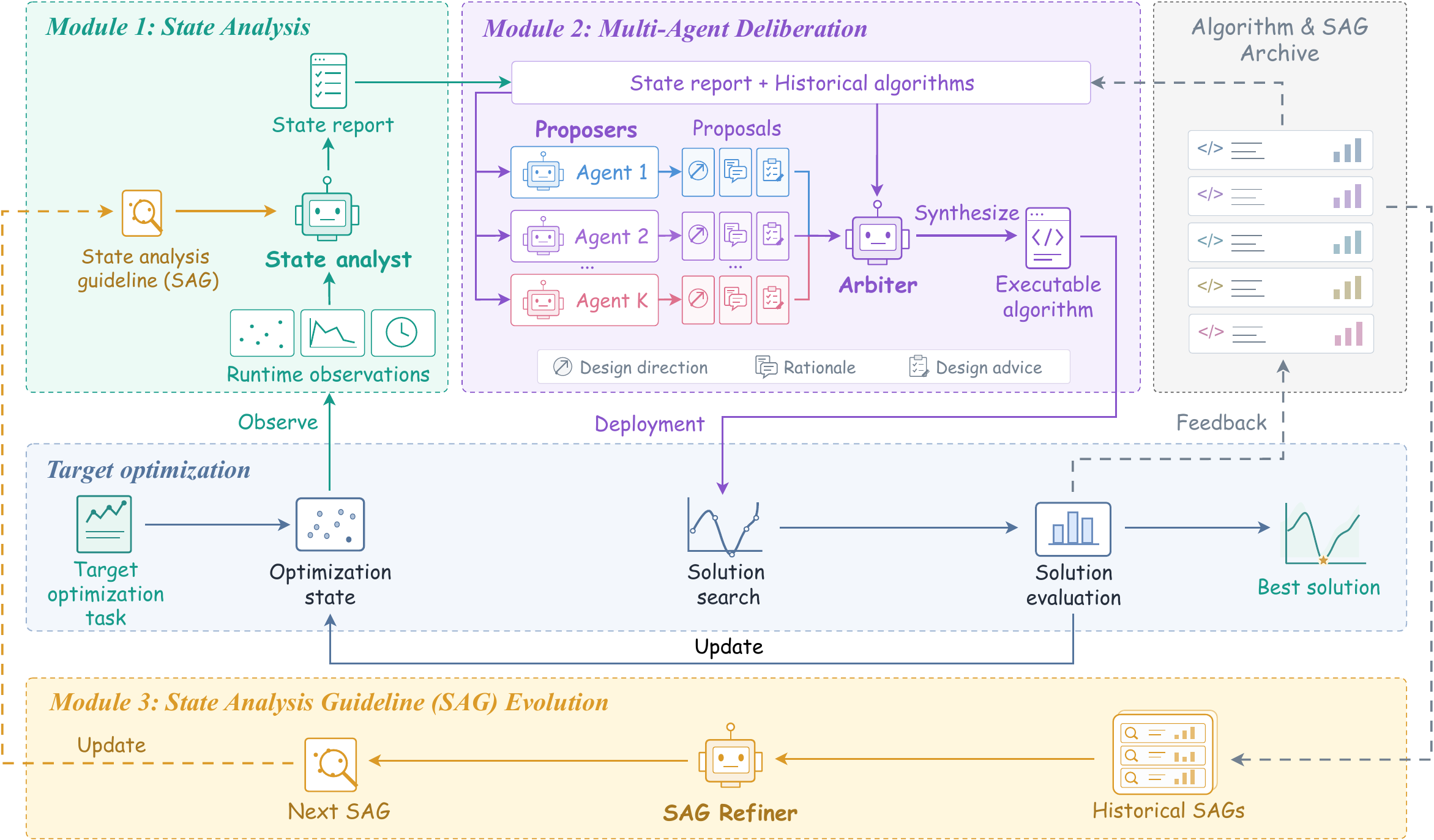}
\end{center}
\caption{{Overview of OnDesign.}
{\textbf{Module 1 (State Analysis)}}:  \emph{State Analyst} organizes runtime evidence according to the current State Analysis Guideline (SAG), producing the report used for algorithm design.
{\textbf{Module 2 (Multi-Agent Deliberation)}}: informed by this report and previously generated algorithms, \emph{Proposers} contribute distinct design directions, rationales, and design advices;  \emph{Arbiter} reconciles their proposals into an executable algorithm.
{\textbf{Module 3 (SAG Evolution)}}:  \emph{SAG Refiner} updates the guideline using archived SAG–score pairs.
The deployed algorithm generates performance feedback that is archived to support subsequent algorithm design and SAG evolution.}
\label{fig:ondesign}
\end{figure}

\section{OnDesign}
\label{sec:framework}

To realize online LLM-based AAD, we propose {OnDesign}, a multi-agent instantiation of the design policy $\mathcal{D}_{\mathrm{LLM}}$ introduced in Sec.~\ref{sec:oaad}. Its core mechanism is state-conditioned program synthesis: at each design stage, LLM agents characterize the current search, reason over the resulting state report and algorithm design history, and generate algorithms whose performance feedback informs subsequent algorithm designs. As shown in Fig.~\ref{fig:ondesign}, OnDesign realizes this loop through three modules: 1) {State Analysis}, in which \emph{State Analyst} constructs a state report guided by the current SAG; 2) {Multi-Agent Deliberation}, in which \emph{Proposers} develop complementary design proposals and \emph{Arbiter} produces an executable program; and 3) {SAG Evolution}, in which \emph{SAG Refiner} adapts the guideline for subsequent state analysis.

\paragraph{State Analysis.}
Let $o_t$ denote the runtime observations available for characterizing the optimization state $s_t$ at design stage $t$, and let $\phi_t$ denote the current State Analysis Guideline (SAG). An optimization run exposes heterogeneous signals, including optimization progress, solution distributions, model characteristics, search-space coverage, the remaining budget, and other aspects of the search. The importance of these signals can change as the search evolves. The SAG provides an explicit analytical focus by specifying which evidence to emphasize and how to structure its interpretation. Guided by $\phi_t$, {\emph{State Analyst}}, an LLM agent denoted by $\mathcal{P}$, identifies informative signals and interprets their relationships to characterize the current search dynamics, synthesizing the resulting evidence into a structured state report: $q_t = \mathcal{P}(o_t;\phi_t)$. The report $q_t$ characterizes $s_t$ and can include implications for search behavior. The SAG directs this analysis, while Multi-Agent Deliberation determines the final design direction and synthesizes the algorithm.

OnDesign instantiates the general design state $(s_t,h_t)$ introduced in Sec.~\ref{sec:oaad} through the state report $q_t$ and the algorithm archive $\mathbb{H}_t=\{(\pi_j,r_j)\}_{0\le j<t}$. The report summarizes the current optimization state $s_t$, while the archive records previous algorithms and their execution gains, thereby implementing the historical component $h_t$. Algorithm synthesis therefore uses the report-based design state $z_t=(q_t,\mathbb{H}_t)$. Appendix~\ref{app:theory} formalizes the role of this representation by relating its decision-relevant information to algorithm-value assessment and the final performance bound.

\paragraph{Multi-Agent Deliberation.}
The state report $q_t$ characterizes the current optimization state, but does not directly determine which algorithmic changes will improve subsequent optimization performance. Translating this evidence into executable algorithmic logic requires comparing alternative design directions and weighing competing search priorities. Multi-Agent Deliberation module makes these alternatives explicit through role-specialized proposals and examines their supporting evidence and trade-offs before synthesizing an executable algorithm. To support this synthesis, the module uses the design state $z_t$, combining current-state evidence from $q_t$ with prior design outcomes from the algorithm archive $\mathbb{H}_t$.

Given these inputs, $K$ role-specialized LLM agents, called \emph{Proposers} and denoted by $\mathcal{L}_k$, develop design proposals under the task instructions and program interface specified by $\iota$. Their application-dependent role instructions encourage distinct, potentially competing search priorities:
\begin{equation}
\begin{aligned}
    p_t^{(k)}
    &= \mathcal{L}_k(q_t,\mathbb{H}_t;\iota),
    \qquad k=1,\ldots,K,\\
    p_t^{(k)}
    &= \left(v_t^{(k)},e_t^{(k)},d_t^{(k)}\right).
\end{aligned}
\end{equation}
Each proposal comprises three elements: 1) \emph{Design Direction} $v_t^{(k)}$, guiding algorithm design toward desired search behaviors, such as favoring exploration or exploitation; 2) \emph{Rationale} $e_t^{(k)}$, justifying this direction with evidence from the state report and prior design outcomes; and 3) \emph{Design Advice} $d_t^{(k)}$, suggesting concrete mechanisms to realize the intended behavior.

\emph{Arbiter}, an LLM agent denoted by $\mathcal{G}$, evaluates the proposals against current-state evidence and prior design outcomes, reconciles competing priorities, and synthesizes an executable algorithm:
\begin{equation}
    \pi_t
    =
    \mathcal{G}
    \left(
        q_t,\mathbb{H}_t,\{p_t^{(k)}\}_{k=1}^{K};\iota
    \right)
    \in \Pi.
\end{equation}
\emph{Arbiter} integrates supported elements across proposals into coherent algorithmic logic rather than simply selecting one proposal. The resulting program $\pi_t$ is the algorithm generated in the meta-level design process; applying it to $s_t$ produces the search action $a_t=\operatorname{Execute}(\pi_t,s_t)$. Evaluating the corresponding candidate solutions and performing the associated optimization updates yields the execution gain $r_t$, which is recorded for subsequent designs through
$\mathbb{H}_{t+1}=\mathbb{H}_t\cup\{(\pi_t,r_t)\}$. The same execution gain $r_t$ also contributes to the score of the SAG $\phi_t$ that guided the construction of $q_t$, providing indirect feedback on SAG quality.

\paragraph{SAG Evolution.}
As search conditions evolve, an earlier SAG may become less informative.
The SAG Evolution module adapts this analytical focus through
\emph{SAG Refiner}, an LLM agent denoted by $\mathcal{E}$.
After each execution of $\pi_t$, the gain $r_t$ is credited to
$\phi_t$, which guided the construction of the state report $q_t$
used to design $\pi_t$. Each SAG is scored by the arithmetic mean
of all gains attributed to it, and the resulting SAG--score pairs
form the updated archive $\mathcal{H}^{\mathrm{SAG}}_{t+1}$.

Starting from an initial SAG $\phi_0$ and an empty archive,
the next-stage guideline is determined by
\begin{equation}
    \phi_{t+1}=
    \begin{cases}
        \mathcal{E}\!\left(\mathcal{H}^{\mathrm{SAG}}_{t+1}\right),
        & (t+1)\bmod L_{\mathrm{SAG}}=0,\\
        \phi_t,
        & \text{otherwise},
    \end{cases}
    \label{eq:sag_update}
\end{equation}
where $L_{\mathrm{SAG}}$ is a positive integer specifying the
update interval in completed design stages.
Refinement uses the archive after incorporating $r_t$, and the
resulting $\phi_{t+1}$ first guides the construction of $q_{t+1}$.
Between refinements, the guideline is retained while its feedback
continues to accumulate.

The SAG score provides a performance-based proxy for guideline
quality, rather than an isolated estimate of its causal effect.
Higher scores offer heuristic support for particular analytical
emphases, while lower scores motivate alternative ways of
emphasizing and interpreting runtime observations.
\emph{SAG Refiner} reasons over these SAG--score pairs to adapt
subsequent state analysis.

Thus, the algorithm archive $\mathbb{H}_{t+1}$ informs
\emph{what algorithm should be designed next}, while the SAG
archive $\mathcal{H}^{\mathrm{SAG}}_{t+1}$ informs
\emph{how runtime observations should be interpreted for that design}.
Detailed scoring and archive definitions are provided in
Appendix~\ref{app:ondesign_bo}.

\paragraph{Performance-gap analysis.}
To characterize the sources of performance loss in online algorithm design, we compare OnDesign with an ideal online designer operating under the same initial history, task interface, evaluation budget, and information constraints. The ideal designer may choose any budget-compatible algorithm in $\Pi$ at each stage. Let $V_0^*$ and $V_0^{\mathrm{OnDesign}}$ denote the expected terminal utilities of the ideal designer and OnDesign, respectively, with higher utility indicating better final solutions.

At stage $t$, the design-space error $\kappa_t$ measures loss from restricting the accessible algorithm family; the representation error $\varepsilon_t$ measures inaccuracies in assessing algorithm values from the state report and archive; and the synthesis error $\eta_t$ measures the shortfall relative to the best accessible design under these assessments.

\begin{odproposition}[Performance-gap bound for OnDesign]
\label{prop:online_design_gap}
For a bounded terminal utility and a finite design horizon $T$, under the representation and synthesis conditions specified in Appendix~\ref{app:theory}, the expected performance gap satisfies
\begin{equation}
0 \le V_0^*-V_0^{\mathrm{OnDesign}}
\le
\mathbb E_{\mathrm{OnDesign}}\!\left[
\sum_{t=0}^{T-1}
\bigl(\kappa_t+2\varepsilon_t+\eta_t\bigr)
\right].
\label{eq:online_design_gap}
\end{equation}
The expectation is over the histories and algorithms generated during the OnDesign run.
\end{odproposition}

The bound decomposes OnDesign's performance gap into errors arising from the accessible design space, state representation, and algorithm synthesis. These terms also clarify the complementary roles of its modules: State Analysis organizes decision-relevant evidence ($\varepsilon_t$); Multi-Agent Deliberation explores alternative designs and synthesizes algorithms ($\kappa_t,\eta_t$); and SAG Evolution adapts the analytical focus, targeting representation errors in subsequent stages. Their benefit depends on reducing accumulated design error. Definitions, proofs, and restricted-class comparisons appear in Appendix~\ref{app:theory}.

\raggedbottom
\section{Experiments}
\label{sec:experiments}

\subsection{Experimental Setup}
\label{sec:exp_setup}

We evaluate OnDesign in three scenarios: {Bayesian optimization}, {evolutionary continuous optimization}, and {evolutionary mixed-variable optimization}. The Bayesian optimization scenario represents computationally black-box settings, where limited evaluation budgets make sample-efficient search essential. Scenario-specific instantiations and prompt templates are detailed in Appendices~\ref{app:ondesign_bo}--\ref{app:prompt_templates}.

\paragraph{Baselines.}
We group the baselines into three categories.
(1) {\textbf{Fixed-logic optimizers}}: LogEI~\citep{ament2023unexpected}, LogPI~\citep{ament2023unexpected}, UCB~\citep{srinivas2009gaussian}, and JES~\citep{hvarfner2022joint} for Bayesian optimization; BSPGA~\citep{su2020non}, CLPSO~\citep{liang2006comprehensive}, and CoDE~\citep{wang2011differential} for evolutionary continuous optimization; and CoDE~\citep{wang2011differential}, MDE-IHS~\citep{liao2010two}, $\mathrm{DE}_{\mathrm{MV}}$~\citep{lin2018hybrid}, Sig-DE~\citep{yu2016stock}, and IMI-GWO~\citep{gupta2019efficient} for evolutionary mixed-variable optimization.
(2) {\textbf{Adaptive optimizers}}: SETUP-BO~\citep{vasconcelos2022self} for Bayesian optimization; SAHLPSO~\citep{tao2021self}, SHADE~\citep{tanabe2013success}, and AutoEP~\citep{xu2026autoep} for evolutionary continuous optimization; and SHADE~\citep{tanabe2013success} and $\mathrm{PSO}_{\mathrm{MV}}$~\citep{wang2021particle} for evolutionary mixed-variable optimization.
(3) {\textbf{LLM-designed algorithms}}: FunBO~\citep{aglietti2024funbo}, ATRBO-LCB~\citep{li2026llamea}, and TREvol~\citep{li2026llamea} for Bayesian optimization; ParEvo1~\citep{hu2025partition} and ParEvo2~\citep{hu2025partition} for evolutionary continuous optimization; and scenario-specific algorithms generated offline by EoH~\citep{liu2024eoh}, MEoH~\citep{yao2025meoh}, and ReEvo~\citep{ye2024reevo} for all three scenarios. Detailed baseline descriptions appear in Appendix~\ref{app:benchmarks_baselines}.

\paragraph{Benchmarks.}
Let $D$ denote the number of decision variables.
Bayesian optimization and evolutionary continuous optimization use BBOB~\citep{elhara2019coco} at $D\in\{10,30\}$, CEC2020~\citep{yue2019cec2020} and CEC2022~\citep{kumar2021cec2022} at $D\in\{10,20\}$, and CEC2026~\citep{Chen2026CEC} at $D=30$. Evolutionary mixed-variable optimization uses MV-BBOB at $D\in\{10,30\}$ and EOPCCV~\citep{9464165} at $D=10$, covering continuous--integer and continuous--categorical problems, respectively. The construction of MV-BBOB is described in Appendix~\ref{app:mv_bbob} and detailed benchmark settings are provided in Appendix~\ref{app:evaluation_protocol}.

\paragraph{Protocol and metrics.}
All methods are evaluated over ten independent runs with the same fitness-evaluation budget within each scenario. For Bayesian optimization, $FE_{\max}=2D+101$, including $2D+1$ Sobol initialization samples; the evolutionary optimization scenarios use a population size of 100 and $FE_{\max}=5{,}000$, including initialization. OnDesign uses $L_{\mathrm{SAG}}=10$ by default; see Appendix~\ref{app:sag_sensitivity} for an illustrative sensitivity analysis. Other baseline hyperparameters retain their implementation defaults. OnDesign and the three offline LLM-based AAD methods (EoH, MEoH, and ReEvo) use DeepSeek-V4~\citep{xu2026deepseek}. For each problem, we report the mean and standard deviation of the final best objective value. Function-wise ranks are averaged within each suite--dimension configuration, with lower objective values and ranks indicating better performance. For each baseline, $+/-/=$ counts denote the numbers of functions on which it performs better, worse, or equally well relative to OnDesign. Detailed settings are provided in Appendix~\ref{app:evaluation_protocol}.

\begin{table}[t]
  \centering
  \caption{Results on Bayesian and evolutionary continuous optimization scenarios. Entries report average rank ($\downarrow$), with $+/-/=$ counts for baselines. }
  \label{tab:bo-ec-main}
  \fontsize{8.5}{9.5}\selectfont
  \setlength{\tabcolsep}{0.95pt}
  \renewcommand{\arraystretch}{1.0}
  \arrayrulecolor{black}
  \setlength{\ODOverviewDataWidth}{\dimexpr(\linewidth-54pt-14\tabcolsep)/7\relax}
  \begin{tabular}{@{}>{\raggedright\arraybackslash}p{54pt}*{7}{>{\centering\arraybackslash}p{\ODOverviewDataWidth}}@{}}
    \toprule
    \multirow{2}{*}[-2.8pt]{\textbf{Method}} & \multicolumn{2}{c}{\textbf{BBOB}} & \multicolumn{2}{c}{\textbf{CEC2020}} & \multicolumn{2}{c}{\textbf{CEC2022}} & \textbf{CEC2026} \\
    \cmidrule(lr){2-3}\cmidrule(lr){4-5}\cmidrule(lr){6-7}\cmidrule(l){8-8}
     & 10D & 30D & 10D & 20D & 10D & 20D & 30D \\
    \hline
    \rowcolor{ODOverviewScenario}
     & \multicolumn{7}{c}{\textbf{\textit{Bayesian Optimization Scenario}}} \\
    \hline
    LogEI & 7.15~{\fontsize{7.5}{9}\selectfont(4/20/0)} & 7.17~{\fontsize{7.5}{9}\selectfont(1/22/1)} & 5.35~{\fontsize{7.5}{9}\selectfont(3/7/0)} & 7.30~{\fontsize{7.5}{9}\selectfont(4/6/0)} & 6.29~{\fontsize{7.5}{9}\selectfont(4/8/0)} & 6.88~{\fontsize{7.5}{9}\selectfont(3/9/0)} & 7.45~{\fontsize{7.5}{9}\selectfont(4/25/0)} \\
    LogPI & 7.38~{\fontsize{7.5}{9}\selectfont(3/21/0)} & 7.13~{\fontsize{7.5}{9}\selectfont(1/22/1)} & 7.55~{\fontsize{7.5}{9}\selectfont(3/7/0)} & 5.00~{\fontsize{7.5}{9}\selectfont(2/8/0)} & 6.21~{\fontsize{7.5}{9}\selectfont(4/8/0)} & 7.83~{\fontsize{7.5}{9}\selectfont(2/10/0)} & 8.02~{\fontsize{7.5}{9}\selectfont(4/25/0)} \\
    UCB & 7.81~{\fontsize{7.5}{9}\selectfont(3/21/0)} & 8.00~{\fontsize{7.5}{9}\selectfont(1/22/1)} & 7.40~{\fontsize{7.5}{9}\selectfont(2/8/0)} & 7.55~{\fontsize{7.5}{9}\selectfont(2/8/0)} & 6.13~{\fontsize{7.5}{9}\selectfont(5/7/0)} & 6.67~{\fontsize{7.5}{9}\selectfont(3/9/0)} & 7.72~{\fontsize{7.5}{9}\selectfont(4/25/0)} \\
    JES & 4.58~{\fontsize{7.5}{9}\selectfont(4/20/0)} & 6.44~{\fontsize{7.5}{9}\selectfont(2/21/1)} & 6.20~{\fontsize{7.5}{9}\selectfont(2/8/0)} & 6.70~{\fontsize{7.5}{9}\selectfont(2/8/0)} & 5.13~{\fontsize{7.5}{9}\selectfont(4/8/0)} & 6.17~{\fontsize{7.5}{9}\selectfont(3/9/0)} & 4.66~{\fontsize{7.5}{9}\selectfont(4/25/0)} \\
    \noalign{\vskip1.25pt\hbox to\linewidth{\color[gray]{0.45}\leaders\hbox{\rule{2pt}{0.3pt}\kern2pt}\hfill}\vskip1.25pt}
    SETUP-BO & 4.94~{\fontsize{7.5}{9}\selectfont(4/20/0)} & 8.17~{\fontsize{7.5}{9}\selectfont(1/22/1)} & 6.95~{\fontsize{7.5}{9}\selectfont(2/8/0)} & 5.20~{\fontsize{7.5}{9}\selectfont(3/7/0)} & 7.38~{\fontsize{7.5}{9}\selectfont(3/9/0)} & 5.46~{\fontsize{7.5}{9}\selectfont(4/8/0)} & 7.57~{\fontsize{7.5}{9}\selectfont(5/24/0)} \\
    \noalign{\vskip1.25pt\hbox to\linewidth{\color[gray]{0.45}\leaders\hbox{\rule{2pt}{0.3pt}\kern2pt}\hfill}\vskip1.25pt}
    FunBO & 6.75~{\fontsize{7.5}{9}\selectfont(3/20/1)} & 5.10~{\fontsize{7.5}{9}\selectfont(3/21/0)} & 6.90~{\fontsize{7.5}{9}\selectfont(2/8/0)} & 7.30~{\fontsize{7.5}{9}\selectfont(2/8/0)} & 9.08~{\fontsize{7.5}{9}\selectfont(3/9/0)} & 6.79~{\fontsize{7.5}{9}\selectfont(3/9/0)} & 6.81~{\fontsize{7.5}{9}\selectfont(6/23/0)} \\
    ATRBO-LCB & 8.52~{\fontsize{7.5}{9}\selectfont(3/20/1)} & 5.44~{\fontsize{7.5}{9}\selectfont(3/19/2)} & 8.45~{\fontsize{7.5}{9}\selectfont(2/7/1)} & 8.35~{\fontsize{7.5}{9}\selectfont(2/8/0)} & 8.38~{\fontsize{7.5}{9}\selectfont(3/8/1)} & 6.42~{\fontsize{7.5}{9}\selectfont(3/9/0)} & 7.66~{\fontsize{7.5}{9}\selectfont(4/25/0)} \\
    TREvol & 8.06~{\fontsize{7.5}{9}\selectfont(4/20/0)} & 8.54~{\fontsize{7.5}{9}\selectfont(1/22/1)} & 8.75~{\fontsize{7.5}{9}\selectfont(2/8/0)} & 6.35~{\fontsize{7.5}{9}\selectfont(3/7/0)} & 7.63~{\fontsize{7.5}{9}\selectfont(3/9/0)} & 6.13~{\fontsize{7.5}{9}\selectfont(3/9/0)} & 7.41~{\fontsize{7.5}{9}\selectfont(6/23/0)} \\
    EoH & 5.75~{\fontsize{7.5}{9}\selectfont(6/18/0)} & 8.06~{\fontsize{7.5}{9}\selectfont(1/22/1)} & 4.80~{\fontsize{7.5}{9}\selectfont(2/8/0)} & 6.95~{\fontsize{7.5}{9}\selectfont(3/7/0)} & 4.75~{\fontsize{7.5}{9}\selectfont(5/7/0)} & 8.33~{\fontsize{7.5}{9}\selectfont(2/10/0)} & 7.24~{\fontsize{7.5}{9}\selectfont(5/24/0)} \\
    MEoH & 8.75~{\fontsize{7.5}{9}\selectfont(2/21/1)} & 5.71~{\fontsize{7.5}{9}\selectfont(3/20/1)} & 7.00~{\fontsize{7.5}{9}\selectfont(3/7/0)} & 6.30~{\fontsize{7.5}{9}\selectfont(3/7/0)} & 6.25~{\fontsize{7.5}{9}\selectfont(5/7/0)} & 7.13~{\fontsize{7.5}{9}\selectfont(2/10/0)} & 5.69~{\fontsize{7.5}{9}\selectfont(5/24/0)} \\
    ReEvo & 5.38~{\fontsize{7.5}{9}\selectfont(9/15/0)} & 6.23~{\fontsize{7.5}{9}\selectfont(2/21/1)} & 5.00~{\fontsize{7.5}{9}\selectfont(3/7/0)} & 7.20~{\fontsize{7.5}{9}\selectfont(2/8/0)} & 6.25~{\fontsize{7.5}{9}\selectfont(3/9/0)} & 6.63~{\fontsize{7.5}{9}\selectfont(3/9/0)} & 4.93~{\fontsize{7.5}{9}\selectfont(6/22/1)} \\
    \hline
    \rowcolor{ODOverviewOurs}
    \textbf{OnDesign} & \textbf{2.94} & \textbf{2.02} & \textbf{3.65} & \textbf{3.80} & \textbf{4.54} & \textbf{3.58} & \textbf{2.84} \\
    \hline
    \rowcolor{ODOverviewScenario}
     & \multicolumn{7}{c}{\textbf{\textit{Evolutionary Continuous Optimization Scenario}}} \\
    \hline
    BSPGA & 6.73~{\fontsize{7.5}{9}\selectfont(4/19/1)} & 7.83~{\fontsize{7.5}{9}\selectfont(1/23/0)} & 6.90~{\fontsize{7.5}{9}\selectfont(1/9/0)} & 8.00~{\fontsize{7.5}{9}\selectfont(1/9/0)} & 6.21~{\fontsize{7.5}{9}\selectfont(0/9/3)} & 7.04~{\fontsize{7.5}{9}\selectfont(0/11/1)} & 7.45~{\fontsize{7.5}{9}\selectfont(1/28/0)} \\
    CLPSO & 9.54~{\fontsize{7.5}{9}\selectfont(2/22/0)} & 10.23~{\fontsize{7.5}{9}\selectfont(0/23/1)} & 9.50~{\fontsize{7.5}{9}\selectfont(0/10/0)} & 9.80~{\fontsize{7.5}{9}\selectfont(0/10/0)} & 9.33~{\fontsize{7.5}{9}\selectfont(0/11/1)} & 10.38~{\fontsize{7.5}{9}\selectfont(0/12/0)} & 9.86~{\fontsize{7.5}{9}\selectfont(0/29/0)} \\
    CoDE & 10.88~{\fontsize{7.5}{9}\selectfont(1/23/0)} & 11.19~{\fontsize{7.5}{9}\selectfont(0/23/1)} & 10.80~{\fontsize{7.5}{9}\selectfont(0/10/0)} & 11.10~{\fontsize{7.5}{9}\selectfont(0/10/0)} & 10.38~{\fontsize{7.5}{9}\selectfont(0/11/1)} & 10.42~{\fontsize{7.5}{9}\selectfont(0/12/0)} & 10.91~{\fontsize{7.5}{9}\selectfont(0/29/0)} \\
    \noalign{\vskip1.25pt\hbox to\linewidth{\color[gray]{0.45}\leaders\hbox{\rule{2pt}{0.3pt}\kern2pt}\hfill}\vskip1.25pt}
    SAHLPSO & 4.31~{\fontsize{7.5}{9}\selectfont(4/17/3)} & 6.02~{\fontsize{7.5}{9}\selectfont(3/21/0)} & 5.45~{\fontsize{7.5}{9}\selectfont(1/8/1)} & 4.45~{\fontsize{7.5}{9}\selectfont(3/7/0)} & 6.71~{\fontsize{7.5}{9}\selectfont(2/9/1)} & 5.92~{\fontsize{7.5}{9}\selectfont(3/8/1)} & 5.64~{\fontsize{7.5}{9}\selectfont(2/27/0)} \\
    SHADE & 5.96~{\fontsize{7.5}{9}\selectfont(3/19/2)} & 4.98~{\fontsize{7.5}{9}\selectfont(2/21/1)} & 6.45~{\fontsize{7.5}{9}\selectfont(0/10/0)} & 7.90~{\fontsize{7.5}{9}\selectfont(0/10/0)} & 5.75~{\fontsize{7.5}{9}\selectfont(0/9/3)} & 6.00~{\fontsize{7.5}{9}\selectfont(1/10/1)} & 6.36~{\fontsize{7.5}{9}\selectfont(0/29/0)} \\
    AutoEP & 4.56~{\fontsize{7.5}{9}\selectfont(5/17/2)} & 4.35~{\fontsize{7.5}{9}\selectfont(4/19/1)} & 4.15~{\fontsize{7.5}{9}\selectfont(1/8/1)} & 3.95~{\fontsize{7.5}{9}\selectfont(2/7/1)} & 5.00~{\fontsize{7.5}{9}\selectfont(0/8/4)} & 3.88~{\fontsize{7.5}{9}\selectfont(2/8/2)} & 3.53~{\fontsize{7.5}{9}\selectfont(5/23/1)} \\
    \noalign{\vskip1.25pt\hbox to\linewidth{\color[gray]{0.45}\leaders\hbox{\rule{2pt}{0.3pt}\kern2pt}\hfill}\vskip1.25pt}
    ParEvo1 & 10.63~{\fontsize{7.5}{9}\selectfont(2/22/0)} & 9.44~{\fontsize{7.5}{9}\selectfont(2/21/1)} & 11.50~{\fontsize{7.5}{9}\selectfont(0/10/0)} & 11.60~{\fontsize{7.5}{9}\selectfont(0/10/0)} & 10.96~{\fontsize{7.5}{9}\selectfont(0/11/1)} & 11.25~{\fontsize{7.5}{9}\selectfont(0/12/0)} & 11.79~{\fontsize{7.5}{9}\selectfont(0/29/0)} \\
    ParEvo2 & 5.90~{\fontsize{7.5}{9}\selectfont(5/17/2)} & 7.06~{\fontsize{7.5}{9}\selectfont(1/21/2)} & 3.95~{\fontsize{7.5}{9}\selectfont(1/8/1)} & 5.05~{\fontsize{7.5}{9}\selectfont(1/9/0)} & 3.88~{\fontsize{7.5}{9}\selectfont(2/7/3)} & 4.79~{\fontsize{7.5}{9}\selectfont(1/9/2)} & 5.22~{\fontsize{7.5}{9}\selectfont(2/27/0)} \\
    EoH & 8.98~{\fontsize{7.5}{9}\selectfont(2/22/0)} & 4.92~{\fontsize{7.5}{9}\selectfont(10/11/3)} & 9.85~{\fontsize{7.5}{9}\selectfont(0/10/0)} & 7.10~{\fontsize{7.5}{9}\selectfont(0/10/0)} & 8.92~{\fontsize{7.5}{9}\selectfont(0/10/2)} & 6.00~{\fontsize{7.5}{9}\selectfont(1/10/1)} & 6.22~{\fontsize{7.5}{9}\selectfont(5/24/0)} \\
    MEoH & 4.00~{\fontsize{7.5}{9}\selectfont(7/16/1)} & 4.75~{\fontsize{7.5}{9}\selectfont(2/20/2)} & 3.70~{\fontsize{7.5}{9}\selectfont(2/7/1)} & 3.75~{\fontsize{7.5}{9}\selectfont(1/9/0)} & 4.38~{\fontsize{7.5}{9}\selectfont(3/7/2)} & 5.25~{\fontsize{7.5}{9}\selectfont(2/9/1)} & 4.84~{\fontsize{7.5}{9}\selectfont(1/28/0)} \\
    ReEvo & 3.46~{\fontsize{7.5}{9}\selectfont(8/14/2)} & 4.71~{\fontsize{7.5}{9}\selectfont(4/17/3)} & 3.70~{\fontsize{7.5}{9}\selectfont(2/7/1)} & 3.45~{\fontsize{7.5}{9}\selectfont(0/10/0)} & 3.67~{\fontsize{7.5}{9}\selectfont(3/6/3)} & 4.67~{\fontsize{7.5}{9}\selectfont(2/9/1)} & 4.55~{\fontsize{7.5}{9}\selectfont(1/28/0)} \\
    \hline
    \rowcolor{ODOverviewOurs}
    \textbf{OnDesign} & \textbf{3.06} & \textbf{2.52} & \textbf{2.05} & \textbf{1.85} & \textbf{2.83} & \textbf{2.42} & \textbf{1.60} \\
    \hline
  \end{tabular}
\end{table}

\subsection{Performance Across Optimization Scenarios}
\label{sec:exp_performance}

Across the three optimization scenarios, OnDesign achieves the best average rank in all 17 suite--dimension configurations (Tables~\ref{tab:bo-ec-main} and~\ref{tab:mix-main}). Detailed results are provided in Appendix~\ref{app:additional_results}.

\paragraph{Bayesian optimization.}
\label{sec:exp_bo}
On BBOB at $D=30$, OnDesign attains an average rank of 2.02, compared with 5.10 for FunBO, the best-ranked baseline, and obtains lower mean final objective values on 21 of the 24 functions. The average-rank lead extends to CEC2020 at $D=20$ (3.80 versus 5.00 for LogPI) and CEC2026 at $D=30$ (2.84 versus 4.66 for JES), again relative to the best-ranked baseline in each configuration. On CEC2022, the margin over the best-ranked baseline is smaller at $D=10$ (4.54 versus 4.75 for EoH) than at $D=20$ (3.58 versus 5.46 for SETUP-BO), showing that the consistent lead in average rank does not imply an equally large margin across configurations.

\paragraph{Evolutionary continuous optimization.}
\label{sec:exp_ec}
On BBOB at $D=30$ and CEC2022 at $D=20$, OnDesign achieves average ranks of 2.52 and 2.42, respectively, compared with 4.35 and 3.88 for AutoEP, the best-ranked baseline in both configurations. On CEC2020 at $D=20$, it ranks ahead of the best-ranked baseline, ReEvo (1.85 versus 3.45), and obtains lower mean final objective values on all ten functions. On CEC2026 at $D=30$, OnDesign attains an average rank of 1.60, compared with 3.53 for AutoEP and 4.55 for ReEvo, the best-ranked adaptive and offline LLM-designed baselines in this configuration, respectively. It obtains lower mean final objective values than these two methods on 23 and 28 of the 29 functions, respectively.

\paragraph{Evolutionary mixed-variable optimization.}
\label{sec:exp_mixed}
On continuous--integer MV-BBOB, OnDesign achieves average ranks of 2.13 and 2.60 at $D=10$ and $D=30$, compared with 3.60 for SHADE and 3.42 for MEoH, the best-ranked baselines in the respective configurations. On continuous--categorical EOPCCV, it obtains the lowest mean final objective value on all 30 functions, yielding an average rank of 1.00, compared with 3.38 for the best-ranked baseline, IMI-GWO.

\paragraph{Engineering validation.}
\label{sec:crashworthiness}
We further evaluate OnDesign on the crashworthiness design of a multicell automotive energy-absorbing component studied by \citet{zhang2024two}. Nine continuous wall-thickness variables are optimized to maximize specific energy absorption (SEA), with each evaluation requiring an LS-DYNA finite-element simulation. Using the Bayesian optimization instantiation, OnDesign generates acquisition functions online to select subsequent solutions for simulation. The comparison uses the same baselines as the Bayesian optimization benchmarks, with 19 initial and 40 sequential evaluations ($FE_{\max}=59$); full problem and evaluation details appear in Appendix~\ref{app:engineering_problem}. Figure~\ref{fig:case1-convergence} shows that OnDesign converges rapidly and reaches approximately 40.95 kJ/kg, the highest final SEA among the compared methods.

\par\medskip
\noindent
\begin{minipage}[t]{0.53\textwidth}
  \vspace{0pt}
  \centering
  \captionof{table}{Results on evolutionary mixed-variable optimization scenario. Average rank ($\downarrow$) with $+/-/=$.}
  \label{tab:mix-main}
  \begingroup
  \fontsize{8.5}{9.5}\selectfont
  \setlength{\tabcolsep}{0.8pt}
  \renewcommand{\arraystretch}{1.0}
  \arrayrulecolor{black}
  \setlength{\ODMixHalfDataWidth}{\dimexpr(\linewidth-40pt-6\tabcolsep)/3\relax}
  \begin{tabular}{@{}>{\raggedright\arraybackslash}p{40pt}*{3}{>{\centering\arraybackslash}p{\ODMixHalfDataWidth}}@{}}
    \toprule
    \multirow{2}{*}[-2.8pt]{\textbf{Method}} & \multicolumn{2}{c}{\textbf{MV-BBOB}} & \textbf{EOPCCV} \\
    \cmidrule(lr){2-3}\cmidrule(l){4-4}
     & 10D & 30D & 10D \\
    \midrule
    CoDE & 9.79~{\fontsize{7.5}{9}\selectfont(0/23/1)} & 8.96~{\fontsize{7.5}{9}\selectfont(0/23/1)} & 10.50~{\fontsize{7.5}{9}\selectfont(0/30/0)} \\
    MDE-IHS & 5.27~{\fontsize{7.5}{9}\selectfont(1/21/2)} & 6.69~{\fontsize{7.5}{9}\selectfont(0/23/1)} & 7.60~{\fontsize{7.5}{9}\selectfont(0/30/0)} \\
    $\mathrm{DE}_{\mathrm{MV}}$ & 7.94~{\fontsize{7.5}{9}\selectfont(0/23/1)} & 8.77~{\fontsize{7.5}{9}\selectfont(0/23/1)} & 5.53~{\fontsize{7.5}{9}\selectfont(0/30/0)} \\
    Sig-DE & 9.29~{\fontsize{7.5}{9}\selectfont(0/23/1)} & 9.38~{\fontsize{7.5}{9}\selectfont(0/23/1)} & 9.87~{\fontsize{7.5}{9}\selectfont(0/30/0)} \\
    IMI-GWO & 4.54~{\fontsize{7.5}{9}\selectfont(3/18/3)} & 4.15~{\fontsize{7.5}{9}\selectfont(7/17/0)} & 3.38~{\fontsize{7.5}{9}\selectfont(0/30/0)} \\
    \noalign{\vskip1.25pt\hbox to\linewidth{\color[gray]{0.45}\leaders\hbox{\rule{2pt}{0.3pt}\kern2pt}\hfill}\vskip1.25pt}
    SHADE & 3.60~{\fontsize{7.5}{9}\selectfont(3/17/4)} & 3.46~{\fontsize{7.5}{9}\selectfont(3/18/3)} & 4.33~{\fontsize{7.5}{9}\selectfont(0/30/0)} \\
    $\mathrm{PSO}_{\mathrm{MV}}$ & 10.06~{\fontsize{7.5}{9}\selectfont(0/23/1)} & 9.94~{\fontsize{7.5}{9}\selectfont(0/23/1)} & 8.17~{\fontsize{7.5}{9}\selectfont(0/30/0)} \\
    \noalign{\vskip1.25pt\hbox to\linewidth{\color[gray]{0.45}\leaders\hbox{\rule{2pt}{0.3pt}\kern2pt}\hfill}\vskip1.25pt}
    EoH & 4.65~{\fontsize{7.5}{9}\selectfont(1/18/5)} & 4.40~{\fontsize{7.5}{9}\selectfont(6/18/0)} & 3.55~{\fontsize{7.5}{9}\selectfont(0/30/0)} \\
    MEoH & 4.02~{\fontsize{7.5}{9}\selectfont(1/17/6)} & 3.42~{\fontsize{7.5}{9}\selectfont(9/13/2)} & 4.72~{\fontsize{7.5}{9}\selectfont(0/30/0)} \\
    ReEvo & 4.71~{\fontsize{7.5}{9}\selectfont(4/16/4)} & 4.25~{\fontsize{7.5}{9}\selectfont(8/15/1)} & 7.35~{\fontsize{7.5}{9}\selectfont(0/30/0)} \\
    \hline
    \rowcolor{ODMixHalfOurs}
    \textbf{OnDesign} & \textbf{2.13} & \textbf{2.60} & \textbf{1.00} \\
    \hline
  \end{tabular}
  \endgroup
\end{minipage}
\hfill
\begin{minipage}[t]{0.49\textwidth}
  \vspace{19pt}
  \centering
  \includegraphics[width=0.75\linewidth]{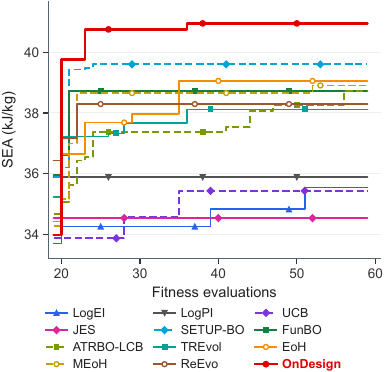}
  \par\vspace{-12pt}
  \captionof{figure}{Engineering validation.}
  \label{fig:case1-convergence}
\end{minipage}
\par\medskip

Taken together, these results show that OnDesign can design acquisition functions and offspring-generation operators across optimization paradigms and variable types, while also effectively addressing expensive engineering optimization problems under limited evaluation budgets.

\paragraph{Case study: Online algorithm design in action.}
To illustrate the dynamics of online algorithm design, we conduct a case study in the Bayesian optimization setting. Figure~\ref{fig:algorithm_evolution} traces how OnDesign redesigns acquisition functions over the course of a target optimization run. Specifically, the traced run considers the 10-dimensional BBOB F1 problem, using 21 initial evaluations followed by 100 sequential evaluations. The 12 displayed checkpoints include all eight evaluations that improve the incumbent and four additional snapshots that expose intervening design changes. Throughout the run, the generated acquisition functions retain an improvement-based backbone while repeatedly restructuring how posterior mean and uncertainty affect candidate selection, including confidence-bound corrections, uncertainty normalization, and improvement-dependent gates. The selected design direction also shifts between exploration- and exploitation-oriented behavior rather than following a predefined schedule. The optimality gap decreases from 26.167 after initialization to 0.02496 at the end of the run. In particular, the acquisition function generated at $FE=56$ reduces the gap from 0.12372 to 0.03794. Appendix~\ref{app:bo_case_study_fe056} expands this checkpoint, tracing the runtime evidence, competing proposals, Arbiter synthesis, executed code, and resulting feedback.

\begin{figure}[!t]
\centering
\includegraphics[width=0.8\linewidth]{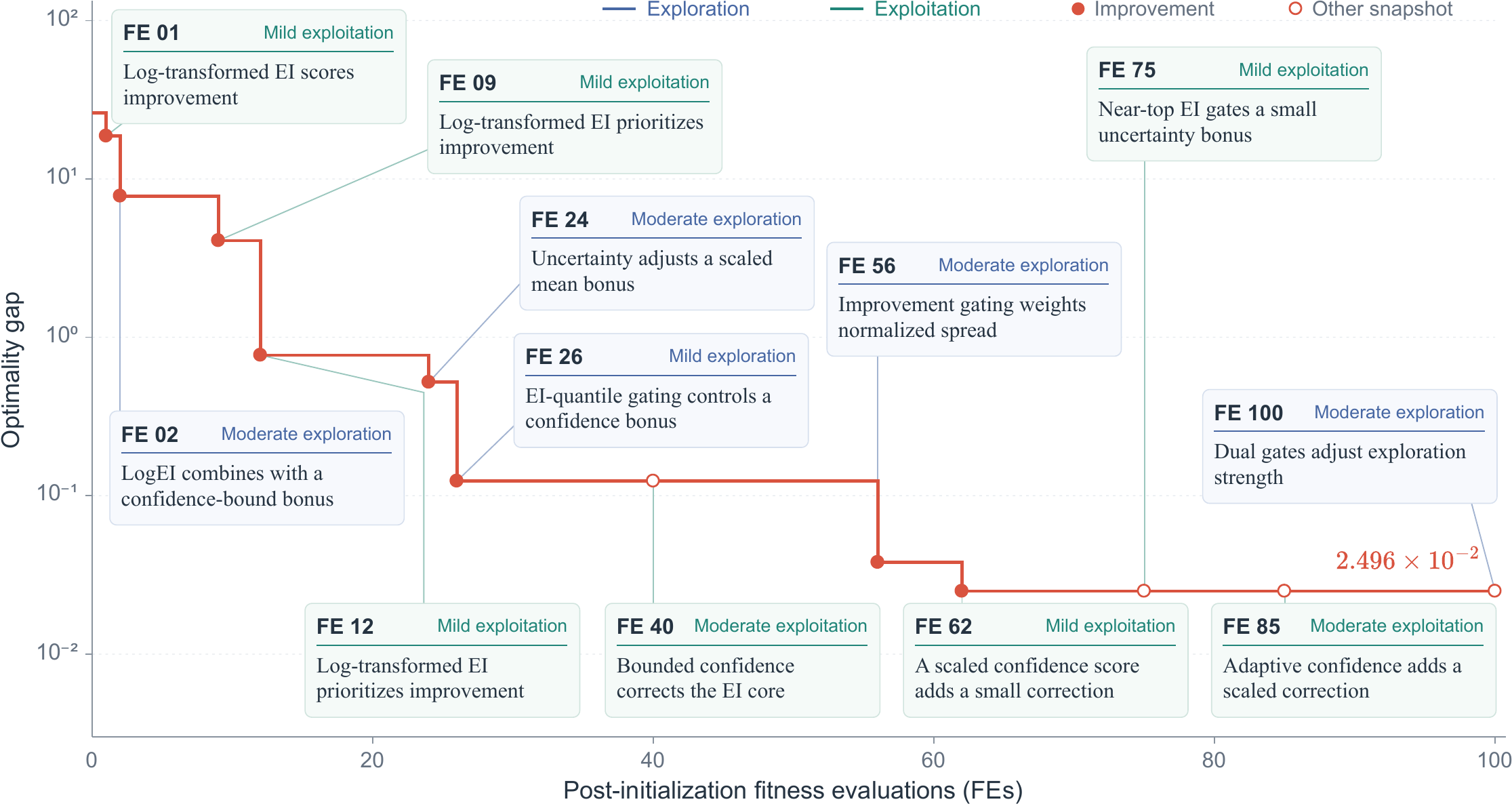}
\caption{Evolution of acquisition functions during a Bayesian optimization run on the 10-dimensional BBOB F1 problem. The red step curve shows the incumbent optimality gap, $f_{\mathrm{best}}-f^\star$.} 
\label{fig:algorithm_evolution}
\end{figure}

\Needspace{18\baselineskip}

\subsection{Ablation Study}
\label{sec:ablation}

\begin{wraptable}{r}{0.49\textwidth}
  \centering
  \caption{Module ablation on Bayesian optimization scenario. Average rank ($\downarrow$) with $+/-/= $.}
  \label{tab:bo-ablation-main}

  \begingroup
  \fontsize{8.5}{9.5}\selectfont
  \setlength{\tabcolsep}{0.8pt}
  \renewcommand{\arraystretch}{1.0}
  \arrayrulecolor{black}

  \begin{tabular}{@{}>{\raggedright\arraybackslash}p{54pt}*{2}{>{\centering\arraybackslash}p{\dimexpr(\linewidth-54pt-4\tabcolsep)/2\relax}}@{}}
    \toprule
    \multirow{2}{*}[-2.8pt]{\textbf{Method}}
      & \multicolumn{2}{c}{\textbf{BBOB}} \\
    \cmidrule(l){2-3}
      & 10D & 30D \\
    \midrule
    w/o SA
      & 2.69~{\fontsize{7.5}{9}\selectfont(6/15/3)}
      & 2.60~{\fontsize{7.5}{9}\selectfont(8/14/2)} \\
    w/o MAD
      & 2.56~{\fontsize{7.5}{9}\selectfont(5/15/4)}
      & 2.88~{\fontsize{7.5}{9}\selectfont(7/16/1)} \\
    w/o SAGE
      & 2.83~{\fontsize{7.5}{9}\selectfont(6/15/3)}
      & 2.46~{\fontsize{7.5}{9}\selectfont(8/14/2)} \\
    \hline
    \rowcolor[HTML]{DAF5FB}
    \textbf{OnDesign}
      & \textbf{1.92}
      & \textbf{2.06} \\
    \hline
  \end{tabular}
  \endgroup
\end{wraptable}

To assess the contribution of each module, we compare OnDesign with three variants on the 24 BBOB functions at $D\in\{10,30\}$. \emph{w/o SA} bypasses the State Analysis module and provides runtime observations directly for algorithm design; \emph{w/o MAD} removes the \emph{Proposers} and relies solely on the \emph{Arbiter} for algorithm synthesis; and \emph{w/o SAGE} disables the SAG Evolution module, retaining the initial SAG throughout optimization. All other settings remain unchanged. Detailed results are given in Table~\ref{tab:bo-ablation-details}.

 As summarized in Table~\ref{tab:bo-ablation-main}, OnDesign achieves the best average rank at both dimensions (1.92 and 2.06, respectively) and outperforms each variant on 14--16 of the 24 functions. All three variants have worse average ranks, with the largest degradation observed for \emph{w/o SAGE} at $D=10$ and \emph{w/o MAD} at $D=30$. Removing any module worsens the average rank in both tested dimensions. This pattern is consistent with the complementary roles highlighted by Proposition~\ref{prop:online_design_gap}: State Analysis and SAG Evolution support the decision-relevant representation associated with current and subsequent $\varepsilon_t$, while Multi-Agent Deliberation supports algorithm coverage and synthesis associated with $\kappa_t$ and $\eta_t$.

\section{Conclusion}
\label{sec:conclusion}

We introduced online LLM-based AAD, bringing state-dependent algorithm synthesis into the target optimization process. OnDesign implements this formulation through State Analysis, Multi-Agent Deliberation, and SAG Evolution, with execution feedback informing both subsequent algorithm designs and the interpretation of runtime evidence. Experiments in Bayesian optimization, evolutionary continuous optimization, and evolutionary mixed-variable optimization show overall advantages over conventional optimizers and offline LLM-based AAD baselines, with further validation on an engineering task. 

The present evaluation focuses on single-objective problems with up to 30 decision variables, while applicability to large-scale, noisy, constrained, and multiobjective settings requires further study. Future work will broaden the evaluation to these settings.
%
%

\section*{Ethics Statement}
We confirm that all authors of this submission have read and agree to abide by the ICLR Code of Ethics.

\section*{Reproducibility Statement}
The code is provided in the supplementary material to enable the reproduction of our results.

\section*{AI Use Statement}
Generative AI tools were used in this work to improve the presentation of the manuscript and assist with preliminary coding. They were not used to formulate the research idea or design the proposed algorithm. All AI-assisted outputs, including revisions to the text, were carefully reviewed by the authors. The authors assume full responsibility for the final manuscript and for all text, claims, and artifacts produced with AI assistance.

\flushbottom
\bibliography{iclr2027_conference}
\bibliographystyle{iclr2027_conference}

\appendix

\clearpage
\section{Related Work}

\textbf{LLM-Based Automated Algorithm Design.}
LLM-based AAD enables algorithms to be generated and refined through program synthesis~\citep{romera2024mathematical}. Representative methods, including EoH~\citep{liu2024eoh}, MEoH~\citep{yao2025meoh}, and ReEvo~\citep{ye2024reevo}, formulate algorithm discovery as a search over executable programs or heuristics. Further advances enhance this search through feedback-guided reflection~\citep{zhong2026hifo}, niching-enhanced population evolution~\citep{hu2025partition}, tree-based exploration~\citep{zheng2025monte}, and performance-guided co-evolution of prompts and heuristics~\citep{liu2025experience}. These developments broaden both the range of discoverable algorithms and the mechanisms for exploring candidate designs. Nevertheless, the methods above generally follow an offline workflow: candidates are generated and evaluated before deployment, and a selected algorithm is then held fixed during the target run. Consequently, performance feedback guides a separate pre-deployment design loop rather than algorithm redesign within the target optimization process.

\textbf{Runtime Algorithm Adaptation.}
A complementary line of research adapts optimization behavior during execution through dynamic parameter control~\citep{zhang2009jade,tanabe2013success}, operator selection, algorithm portfolios~\citep{vasconcelos2022self}, or learned control policies~\citep{guo2024deep,guo2025configx}. These approaches respond to evolving search conditions, typically within a predefined space of parameters, operators, or strategies. Recent efforts have also explored LLM-assisted parameter adaptation~\citep{xu2026autoep}. However, their reliance on predefined adaptation mechanisms can limit their flexibility when effective search requires changes to the underlying algorithmic logic.


\textbf{Positioning of Our Work.}
Online LLM-based AAD integrates program synthesis into the target optimization process, combining the generative scope of LLM-based AAD with the responsiveness of runtime algorithm adaptation. Unlike offline LLM-based AAD, algorithm design continues during the target run; unlike conventional runtime algorithm adaptation, it can extend beyond predefined parameters, operators, or strategies to synthesize new executable algorithmic logic. OnDesign instantiates this formulation through a multi-agent feedback loop: LLM agents generate algorithms from current-state evidence and prior design outcomes, while executing these algorithms advances the search and informs subsequent designs. This shifts the design question from \emph{how a predefined algorithm should adapt} to \emph{what algorithm should run next}.

\clearpage
\section{Performance-Gap Analysis of OnDesign}
\label{app:theory}

This appendix develops the performance-gap analysis in Sec.~\ref{sec:framework}. Its central result, Proposition~\ref{prop:online_design_gap}, bounds the difference in expected terminal utility between OnDesign and an ideal online designer by accumulated design-space, representation, and synthesis errors. Appendix~\ref{app:theory_setup} specifies the comparison and information structure, and Appendix~\ref{app:theory_errors} derives the bound and relates its terms to OnDesign's modules. Appendix~\ref{app:theory_aliasing} examines decision-relevant state information, while Appendix~\ref{app:theory_comparison} gives a conditional comparison with restricted design classes as a corollary. Complete proofs and additional technical details appear in Appendix~\ref{app:theory_proof}.

\subsection{Setup and Reference Designer}
\label{app:theory_setup}

\paragraph{What is compared?}
We compare the expected quality of the best solution found by the end of a common function-evaluation budget. All designers share the task distribution, initial information, program interface, and execution and observation law. The common information constraints specify access to task observations: decisions use only information available at the current stage, without access to hidden task information or future observations. At each history, the ideal online designer may choose any budget-compatible algorithm in $\Pi$ and use the full available history. OnDesign organizes this information into the report-based input defined below. Information obtained before deployment, including information from offline design, is included in the common initial history $\xi_0$.

Let $\xi_t$ denote the complete available history before the algorithm at stage $t$ is chosen, including observations, previously deployed algorithms, relevant internal records, and the remaining budget. Assume terminating executions and a finite upper bound $T$ on the number of design stages. Runs ending earlier are completed with absorbing stages that leave the outcome unchanged. Program-specific evaluation requirements are included in the remaining-budget constraint. We measure the final outcome by a bounded utility
\begin{equation}
U:\Xi_T\longrightarrow[0,1],
\label{eq:theory_terminal_utility}
\end{equation}
where $\Xi_T$ denotes the set of terminal histories, and larger utility values mean better solutions. Any utility with a known finite range can be put on this scale by a positive affine transformation. The same utility is used for every designer. The execution gains $r_t$ provide feedback during the run, while $U$ evaluates the final result; $U$ need not equal the sum of those gains.

\paragraph{What information reaches algorithm synthesis?}
The general design state $(s_t,h_t)$ in Sec.~\ref{sec:oaad} combines the optimization state and design history. As described in Sec.~\ref{sec:framework}, OnDesign instantiates these components through a state report $q_t$ and an algorithm archive $\mathbb H_t$. Following that section, we denote its report-based design input by $z_t$ and write
\begin{equation}
z_t=\rho_t(\xi_t)=(q_t,\mathbb H_t).
\label{eq:theory_representation}
\end{equation}
Thus, $\xi_t$ is the full history used for analysis, while $z_t$ is the input used for algorithm synthesis under the fixed task instructions $\iota$. Equal access to task observations permits different representations of those observations; the representation error below quantifies how well $z_t$ supports value assessment relative to the full-history reference. Different histories may produce the same input $z_t$, and SAG updates can change the mapping $\rho_t$. Using complete histories for the analysis accommodates report-based inputs that need not themselves be Markov states, as in work on history aggregation~\citep{hutter2016extreme}.

\paragraph{How is the value of an algorithm defined?}
A design policy $\mathcal D$ is a rule for choosing the next program from the available information. At history $\xi_t$, its choice belongs to $\Pi(\xi_t)\subseteq\Pi$, the algorithms compatible with the task interface and remaining evaluation budget. Executing the chosen program produces the next history. The expected final utility of following $\mathcal D$ from a history $\xi$ is
\begin{equation}
V_t^{\mathcal D}(\xi)
=\mathbb E_{\mathcal D}\!\left[U(\xi_T)\mid\xi_t=\xi\right].
\label{eq:theory_policy_value}
\end{equation}
Two reference values help evaluate a design decision. The value $V_t^*(\xi)$ describes the best expected outcome attainable from the current history. The value $Q_t^*(\xi,\pi)$ describes the outcome of executing program $\pi$ now and making optimal design decisions afterwards. Their difference therefore measures the loss associated with the current choice, including its effects on the subsequent search.

Using finite-horizon dynamic programming~\citep{puterman2014markov}, these values satisfy
\begin{equation}
\begin{aligned}
V_t^*(\xi)
&=\sup_{\mathcal D}V_t^{\mathcal D}(\xi),\\
Q_t^*(\xi,\pi)
&=\mathbb E\!\left[V_{t+1}^*(\xi_{t+1})
    \mid\xi_t=\xi,\pi_t=\pi\right],\\
V_t^*(\xi)
&=\sup_{\pi\in\Pi(\xi)}Q_t^*(\xi,\pi),
\qquad V_T^*(\xi)=U(\xi).
\end{aligned}
\label{eq:theory_bellman}
\end{equation}
These are analytical reference values; OnDesign need not compute them. We use suprema so that a best program or policy need not be attained. The regularity conditions for this recursion are stated in Appendix~\ref{app:theory_proof}. Below, $V_0^*$ abbreviates $V_0^*(\xi_0)$, $V_0^{\mathrm{OnDesign}}$ abbreviates $V_0^{\mathrm{OnDesign}}(\xi_0)$, and expectations are conditional on $\xi_0$.

\subsection{Error Decomposition and Performance-Gap Bound}
\label{app:theory_errors}

We analyze each design decision through the accessible algorithm family, the information supporting value assessment, and the quality of the synthesized program. Fix an accessible family and a reference score as specified below. The resulting three error terms share the final-utility scale and jointly bound the performance gap; their interaction is allowed throughout the analysis.

\paragraph{Design-space error: which algorithms are available?}
Let $\mathcal C_t(\xi_t)\subseteq\Pi(\xi_t)$ be a nonempty family accessible to the design procedure, specified before the current output is sampled and containing that output $\pi_t$ almost surely. The difference between the best attainable values in the full and accessible families is
\begin{equation}
\kappa_t(\xi_t)
=V_t^*(\xi_t)
-\sup_{\pi\in\mathcal C_t(\xi_t)}Q_t^*(\xi_t,\pi)
\ge0.
\label{eq:theory_coverage}
\end{equation}
Both alternatives are evaluated with optimal continuation after the current stage. At a fixed history, enlarging the accessible family cannot increase this error. The family describes the programs accessible through synthesis; it need not be explicitly enumerated.

\paragraph{Representation error: can the input support accurate value judgments?}
Introduce a reference score $G_t(z,\pi)\in[0,1]$ that uses only the design input and the candidate program. It expresses the value information supported by the representation for purposes of analysis, without requiring an explicit value estimator in OnDesign. We assume that, along the histories visited by OnDesign, this score approximates program values with error at most $\varepsilon_t(\xi_t)\ge0$:
\begin{equation}
\sup_{\pi\in\mathcal C_t(\xi_t)}
\left|G_t(z_t,\pi)-Q_t^*(\xi_t,\pi)\right|
\le\varepsilon_t(\xi_t).
\label{eq:theory_representation_error}
\end{equation}
The same score function must apply whenever the input $z_t$ is the same; it cannot recover information omitted from that input. This connects the condition to approximate value-preserving state abstraction~\citep{abel2016near}. Appendix~\ref{app:theory_aliasing} gives an example of how omitted information can affect design decisions.

The archive $\mathbb H_t=\{(\pi_j,r_j)\}_{j<t}$ supplies feedback from earlier designs, while $q_t$ characterizes current search conditions. The representation condition concerns their joint ability to support candidate-value assessment, including any historical context needed to interpret earlier gains.

\paragraph{Synthesis error: how good is the program actually generated?}
Using the same reference score $G_t$ as in Eq.~(\ref{eq:theory_representation_error}), we quantify how closely the generated program realizes the best accessible assessed value. We assume, almost surely,
\begin{equation}
\sup_{\pi\in\mathcal C_t(\xi_t)}G_t(z_t,\pi)
-G_t(z_t,\pi_t)
\le\eta_t(\xi_t,\pi_t),
\label{eq:theory_synthesis_error}
\end{equation}
where $\eta_t(\xi_t,\pi_t)\ge0$. This error measures how closely synthesis realizes the highest-scoring accessible design. The condition permits randomized synthesis and allows the error to depend on the sampled output.

\paragraph{From one decision to the final result.}
These conditions bound the loss of each actual design decision:
\begin{equation}
V_t^*(\xi_t)-Q_t^*(\xi_t,\pi_t)
\le\kappa_t+2\varepsilon_t+\eta_t.
\label{eq:theory_stage_bound}
\end{equation}
The factor $2$ appears because value approximation enters both sides of the comparison: the best accessible program and the program actually generated. The terms form an upper bound rather than an equality decomposition of the decision loss.

The next identity connects these individual decisions to final performance. It is a finite-horizon form of the performance-difference identity~\citep{kakade2002approximately}, specialized to the terminal utility used here.
\begin{odlemma}[Finite-horizon value difference]
\label{lem:theory_value_difference}
For the process in Appendix~\ref{app:theory_setup}, under the technical conditions in Appendix~\ref{app:theory_proof},
\begin{equation}
V_0^*-V_0^{\mathrm{OnDesign}}
=\mathbb E_{\mathrm{OnDesign}}\!\left[
\sum_{t=0}^{T-1}
\left(V_t^*(\xi_t)-Q_t^*(\xi_t,\pi_t)\right)
\right].
\label{eq:theory_value_difference}
\end{equation}
\end{odlemma}
The sum follows OnDesign's actual trajectory. Each term already includes the effect of the current choice on future opportunities through $Q_t^*$. Thus, the identity accounts for the interaction between successive design decisions.

Combining Eqs.~(\ref{eq:theory_stage_bound}) and~(\ref{eq:theory_value_difference}) proves Proposition~\ref{prop:online_design_gap}. For later comparison, define the accumulated error bound
\begin{equation}
\mathcal E_{\mathrm{OnDesign}}
:=
\min\!\left\{1,\,
\mathbb E_{\mathrm{OnDesign}}\!\left[
\sum_{t=0}^{T-1}(\kappa_t+2\varepsilon_t+\eta_t)
\right]\right\}.
\label{eq:theory_accumulated_error}
\end{equation}
Then $0\le V_0^*-V_0^{\mathrm{OnDesign}}\le\mathcal E_{\mathrm{OnDesign}}$. The cap of one uses the normalization $U\in[0,1]$ adopted here; the uncapped bound is the result stated in Proposition~\ref{prop:online_design_gap}. If the expected stage error satisfies $\mathbb E_{\mathrm{OnDesign}}[\kappa_t+2\varepsilon_t+\eta_t]\le\delta$ for every $t$, then $\mathcal E_{\mathrm{OnDesign}}\le\min\{1,T\delta\}$. The result thus identifies the accumulated design error that controls proximity to the ideal reference under the stated conditions.

\paragraph{Connection to OnDesign.}
The three modules address complementary aspects of the bound. State Analysis organizes decision-relevant evidence in $q_t$, targeting representation error $\varepsilon_t$. Multi-Agent Deliberation develops alternative design directions and integrates them into executable code, linking its role to the accessible family $\mathcal C_t$ and the errors $\kappa_t$ and $\eta_t$. SAG Evolution updates the analytical focus used to construct subsequent reports, targeting future representation errors. These connections describe how the modules can contribute to a smaller performance gap through reductions in accumulated error; the bound accommodates their coupled effects.

\subsection{Representation Ambiguity and State Information}
\label{app:theory_aliasing}

A useful state report preserves information that can change which algorithm should be generated next. Consider a one-stage example with two equally likely histories and two candidate programs. One program emphasizes exploration and the other emphasizes local refinement. Suppose their terminal utilities are
\begin{center}
\begin{tabular}{lcc}
\toprule
History & Exploration program & Refinement program\\
\midrule
$\xi$: exploration is useful & $1$ & $0$\\
$\xi'$: refinement is useful & $0$ & $1$\\
\bottomrule
\end{tabular}
\end{center}
If both histories produce the same design input $z$, including the same report and algorithm archive, a designer whose choices depend only on that input uses the same distribution over programs in both cases. Its expected utility is $1/2$, whatever that distribution is. A designer that distinguishes the histories can choose the useful program in each case and achieve utility $1$. This example isolates the value of retaining information that changes the preferred algorithm, even when both useful programs are already accessible.

The relevant connection is to value-preserving state abstraction~\citep{li2006towards} and its approximate variants~\citep{abel2016near}. More generally, when two histories share the same input, a score based on that input must assign the same value to a given program. If its true values differ, at least one approximation must be inaccurate.

\begin{odproposition}[Value ambiguity under a shared representation]
\label{prop:theory_aliasing}
Consider two histories $\xi$ and $\xi'$ at the same stage with $\rho_t(\xi)=\rho_t(\xi')=z$. For any program $\pi$ belonging to both accessible families, every score based only on $(z,\pi)$ satisfies
\begin{equation}
\begin{aligned}
&\max\left\{
\left|G_t(z,\pi)-Q_t^*(\xi,\pi)\right|,
\left|G_t(z,\pi)-Q_t^*(\xi',\pi)\right|
\right\}\\
&\qquad\ge\frac12
\left|Q_t^*(\xi,\pi)-Q_t^*(\xi',\pi)\right|.
\end{aligned}
\label{eq:theory_aliasing}
\end{equation}
\end{odproposition}
The shared-representation condition applies to the complete design input $(q_t,\mathbb H_t)$, including both the report and the archive. The proposition gives a lower bound on value-approximation error when histories with different program values share that input. Its decision-level implication depends on which program is preferred: the example above turns value ambiguity into a performance gap by reversing the preferred program across histories. Preserving distinctions relevant to this choice therefore provides a concrete objective for state analysis.

\paragraph{Connection to State Analysis and SAG Evolution.}
State Analysis organizes evidence that distinguishes useful algorithmic responses. As the search evolves, SAG Evolution updates the guideline underlying $\rho_t$ to direct subsequent reports toward newly relevant evidence. The execution gains provide feedback for this refinement, as described in Sec.~\ref{sec:framework}. Proposition~\ref{prop:theory_aliasing} explains the information-preservation objective, while Proposition~\ref{prop:online_design_gap} characterizes its benefit when refinement reduces accumulated representation error.

\subsection{Conditional Comparisons with Restricted Design Classes}
\label{app:theory_comparison}

The performance-gap bound also yields a comparison with a specified design class. The following corollary relates the potential value of greater design freedom to the accumulated error bound for OnDesign.

\paragraph{Which methods form the comparison class?}
Let $\mathfrak D_{\mathrm{all}}$ contain all design policies allowed in Eq.~(\ref{eq:theory_bellman}). These policies use available history and obey the common interface and budget. Two restricted classes are of interest:
\begin{itemize}
\item \textbf{Offline AAD.} Fix a family $\mathcal F$ before deployment. A policy in $\mathfrak D_{\mathrm{off}}(\mathcal F)$ selects a program from this family using the initial information and keeps that program throughout the run. The deployed program may have internal state and adapt its actions to observations; the restriction concerns redesigning its code.
\item \textbf{Predefined adaptation.} Fix a template family $\{\pi_\theta:\theta\in\Theta_{\mathrm{pre}}\}$ before deployment. A policy in $\mathfrak D_{\mathrm{adapt}}(\Theta_{\mathrm{pre}})$ may use the complete available history to choose configurations, but its programs must remain in this family.
\end{itemize}
Both are subsets of $\mathfrak D_{\mathrm{all}}$: an unrestricted online policy can retain a fixed program or stay within a fixed template family. Each class is compared separately with unrestricted online design while holding its declared family fixed.

\paragraph{Potential value and realized advantage.}
For any nonempty restricted class $\mathfrak R\subseteq\mathfrak D_{\mathrm{all}}$, define
\begin{equation}
V_0^{\mathfrak R}
:=\sup_{\mathcal D\in\mathfrak R}V_0^{\mathcal D}(\xi_0),
\qquad
\Delta_{\mathfrak R}:=V_0^*-V_0^{\mathfrak R}\ge0.
\label{eq:theory_restriction_gap}
\end{equation}
The restriction gap $\Delta_{\mathfrak R}$ is the performance improvement potentially available beyond the best method in the specified class. Its value depends on the expressiveness of that class and is zero if the class attains the ideal value. It compares complete policies and their induced trajectories, whereas $\kappa_t$ concerns the programs available at one current history with optimal continuation thereafter.

\begin{odcorollary}[Advantage over a restricted design class]
\label{prop:theory_comparison}
Under the conditions of Proposition~\ref{prop:online_design_gap}, every nonempty class $\mathfrak R\subseteq\mathfrak D_{\mathrm{all}}$ satisfies
\begin{equation}
\Delta_{\mathfrak R}-\mathcal E_{\mathrm{OnDesign}}
\le V_0^{\mathrm{OnDesign}}-V_0^{\mathfrak R}
\le\Delta_{\mathfrak R}.
\label{eq:theory_comparison}
\end{equation}
Consequently, if $\mathcal E_{\mathrm{OnDesign}}<\Delta_{\mathfrak R}$, then OnDesign has strictly higher expected terminal utility than every policy in $\mathfrak R$.
\end{odcorollary}
The lower bound expresses actual advantage as at least the potential improvement minus the accumulated error bound. Choosing either of the specified offline or predefined adaptive classes gives the corresponding comparison. Thus, a positive restriction gap is realized as a strict advantage when OnDesign's accumulated error bound is smaller than that gap.

Appendix~\ref{app:theory_proof} gives a sufficient criterion for a positive restriction gap. Together, these results connect the scope of accessible designs and the quality of state-conditioned synthesis to expected final optimization performance under a common evaluation budget.

\Needspace{9\baselineskip}
\subsection{Proofs and Technical Details}
\label{app:theory_proof}

\paragraph{Technical conditions.}
Fix the task instructions $\iota$, the initial history $\xi_0$, a task distribution, and a common execution and observation law before the run. Policies may condition on available histories but cannot access the realized hidden task or future observations directly. Histories include relevant realized internal randomness; OnDesign's records include its reports and SAG history. We assume that the conditional distributions and suprema are measurable and admit the finite-horizon Bellman recursion in Eq.~(\ref{eq:theory_bellman}). Countable spaces of finite programs with well-defined execution provide a natural setting.

Programs may consume different amounts of the budget. Their costs enter the transition law and the history-dependent algorithm space $\Pi(\xi_t)$. The analysis considers terminating executions for which the budget imposes a finite upper bound $T$ on the number of design stages. Runs ending earlier are padded with absorbing stages that have a single no-operation program, zero cost, unchanged utility, and zero design errors.

The accessible family contains the generated program almost surely. If $\mathcal C_t(\xi_t)=\Pi(\xi_t)$, then $\kappa_t=0$, and the performance bound is determined by representation and synthesis errors. The analysis allows correlated proposals and an Arbiter that combines advice into a new program. Equations~(\ref{eq:theory_representation_error}) and~(\ref{eq:theory_synthesis_error}) are imposed along OnDesign's histories and outputs. The same expected performance bound follows if the synthesis condition holds in conditional expectation.

\begin{proof}[Proof of Lemma~\ref{lem:theory_value_difference}]
By the definition of $Q_t^*$ and the tower property,
\begin{equation}
\mathbb E_{\mathrm{OnDesign}}\!\left[Q_t^*(\xi_t,\pi_t)\right]
=\mathbb E_{\mathrm{OnDesign}}\!\left[V_{t+1}^*(\xi_{t+1})\right].
\end{equation}
Subtracting and summing over stages cancels all intermediate values:
\begin{equation}
\begin{aligned}
&\mathbb E_{\mathrm{OnDesign}}\!\left[
\sum_{t=0}^{T-1}\bigl(V_t^*(\xi_t)-Q_t^*(\xi_t,\pi_t)\bigr)
\right]\\
&\qquad=V_0^*(\xi_0)
-\mathbb E_{\mathrm{OnDesign}}\!\left[V_T^*(\xi_T)\right]\\
&\qquad=V_0^*(\xi_0)
-\mathbb E_{\mathrm{OnDesign}}\!\left[U(\xi_T)\right]
=V_0^*-V_0^{\mathrm{OnDesign}}.
\end{aligned}
\end{equation}
This argument follows the actual trajectory and does not require independent design decisions.
\end{proof}

\begin{proof}[Proof of Proposition~\ref{prop:online_design_gap}]
First bound the loss of one design decision. Fix a visited history and output program, and suppress the arguments of $\kappa_t$, $\varepsilon_t$, and $\eta_t$. The representation condition gives
\begin{equation}
\begin{aligned}
\sup_{\pi\in\mathcal C_t}Q_t^*(\xi_t,\pi)
&\le\sup_{\pi\in\mathcal C_t}G_t(z_t,\pi)+\varepsilon_t,\\
Q_t^*(\xi_t,\pi_t)
&\ge G_t(z_t,\pi_t)-\varepsilon_t.
\end{aligned}
\end{equation}
Using the definition of $\kappa_t$ and then the synthesis condition,
\begin{equation}
\begin{aligned}
V_t^*(\xi_t)-Q_t^*(\xi_t,\pi_t)
&=\kappa_t+\sup_{\pi\in\mathcal C_t}Q_t^*(\xi_t,\pi)
-Q_t^*(\xi_t,\pi_t)\\
&\le\kappa_t+2\varepsilon_t
+\sup_{\pi\in\mathcal C_t}G_t(z_t,\pi)-G_t(z_t,\pi_t)\\
&\le\kappa_t+2\varepsilon_t+\eta_t.
\end{aligned}
\end{equation}
This proves Eq.~(\ref{eq:theory_stage_bound}). Next take expectations and sum over stages. Lemma~\ref{lem:theory_value_difference} converts this sum into the final performance gap, proving the upper bound. The lower bound follows because the ideal policy class contains OnDesign.
\end{proof}

\begin{proof}[Proof of Proposition~\ref{prop:theory_aliasing}]
The same input requires the same score $G_t(z,\pi)$ in both histories. The triangle inequality gives
\[
\left|Q_t^*(\xi,\pi)-Q_t^*(\xi',\pi)\right|
\le
\left|Q_t^*(\xi,\pi)-G_t(z,\pi)\right|
+
\left|G_t(z,\pi)-Q_t^*(\xi',\pi)\right|.
\]
At least one term on the right must be at least half the left-hand side.
\end{proof}

\Needspace{7\baselineskip}
\begin{proof}[Proof of Corollary~\ref{prop:theory_comparison}]
Adding and subtracting the ideal value gives
\begin{equation}
V_0^{\mathrm{OnDesign}}-V_0^{\mathfrak R}
=\Delta_{\mathfrak R}
-\left(V_0^*-V_0^{\mathrm{OnDesign}}\right).
\end{equation}
Combining Proposition~\ref{prop:online_design_gap} with the range of the terminal utility bounds the parenthesized term between zero and $\mathcal E_{\mathrm{OnDesign}}$, as defined in Eq.~(\ref{eq:theory_accumulated_error}), yielding Eq.~(\ref{eq:theory_comparison}). If $\mathcal E_{\mathrm{OnDesign}}<\Delta_{\mathfrak R}$, OnDesign exceeds the supremum of the restricted class.
\end{proof}

\paragraph{A sufficient criterion for a positive restriction gap.}
Suppose every policy in the restricted class has probability at least $p>0$ of reaching, at a specified stage, histories where each program it can deploy loses at least $m>0$ relative to the ideal choice. Then
\begin{equation}
\Delta_{\mathfrak R}\ge pm,
\qquad
V_0^{\mathrm{OnDesign}}-V_0^{\mathfrak R}
\ge pm-\mathcal E_{\mathrm{OnDesign}}.
\label{eq:theory_positive_advantage}
\end{equation}
The condition and its proof can be made precise as follows.
The value-difference identity applies to any admissible policy $\mathcal D$. Therefore,
\begin{equation}
\Delta_{\mathfrak R}
=\inf_{\mathcal D\in\mathfrak R}
\mathbb E_{\mathcal D}\!\left[
\sum_{t=0}^{T-1}
\left(V_t^*(\xi_t)-Q_t^*(\xi_t,\pi_t)\right)
\right].
\label{eq:theory_restriction_identity}
\end{equation}
Fix a stage $\tau$ and a measurable set $\mathcal B_\tau$ of histories at that stage. Suppose every $\mathcal D\in\mathfrak R$ reaches $\mathcal B_\tau$ at stage $\tau$ with probability at least $p>0$. On this event, suppose every program that the restricted policy may deploy satisfies
\begin{equation}
V_\tau^*(\xi)-Q_\tau^*(\xi,\pi)\ge m>0.
\label{eq:theory_unavoidable_loss}
\end{equation}
Every summand in Eq.~(\ref{eq:theory_restriction_identity}) is nonnegative, and the stage-$\tau$ summand is at least $pm$ in expectation for each restricted policy. Taking the infimum gives $\Delta_{\mathfrak R}\ge pm$. Substituting this into Eq.~(\ref{eq:theory_comparison}) proves Eq.~(\ref{eq:theory_positive_advantage}).

\clearpage
\section{OnDesign for Bayesian Optimization}
\label{app:ondesign_bo}
\textbf{Algorithm~\ref{alg:ondesign_bo}} summarizes OnDesign for Bayesian optimization, with task instructions $\iota$ in Fig.~\ref{fig:prompt_bo_task}. At design stage $t$, $\pi_t$ denotes the current acquisition function, which maps the incumbent value and Gaussian Process (GP) posterior statistics to an acquisition utility for selecting the next candidate. Throughout Appendices~\ref{app:ondesign_bo}--\ref{app:ondesign_emvo}, $f$ denotes the true objective function, $t$ indexes design stages, $N$ is the initial sample size, $FE$ is the number of consumed fitness evaluations,  and $FE_{\text{max}}$ is the fitness evaluation budget. Following Sec.~\ref{sec:framework}, $\mathcal{P}$ denotes \emph{State Analyst} and $\mathcal{E}$ denotes \emph{SAG Refiner}.

\begin{algorithm}[H]
\renewcommand{\algorithmicrequire}{\textbf{Input:}}
\renewcommand{\algorithmicensure}{\textbf{Output:}}
\caption{OnDesign for Bayesian Optimization}
\label{alg:ondesign_bo}
\small
\begin{algorithmic}[1]
\Require $f$, $N$, $FE_{\text{max}}$, $L_{\mathrm{SAG}}$, and task instructions $\iota$
\Ensure Best solution $\textbf{x}^{*}$ and its fitness $f(\textbf{x}^{*})$
\State $t\leftarrow 0$; $[\mathbb{DB},FE]\leftarrow\mathit{Initialization}(N)$;
\State Initialize $\phi_t$; $\mathbb{H}_t\leftarrow\emptyset$; $\mathcal{H}^{\mathrm{SAG}}_t\leftarrow\emptyset$;
\State $\mathbb{MODEL}_t\leftarrow\mathit{Surrogate\_Construction}(\mathbb{DB})$;
\While{$FE<FE_{\text{max}}$}
    \State Collect runtime observations $o_t$;
    \State $q_t\leftarrow\mathcal{P}(o_t;\phi_t)$; \Comment{State Analyst}
    \State $\pi_t\leftarrow\mathit{Multi\text{-}Agent\_Deliberation}(q_t,\mathbb{H}_t;\iota)$;
    \State $\check{\textbf{x}}_t\leftarrow\mathit{Candidate\_Selection}(\mathbb{DB},\mathbb{MODEL}_t,\pi_t)$;
    \State Evaluate $\check{\textbf{x}}_t$ by $f$;
    \State $[\mathbb{H}_{t+1},r_t]\leftarrow\mathit{Feedback\_Update}(\mathbb{H}_t,\pi_t,\mathbb{DB},\check{\textbf{x}}_t)$;
    \State $\mathcal{H}^{\mathrm{SAG}}_{t+1}\leftarrow\mathit{SAG\_Feedback}(\mathcal{H}^{\mathrm{SAG}}_t,\phi_t,r_t)$;
    \State $\mathbb{DB}\leftarrow\mathbb{DB}\cup\{\check{\textbf{x}}_t\}$ and $FE\leftarrow FE+1$;
    \State $\mathbb{MODEL}_{t+1}\leftarrow\mathit{Surrogate\_Construction}(\mathbb{DB})$;
    \State $\phi_{t+1}\leftarrow\mathit{SAG\_Update}(\phi_t,\mathcal{H}^{\mathrm{SAG}}_{t+1};t,L_{\mathrm{SAG}})$; \Comment{SAG Refiner}
    \State $t\leftarrow t+1$;
\EndWhile
\State $\textbf{x}^{*}\leftarrow\arg\min_{\textbf{x}\in\mathbb{DB}} f(\textbf{x})$;
\end{algorithmic}
\end{algorithm}

At design stage $t$, $\mathbb{H}_t$ and $\mathcal{H}^{\mathrm{SAG}}_t$ denote the algorithm and SAG archives, respectively, following Sec.~\ref{sec:framework}. Execution feedback produces $\mathbb{H}_{t+1}$ and $\mathcal{H}^{\mathrm{SAG}}_{t+1}$, while $\mathbb{DB}$ and $FE$ are updated in place. The following feedback conventions also apply to Appendices~\ref{app:ondesign_ec} and~\ref{app:ondesign_emvo}.

\paragraph{SAG scoring and refinement.}
After executing $\pi_t$, $\mathit{SAG\_Feedback}$ credits $r_t$ to $\phi_t$, the SAG that guided the construction of $q_t$. Let $\mathcal{I}_{t+1}(\phi)=\{j\mid 0\le j\le t,\ \phi_j=\phi\}$ index the completed design stages associated with a SAG $\phi$. Its score and the scored SAG archive are updated as
\begin{equation}
\begin{aligned}
    g_{t+1}^{\mathrm{SAG}}(\phi)
    &= \frac{1}{|\mathcal{I}_{t+1}(\phi)|}
       \sum_{j\in\mathcal{I}_{t+1}(\phi)}r_j,\\
    \mathcal{H}^{\mathrm{SAG}}_{t+1}
    &= \left\{\big(\phi,g_{t+1}^{\mathrm{SAG}}(\phi)\big)
       \mid \mathcal{I}_{t+1}(\phi)\ne\emptyset\right\},
\end{aligned}
\label{eq:sag_feedback}
\end{equation}
The score averages all completed execution gains attributed to the same SAG; only SAGs with at least one such execution are included in the archive. Feedback is recorded after every execution, independently of the refinement interval. $\mathit{SAG\_Update}$ then invokes $\mathcal{E}$ on $\mathcal{H}^{\mathrm{SAG}}_{t+1}$ if $(t+1)\bmod L_{\mathrm{SAG}}=0$, otherwise retaining $\phi_t$, as specified in Eq.~(\ref{eq:sag_update}). The resulting $\phi_{t+1}$ is used for state analysis at the next design stage.

Throughout Appendices~\ref{app:ondesign_bo}--\ref{app:ondesign_emvo}, execution feedback uses the square root of the nonnegative relative improvement in the best observed objective value. Let $b_t$ and $b_{t+1}$ denote the best values before and after executing $\pi_t$, respectively. For minimization,
\begin{equation}
    r_t=\sqrt{\frac{\max\{b_t-b_{t+1},0\}}{d_t}},
    \qquad
    d_t=\begin{cases}
        |b_t|, & b_t\ne0,\\
        \max\{|b_{t+1}|,10^{-20}\}, & b_t=0.
    \end{cases}
    \label{eq:execution_gain}
\end{equation}
Thus, the previous best value provides the normalization baseline; the second case avoids a zero denominator. A stage that does not improve the incumbent receives zero feedback. In BO, $b_{t+1}$ is the best value in the database after adding the evaluated candidate; in the evolutionary settings, it is the best value after population selection.

The framework has five main steps:
\begin{enumerate}
    \item \textbf{Initialization} (Lines 1--2): Evaluate $N$ samples to initialize $\mathbb{DB}$ and $FE=N$. The SAG Refiner $\mathcal{E}$ generates $\phi_t$ from the task instructions and available state-data types. Both archives start empty.
    \item \textbf{Surrogate Construction} (Lines 3 and 13): Fit a GP model to $\mathbb{DB}$, obtaining $\mathbb{MODEL}_t$ with mean $\hat{y}_t(\textbf{x})$ and standard deviation $\hat{s}_t(\textbf{x})$. Refit after each evaluation.
    \item \textbf{Online Acquisition-Function Design} (Lines 5--7): Starting at $t=0$, $\mathcal{P}$ constructs $q_t$ from $o_t$ under $\phi_t$ (Table~\ref{tab:bo_state_data_hint}). Multi-Agent Deliberation module synthesizes $\pi_t$ using $q_t$ and the algorithm descriptions and gains in $\mathbb{H}_t$.
    \item \textbf{Candidate Selection} (Lines 8--9): Execution of $\pi_t$ selects
    \[
        \check{\textbf{x}}_t=\arg\max_{\textbf{x}\in\mathcal{X}}
        \pi_t\big(best\_y,\hat{y}_t(\textbf{x}),\hat{s}_t(\textbf{x})\big),
    \]
    where $best\_y$ is the best observed fitness. Evaluate $\check{\textbf{x}}_t$ by $f$.
    \item \textbf{Feedback and SAG Update} (Lines 10--14): Compute $r_t$ from the previous best fitness and the evaluated candidate using Eq.~(\ref{eq:execution_gain}). Archive $(\pi_t,r_t)$ in $\mathbb{H}_{t+1}$, credit $r_t$ to $\phi_t$ in $\mathcal{H}^{\mathrm{SAG}}_{t+1}$, and add the evaluated candidate to $\mathbb{DB}$. After incrementing $FE$ and refitting the surrogate, $\mathit{SAG\_Update}$ determines $\phi_{t+1}$ from $\mathcal{H}^{\mathrm{SAG}}_{t+1}$. Advance $t$ once (Line 15), and use this guideline for the next state report.
\end{enumerate}
The best solution in $\mathbb{DB}$ is returned at budget exhaustion (i.e., $FE=FE_{\text{max}}$). The shared prompts appear in Appendix~\ref{app:prompt_templates}.

\paragraph{State information.}
An acquisition function determines how model predictions and uncertainty guide the next evaluation. Its online design therefore requires a joint view of observed progress, surrogate beliefs, and the spatial distribution of evaluated solutions. Table~\ref{tab:bo_state_data_hint} organizes this evidence into complementary categories that form \emph{State Data Types Hint} provided to \emph{SAG Refiner} $\mathcal{E}$ within SAG Evolution module. At design stage $t$, the resulting SAG $\phi_t$ guides \emph{State Analyst} $\mathcal{P}$ in emphasizing and relating the available observations $o_t$ in the state report $q_t$.

\begin{table}[H]
    \centering
    \caption{State information summarized by the \emph{State Data Types Hint} for Bayesian optimization.}
    \label{tab:bo_state_data_hint}
    \small
    \renewcommand{\arraystretch}{1.12}
    \begin{tabular}{@{}>{\raggedright\arraybackslash}p{0.20\linewidth}@{\hspace{0.025\linewidth}}>{\raggedright\arraybackslash}p{0.48\linewidth}@{\hspace{0.025\linewidth}}>{\raggedright\arraybackslash}p{0.27\linewidth}@{}}
        \toprule
        \textbf{State information} & \textbf{Representative descriptors} & \textbf{Role in state analysis} \\
        \midrule
        Observed fitness distribution &
        Best and worst fitness, mean and dispersion of fitness values in $\mathbb{DB}$. &
        Characterizes attained solution quality and variation across evaluated regions. \\
        \addlinespace
        Search progress &
        Recent evaluation trajectory, successive fitness changes, and absolute or normalized improvement in the best-found solution. &
        Distinguishes short-term fluctuations from sustained progress or stagnation. \\
        \addlinespace
        Surrogate characteristics &
        GP signal and noise variances, their relative magnitudes, kernel lengthscales, and their evolution during search. &
        Describes the model's inferred correlation structure and signal--noise balance. \\
        \addlinespace
        Posterior predictions &
        Predicted fitness and uncertainty across the search space, including the locations of low-predicted-fitness regions. &
        Relates model-predicted opportunities to uncertainty in those regions. \\
        \addlinespace
        Sampling geometry &
        Spatial isolation of the incumbent, distances between recent candidates and evaluated solutions, and proximity to the domain boundary. &
        Places the incumbent and recent evaluations in the context of sampled regions. \\
        \addlinespace
        Model--data agreement &
        Discrepancy between observed fitness and the GP posterior mean at evaluated solutions. &
        Provides local evidence of agreement between the surrogate and observed outcomes. \\
        \addlinespace
        Evaluation budget &
        Consumed evaluations $FE$ and remaining evaluations $FE_{\text{max}}-FE$. &
        Contextualizes the search evidence within the remaining optimization horizon. \\
        \bottomrule
    \end{tabular}
\end{table}

These categories distinguish what the search has achieved, what the surrogate predicts, and where evaluations have been concentrated. Their relationships are central to state analysis: similar best-fitness trajectories can coexist with different uncertainty patterns and sampling geometries. The SAG-guided state report makes these differences explicit, providing Multi-Agent Deliberation module with the context for synthesizing the next acquisition function.

\clearpage
\section{OnDesign for Evolutionary Continuous Optimization}
\label{app:ondesign_ec}

\textbf{Algorithm~\ref{alg:ondesign_ec}} summarizes OnDesign for evolutionary continuous optimization, with task instructions $\iota$ specified in Fig.~\ref{fig:prompt_ec_task}. At design stage $t$, $\pi_t$ generates an offspring population $\mathbb{O}_t$ from the parent population $\mathbb{P}_t$ and its fitness values; both populations contain $N$ solutions. Other notation follows Appendix~\ref{app:ondesign_bo}. One design stage corresponds to one generation and evaluates $N$ offspring, after which selection produces $\mathbb{P}_{t+1}$.

\begin{algorithm}[H]
\renewcommand{\algorithmicrequire}{\textbf{Input:}}
\renewcommand{\algorithmicensure}{\textbf{Output:}}
\caption{OnDesign for Evolutionary Continuous Optimization}
\label{alg:ondesign_ec}
\small
\begin{algorithmic}[1]
\Require $f$, $N$, $FE_{\text{max}}$, $L_{\mathrm{SAG}}$, and task instructions $\iota$
\Ensure Best solution $\textbf{x}^{*}$ and its fitness $f(\textbf{x}^{*})$
\State $t\leftarrow 0$; $[\mathbb{P}_t,FE]\leftarrow\mathit{Initialization}(N)$;
\State $\mathbb{DB}\leftarrow\mathbb{P}_t$; initialize $\phi_t$; $\mathbb{H}_t\leftarrow\emptyset$; $\mathcal{H}^{\mathrm{SAG}}_t\leftarrow\emptyset$;
\While{$FE<FE_{\text{max}}$}
    \State Collect runtime observations $o_t$;
    \State $q_t\leftarrow\mathcal{P}(o_t;\phi_t)$; \Comment{State Analyst}
    \State $\pi_t\leftarrow\mathit{Multi\text{-}Agent\_Deliberation}(q_t,\mathbb{H}_t;\iota)$;
    \State $best\_y \leftarrow\min_{\textbf{x}\in\mathbb{P}_t} f(\textbf{x})$;
    \State $\mathbb{O}_t\leftarrow\mathit{Offspring\_Generation}(\mathbb{P}_t,\pi_t)$;
    \State Evaluate the solutions in $\mathbb{O}_t$ by $f$;
    \State $\mathbb{DB}\leftarrow\mathbb{DB}\cup\mathbb{O}_t$ and $FE\leftarrow FE+N$;
    \State $\mathbb{P}_{t+1}\leftarrow\mathit{Environmental\_Selection}(\mathbb{P}_t\cup\mathbb{O}_t,N)$;
    \State $[\mathbb{H}_{t+1},r_t]\leftarrow\mathit{Feedback\_Update}(\mathbb{H}_t,\pi_t,best\_y,\mathbb{P}_{t+1})$;
    \State $\mathcal{H}^{\mathrm{SAG}}_{t+1}\leftarrow\mathit{SAG\_Feedback}(\mathcal{H}^{\mathrm{SAG}}_t,\phi_t,r_t)$;
    \State $\phi_{t+1}\leftarrow\mathit{SAG\_Update}(\phi_t,\mathcal{H}^{\mathrm{SAG}}_{t+1};t,L_{\mathrm{SAG}})$; \Comment{SAG Refiner}
    \State $t\leftarrow t+1$;
\EndWhile
\State $\textbf{x}^{*}\leftarrow\arg\min_{\textbf{x}\in\mathbb{P}_t} f(\textbf{x})$;
\end{algorithmic}
\end{algorithm}

Algorithm design precedes execution from $t=0$, with $t$ advanced once per generation and SAG refinement every $L_{\mathrm{SAG}}$ generations. The framework consists of five main steps:
\begin{enumerate}
    \item \textbf{Initialization} (Lines 1--2): With $t=0$, evaluate $N$ samples to form $\mathbb{P}_t$ and $\mathbb{DB}$, with $FE=N$. Generate $\phi_t$ with \emph{SAG Refiner} $\mathcal{E}$ and initialize both archives as in Appendix~\ref{app:ondesign_bo}. The first offspring-generation operator is synthesized inside the main loop.
    \item \textbf{Online Offspring-Generation Design} (Lines 4--6): \emph{State Analyst} $\mathcal{P}$ constructs $q_t$ from runtime observations $o_t$ under $\phi_t$ (Table~\ref{tab:ec_state_data_hint}). Multi-Agent Deliberation module synthesizes $\pi_t$ from $q_t$ and the algorithm descriptions and gains in $\mathbb{H}_t$.
    \item \textbf{Offspring Generation and Evaluation} (Lines 7--10): Record the best fitness in $\mathbb{P}_t$ as $best\_y$. Then, $\pi_t$ generates $N$ offspring from $\mathbb{P}_t$, forming $\mathbb{O}_t$. Their evaluation updates $\mathbb{DB}\leftarrow\mathbb{DB}\cup\mathbb{O}_t$ and $FE\leftarrow FE+N$.
    \item \textbf{Environmental Selection} (Line 11): Merge $\mathbb{P}_t$ and $\mathbb{O}_t$ and retain the best $N$ solutions in $\mathbb{P}_{t+1}$. Offspring solutions excluded from $\mathbb{P}_{t+1}$ remain in $\mathbb{DB}$ as evidence about previously explored regions.
    \item \textbf{Feedback and SAG Update} (Lines 12--14): Compute $r_t$ using Eq.~(\ref{eq:execution_gain}), with $b_t=best\_y$ and $b_{t+1}$ equal to the best fitness in $\mathbb{P}_{t+1}$. Update $\mathbb{H}_{t+1}\leftarrow\mathbb{H}_t\cup\{(\pi_t,r_t)\}$ and credit $r_t$ to $\phi_t$ through $\mathit{SAG\_Feedback}$. Then $\mathit{SAG\_Update}$ uses $\mathcal{H}^{\mathrm{SAG}}_{t+1}$ to generate $\phi_{t+1}$ when $(t+1)\bmod L_{\mathrm{SAG}}=0$, and otherwise retains $\phi_t$. Advance $t$ (Line 15); the resulting guideline is used in the next generation.
\end{enumerate}
When $FE=FE_{\text{max}}$, the best solution in the final $\mathbb{P}_t$ is returned. The shared prompts appear in Appendix~\ref{app:prompt_templates}.

\clearpage
\paragraph{State information.}
Offspring generation acts on a population whose fitness distribution, spatial structure, and capacity for further progress evolve jointly. At design stage $t$, OnDesign combines evidence about $\mathbb{P}_t$, its changes across generations, and empirical landscape characteristics inferred from $\mathbb{DB}$. Table~\ref{tab:ec_state_data_hint} summarizes these complementary sources of observations $o_t$ in \emph{State Data Types Hint} provided to \emph{SAG Refiner} $\mathcal{E}$. The resulting SAG $\phi_t$ guides \emph{State Analyst} $\mathcal{P}$ in constructing $q_t$, which describes both the population's current condition and the search dynamics that produced it.

\begin{table}[H]
    \centering
    \caption{State information summarized by \emph{State Data Types Hint} for evolutionary continuous optimization.}
    \label{tab:ec_state_data_hint}
    \small
    \renewcommand{\arraystretch}{1.10}
    \begin{tabular}{@{}>{\raggedright\arraybackslash}p{0.20\linewidth}@{\hspace{0.025\linewidth}}>{\raggedright\arraybackslash}p{0.48\linewidth}@{\hspace{0.025\linewidth}}>{\raggedright\arraybackslash}p{0.27\linewidth}@{}}
        \toprule
        \textbf{State information} & \textbf{Representative descriptors} & \textbf{Role in state analysis} \\
        \midrule
        Population quality &
        Best, worst, and mean fitness, fitness dispersion, and the location of the best individual in $\mathbb{P}_t$. &
        Characterizes the quality distribution of the current parent population. \\
        \addlinespace
        Population diversity &
        Pairwise distances, coordinate-wise dispersion, and coordinate ranges within $\mathbb{P}_t$. &
        Describes the spread and concentration of candidate solutions. \\
        \addlinespace
        Evolutionary progress &
        Best-fitness trajectory and absolute or normalized improvement across recent generations. &
        Reveals sustained improvement, diminishing gains, and stagnation. \\
        \addlinespace
        Population dynamics &
        Displacement of the population centroid and changes in coordinate-wise variance across generations. &
        Distinguishes population movement from contraction or expansion. \\
        \addlinespace
        Spatial coverage &
        Proximity of individuals to the domain boundary and spatial spread of historical evaluations. &
        Characterizes boundary concentration and the extent of sampled regions. \\
        \addlinespace
        Landscape structure &
        Estimated curvature, directional variation, and fitness--distance correlation from historical evaluations. &
        Relates observed fitness to the geometry of the sampled landscape. \\
        \addlinespace
        Landscape variation &
        Skewness and kurtosis of historical fitness values, together with fitness differences between nearby solutions. &
        Describes distributional asymmetry and local fitness variation. \\
        \addlinespace
        Dimensional structure &
        Effective dimensionality of evaluated samples and estimated coordinate importance for fitness. &
        Summarizes concentration in decision space and coordinate relevance. \\
        \addlinespace
        Evaluation budget &
        Consumed evaluations $FE$ relative to $FE_{\text{max}}$ and the remaining evaluation budget. &
        Contextualizes population dynamics within the remaining optimization horizon. \\
        \bottomrule
    \end{tabular}
\end{table}

Joint interpretation is especially important in population-based search. Similar fitness progress may accompany a concentrated population with little movement or a dispersed population traversing new regions. By relating progress to diversity, population movement, and empirical landscape evidence, the SAG-guided report provides Multi-Agent Deliberation module with a basis for designing offspring-generation operators suited to the current evolutionary state.

\clearpage
\section{OnDesign for Evolutionary Mixed-Variable Optimization}
\label{app:ondesign_emvo}

\textbf{Algorithm~\ref{alg:ondesign_emvo}} summarizes OnDesign for evolutionary mixed-variable optimization. Task instructions $\iota$ extend the interface in Fig.~\ref{fig:prompt_ec_task2} with variable types and domains. At design stage $t$, $\pi_t$ generates $N$ normalized offspring from $\mathbb{P}_t$ and its fitness values. Discrete variables are ordered integers on MV-BBOB test suite and unordered categories on EOPCCV test suite. Other notation follows Appendix~\ref{app:ondesign_ec}. Each generation evaluates up to $N$ offspring before selection produces $\mathbb{P}_{t+1}$.

\begin{algorithm}[H]
\renewcommand{\algorithmicrequire}{\textbf{Input:}}
\renewcommand{\algorithmicensure}{\textbf{Output:}}
\caption{OnDesign for Evolutionary Mixed-Variable Optimization}
\label{alg:ondesign_emvo}
\small
\begin{algorithmic}[1]
\Require $f$, $N$, $FE_{\text{max}}$, $L_{\mathrm{SAG}}$, and task instructions $\iota$
\Ensure Best solution $\textbf{x}^{*}$ and its fitness $f(\textbf{x}^{*})$
\State $t\leftarrow 0$; $[\mathbb{P}_t,FE]\leftarrow\mathit{Initialization}(N;\iota)$;
\State $\mathbb{DB}\leftarrow\mathbb{P}_t$; initialize $\phi_t$; $\mathbb{H}_t\leftarrow\emptyset$; $\mathcal{H}^{\mathrm{SAG}}_t\leftarrow\emptyset$;
\While{$FE<FE_{\text{max}}$}
    \State Collect mixed-variable runtime observations $o_t$;
    \State $q_t\leftarrow\mathcal{P}(o_t;\phi_t)$; \Comment{State Analyst}
    \State $\pi_t\leftarrow\mathit{Multi\text{-}Agent\_Deliberation}(q_t,\mathbb{H}_t;\iota)$;
    \State $best\_y \leftarrow\min_{\textbf{x}\in\mathbb{P}_t} f(\textbf{x})$;
    \State $\mathbb{O}_t\leftarrow\mathit{Offspring\_Generation}(\mathbb{P}_t,\pi_t)$;
    \State Evaluate the solutions in $\mathbb{O}_t$ by $f$;
    \State $\mathbb{DB}\leftarrow\mathbb{DB}\cup\mathbb{O}_t$ and $FE\leftarrow FE+|\mathbb{O}_t|$;
    \State $\mathbb{P}_{t+1}\leftarrow\mathit{Environmental\_Selection}(\mathbb{P}_t\cup\mathbb{O}_t,N)$;
    \State $[\mathbb{H}_{t+1},r_t]\leftarrow\mathit{Feedback\_Update}(\mathbb{H}_t,\pi_t,best\_y,\mathbb{P}_{t+1})$;
    \State $\mathcal{H}^{\mathrm{SAG}}_{t+1}\leftarrow\mathit{SAG\_Feedback}(\mathcal{H}^{\mathrm{SAG}}_t,\phi_t,r_t)$;
    \State $\phi_{t+1}\leftarrow\mathit{SAG\_Update}(\phi_t,\mathcal{H}^{\mathrm{SAG}}_{t+1};t,L_{\mathrm{SAG}})$; \Comment{SAG Refiner}
    \State $t\leftarrow t+1$;
\EndWhile
\State $\textbf{x}^{*}\leftarrow\arg\min_{\textbf{x}\in\mathbb{P}_t} f(\textbf{x})$;
\end{algorithmic}
\end{algorithm}

Algorithm design precedes execution from $t=0$, with $t$ advanced once per generation and SAG refinement every $L_{\mathrm{SAG}}$ generations. The framework consists of five main steps:
\begin{enumerate}
    \item \textbf{Initialization} (Lines 1--2): With $t=0$, sample continuous coordinates using a scrambled Sobol design and distribute each discrete variable's values as evenly as possible across $N$ individuals. Evaluate these samples to form $\mathbb{P}_t$ and $\mathbb{DB}$, with $FE=N$. Generate $\phi_t$ with the SAG Refiner $\mathcal{E}$ and initialize both archives as in Appendix~\ref{app:ondesign_bo}.
    \item \textbf{Online Offspring-Generation Design} (Lines 4--6): \emph{State Analyst} $\mathcal{P}$ constructs $q_t$ from mixed-variable population and runtime observations $o_t$ under $\phi_t$ (Table~\ref{tab:emvo_state_data_hint}). Multi-Agent Deliberation module synthesizes $\pi_t$ from $q_t$ and the algorithm descriptions and gains in $\mathbb{H}_t$, without past states or reports.
    \item \textbf{Offspring Generation and Evaluation} (Lines 7--10): Record the current best fitness as $best\_y$ and generate $N$ normalized offspring $\mathbb{O}_t$ using $\pi_t$. The external routine clips coordinates to $[0,1]$ and decodes continuous values by rescaling, categories by equal-width bins, and ordered integers by rescaling and rounding. Discrete encodings are canonicalized for storage. Evaluate only the offspring allowed by the remaining budget, updating $\mathbb{DB}\leftarrow\mathbb{DB}\cup\mathbb{O}_t$ and $FE\leftarrow FE+|\mathbb{O}_t|$.
    \item \textbf{Environmental Selection} (Line 11): Merge $\mathbb{P}_t$ and $\mathbb{O}_t$ and retain the best $N$ solutions in $\mathbb{P}_{t+1}$. Offspring solutions excluded from $\mathbb{P}_{t+1}$ remain in $\mathbb{DB}$ as evidence about previously explored regions and discrete configurations.
    \item \textbf{Feedback and SAG Update} (Lines 12--14): Compute $r_t$ using Eq.~(\ref{eq:execution_gain}), with $b_t=best\_y$ and $b_{t+1}$ equal to the best fitness in $\mathbb{P}_{t+1}$. Update $\mathbb{H}_{t+1}\leftarrow\mathbb{H}_t\cup\{(\pi_t,r_t)\}$ and credit $r_t$ to $\phi_t$ through $\mathit{SAG\_Feedback}$. Then $\mathit{SAG\_Update}$ uses $\mathcal{H}^{\mathrm{SAG}}_{t+1}$ to generate $\phi_{t+1}$ when $(t+1)\bmod L_{\mathrm{SAG}}=0$, and otherwise retains $\phi_t$. Advance $t$ (Line 15); the resulting guideline is used in the next generation.
\end{enumerate}
When $FE=FE_{\text{max}}$, the best solution in the final $\mathbb{P}_t$ is returned. The shared prompts appear in Appendix~\ref{app:prompt_templates}.

\clearpage
\paragraph{State information.}
Offspring generation acts on a population whose continuous coordinates and discrete configurations evolve jointly. At design stage $t$, OnDesign combines evidence about $\mathbb{P}_t$, its changes across generations, and empirical landscape characteristics inferred from $\mathbb{DB}$, while respecting the variable types specified in $\iota$. Table~\ref{tab:emvo_state_data_hint} organizes these observations $o_t$ using the same categories as Table~\ref{tab:ec_state_data_hint}, adapted to mixed-variable search, and summarizes \emph{State Data Types Hint} provided to \emph{SAG Refiner} $\mathcal{E}$. The resulting SAG $\phi_t$ guides \emph{State Analyst} $\mathcal{P}$ in constructing $q_t$, which describes both the population's current condition and its type-specific search dynamics.

\begin{table}[H]
    \centering
    \caption{State information summarized by \emph{State Data Types Hint} for evolutionary mixed-variable optimization.}
    \label{tab:emvo_state_data_hint}
    \small
    \renewcommand{\arraystretch}{1.10}
    \begin{tabular}{@{}>{\raggedright\arraybackslash}p{0.20\linewidth}@{\hspace{0.025\linewidth}}>{\raggedright\arraybackslash}p{0.48\linewidth}@{\hspace{0.025\linewidth}}>{\raggedright\arraybackslash}p{0.27\linewidth}@{}}
        \toprule
        \textbf{State information} & \textbf{Representative descriptors} & \textbf{Role in state analysis} \\
        \midrule
        Population quality &
        Best, worst, and mean fitness, fitness dispersion, and the location of the best individual in $\mathbb{P}_t$. &
        Characterizes the quality distribution of the current parent population. \\
        \addlinespace
        Population diversity &
        Mixed pairwise distances, continuous dispersion and ranges, discrete-value frequencies and entropy, distinct discrete combinations, and decoded-decision duplicate rate. &
        Describes continuous spread and discrete concentration without imposing an order on categorical labels. \\
        \addlinespace
        Evolutionary progress &
        Best-fitness trajectory and absolute or normalized improvement across recent generations. &
        Reveals sustained improvement, diminishing gains, and stagnation. \\
        \addlinespace
        Population dynamics &
        Displacement of the continuous population centroid, changes in continuous variance, and changes in discrete-value frequencies across generations. &
        Distinguishes continuous movement from redistribution among discrete values. \\
        \addlinespace
        Spatial coverage &
        Proximity of individuals to continuous-domain boundaries and coverage of allowed discrete values. &
        Characterizes boundary concentration and underexplored discrete values. \\
        \addlinespace
        Landscape structure &
        Fitness--distance correlation using mixed distances and estimated continuous curvature within the best observed discrete configuration, when sufficient data are available. &
        Relates fitness to sampled mixed geometry and conditional local structure. \\
        \addlinespace
        Landscape variation &
        Skewness and kurtosis of historical fitness values, together with fitness differences between mixed-space neighbors and within a shared discrete configuration. &
        Describes distributional asymmetry and local fitness variation. \\
        \addlinespace
        Dimensional structure &
        Principal-component spread of continuous coordinates and estimated variable importance using type-aware inputs. &
        Summarizes concentration in continuous space and empirical variable relevance. \\
        \addlinespace
        Evaluation budget &
        Consumed evaluations $FE$ relative to $FE_{\text{max}}$ and the remaining evaluation budget. &
        Contextualizes population dynamics within the remaining optimization horizon. \\
        \bottomrule
    \end{tabular}
\end{table}

Joint interpretation is especially important in mixed-variable search. Similar fitness progress may accompany a population spread across continuous regions but concentrated on a few discrete configurations, or one covering many discrete values with little continuous refinement. By relating progress to type-specific diversity, population movement, and empirical landscape evidence, the SAG-guided report provides Multi-Agent Deliberation module with a basis for designing offspring-generation operators suited to the current mixed-variable optimization state.

\clearpage
\section{Prompt Templates}
\label{app:prompt_templates}

This appendix presents eight prompt templates: three task-specific instruction templates and five agent-role templates shared by Algorithms~\ref{alg:ondesign_bo}--\ref{alg:ondesign_emvo}. The task-specific instructions define the optimization scenario and the interface of $\pi_t$, while the role prompts support State Analysis, Multi-Agent Deliberation, and SAG Evolution. Braced fields are placeholders for task descriptions, runtime observations, state reports, or archived feedback. The explanations below follow the terminology and notation of Sec.~\ref{sec:framework}; the displayed prompts retain their original wording.

\subsection{Task-Specific Instructions}

Figures~\ref{fig:prompt_bo_task}--\ref{fig:prompt_ec_task2} present the task instructions $\iota$ for applying OnDesign to Bayesian optimization, evolutionary continuous optimization, and evolutionary mixed-variable optimization, respectively. In Bayesian optimization, $\pi_t$ maps the incumbent fitness and GP posterior mean and standard deviation to acquisition utilities, which are maximized to select the next candidate. In both evolutionary scenarios, $\pi_t$ maps the parent population $\mathbb{P}_t$ and its fitness values to an offspring population $\mathbb{O}_t$.

\begin{figure}[H]
    \centering
    \includegraphics[width=0.75\linewidth]{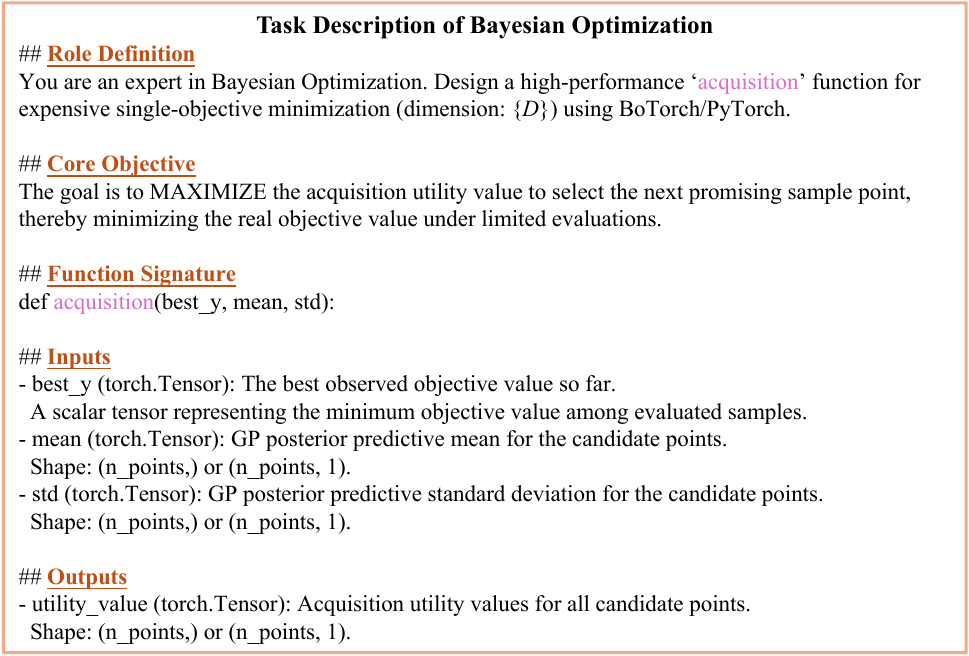}
    \caption{Task-description prompt for applying OnDesign to Bayesian optimization.}
    \label{fig:prompt_bo_task}
\end{figure}

\begin{figure}[H]
    \centering
    \includegraphics[width=0.75\linewidth]{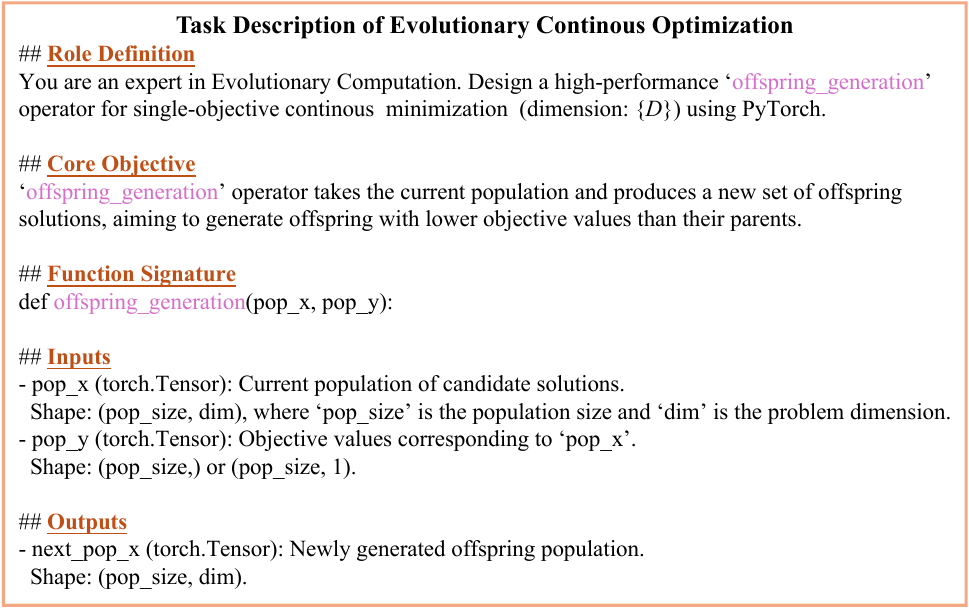}
    \caption{Task-description prompt for applying OnDesign to evolutionary continuous optimization.}
    \label{fig:prompt_ec_task}
\end{figure}

\begin{figure}[H]
    \centering
    \includegraphics[width=0.75\linewidth]{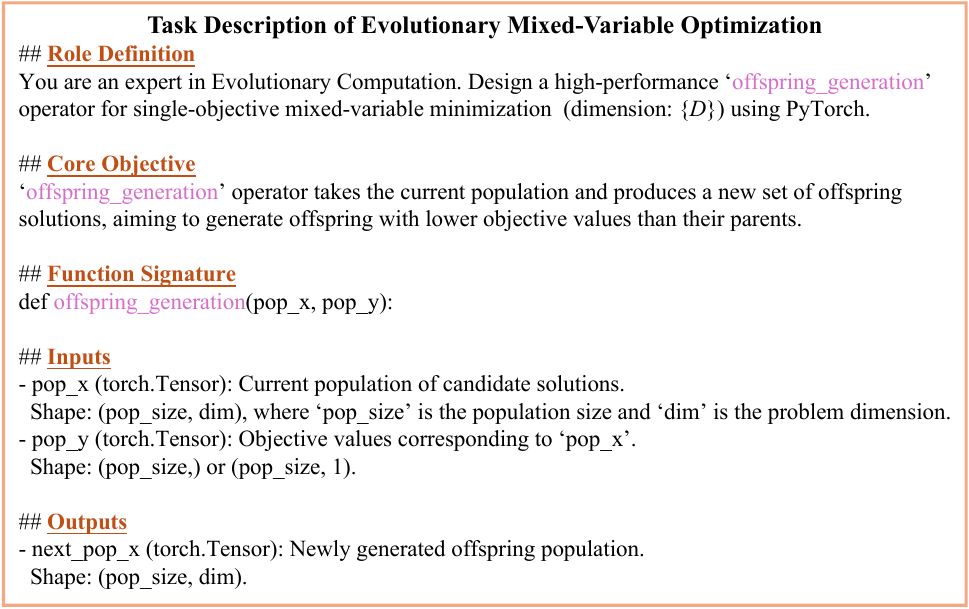}
    \caption{Task-description prompt for applying OnDesign to evolutionary mixed-variable optimization.}
    \label{fig:prompt_ec_task2}
\end{figure}

\subsection{State Analysis}

Figure~\ref{fig:prompt_state_analysis} presents the prompt for \emph{State Analyst} $\mathcal{P}$, with fields for the task instructions $\iota$, runtime observations $o_t$, and current SAG $\phi_t$. It requests a structured state report $q_t$ grounded in the supplied observations and organized according to $\phi_t$. The role instruction emphasizes neutral analysis; algorithmic implications in generated reports are assessed during Multi-Agent Deliberation.

\begin{figure}[H]
    \centering
    \includegraphics[width=0.75\linewidth]{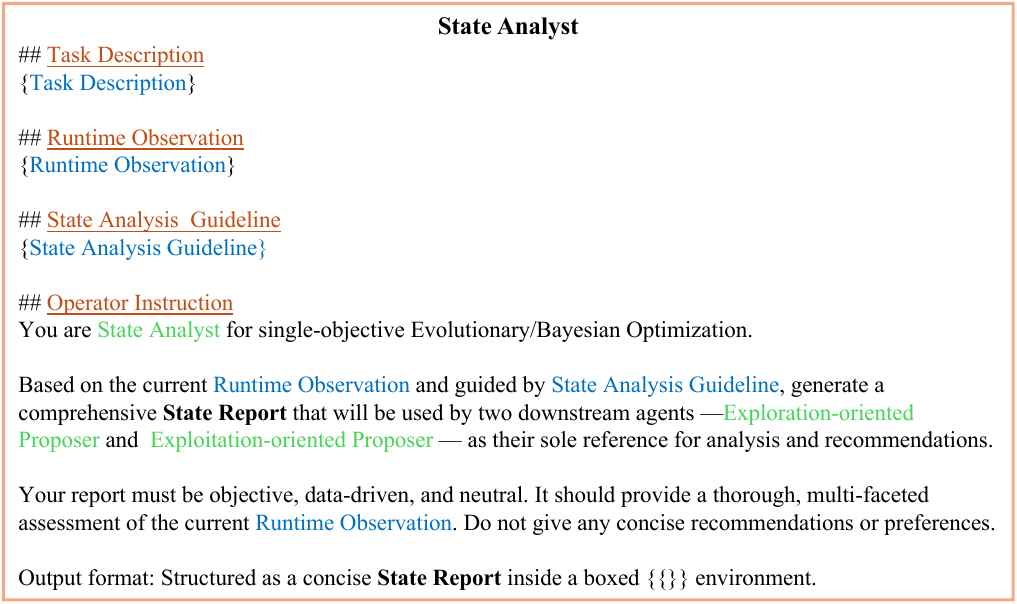}
    \caption{\emph{State Analyst} prompt for State Analysis. The current SAG guides the construction of the state report.}
    \label{fig:prompt_state_analysis}
\end{figure}

\subsection{Multi-Agent Deliberation}

Figures~\ref{fig:prompt_exploration} and~\ref{fig:prompt_exploitation} present the exploration-oriented and exploitation-oriented prompts for \emph{Proposers} $\mathcal{L}_1$ and $\mathcal{L}_2$, respectively. Across all three optimization scenarios, both prompts receive the same task instructions $\iota$, state report $q_t$, and historical entries containing algorithm descriptions and execution gains from $\mathbb{H}_t$.

The exploration-oriented prompt emphasizes uncertainty, insufficient coverage, diversity loss, and stagnation to encourage broader search, while the exploitation-oriented prompt emphasizes solution quality, promising regions, and recent improvement to encourage focused refinement. Both request three output fields corresponding to \emph{Design Direction} $v_t^{(k)}$, \emph{Rationale} $e_t^{(k)}$, and \emph{Design Advice} $d_t^{(k)}$ in Sec.~\ref{sec:framework}. The design direction specifies a \emph{Mild}, \emph{Moderate}, or \emph{Aggressive} preference for exploration or exploitation; executable code is not requested.

\begin{figure}[H]
    \centering
    \includegraphics[width=0.75\linewidth]{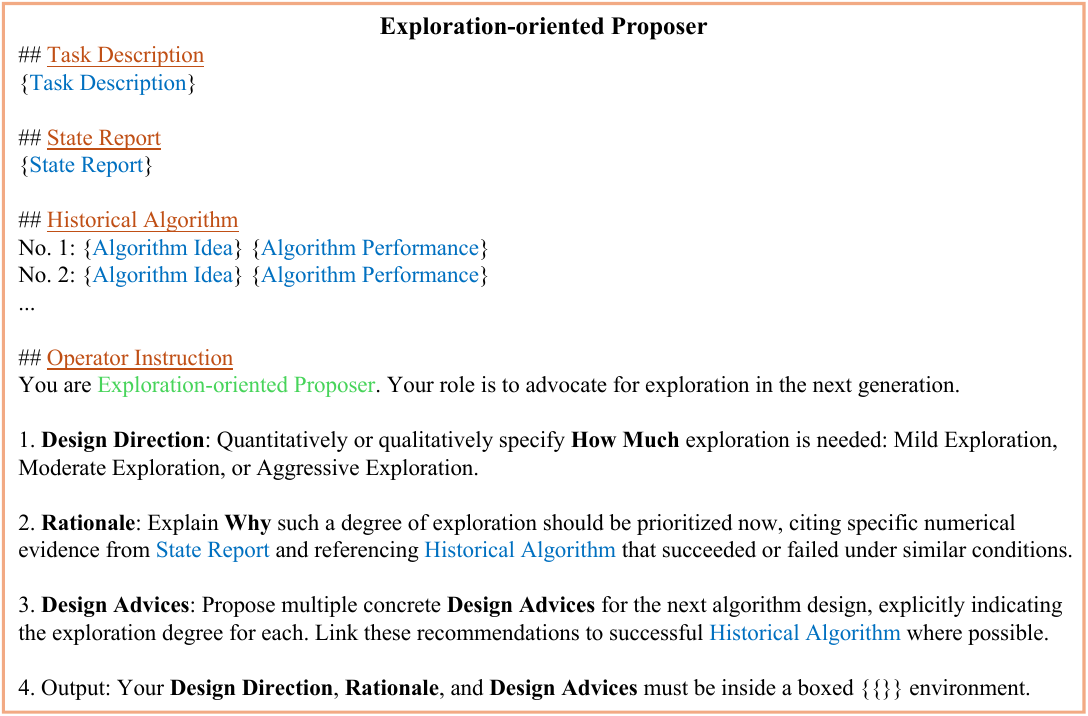}
    \caption{\emph{Exploration-oriented Proposer} prompt for Multi-Agent Deliberation.}
    \label{fig:prompt_exploration}
\end{figure}

\begin{figure}[H]
    \centering
    \includegraphics[width=0.75\linewidth]{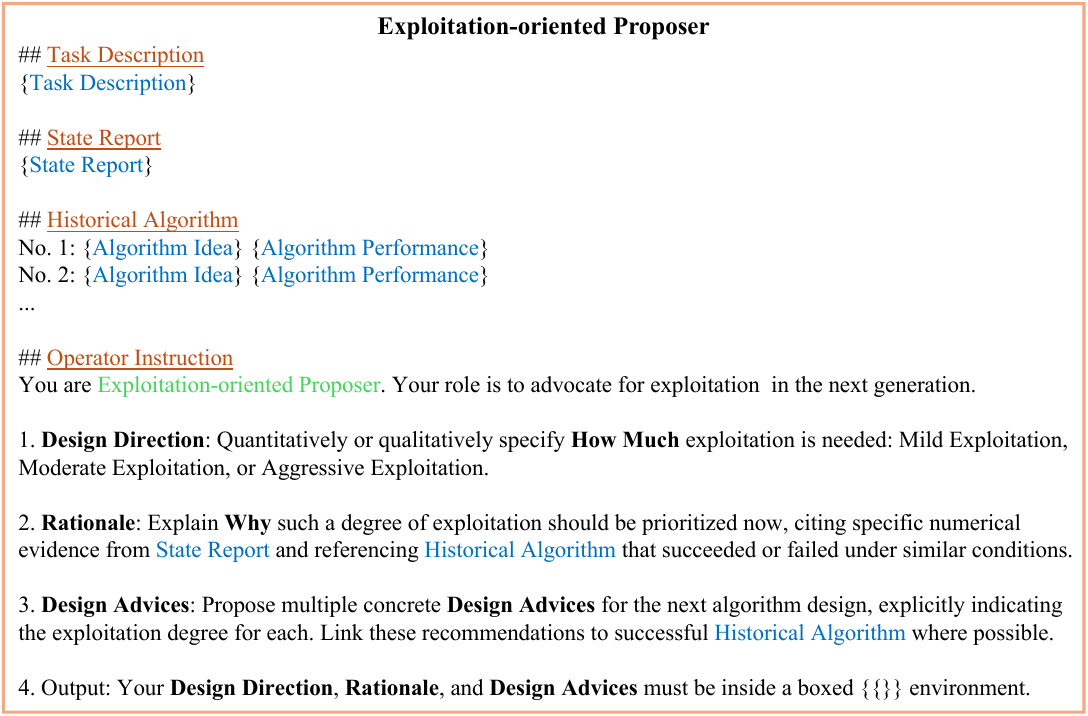}
    \caption{\emph{Exploitation-oriented Proposer} prompt for Multi-Agent Deliberation.}
    \label{fig:prompt_exploitation}
\end{figure}

\paragraph{Deliberation and program synthesis.}
Figure~\ref{fig:prompt_arbiter} presents the prompt for \emph{Arbiter} $\mathcal{G}$. Its inputs comprise both proposals, the task instructions $\iota$, the state report $q_t$, and algorithm descriptions and execution gains from $\mathbb{H}_t$. The prompt requests two outputs: a concise design idea and complete Python code for $\pi_t$, implementing the task-specific interface defined by $\iota$.

\subsection{SAG Evolution}
Figure~\ref{fig:prompt_sag_refiner} presents the prompt for \emph{SAG Refiner} $\mathcal{E}$. The task instructions $\iota$ establish the optimization context, while  \emph{State Data Types Hint} specifies the categories of runtime evidence available to State Analysis (Tables~\ref{tab:bo_state_data_hint}--\ref{tab:emvo_state_data_hint}). Historical entries pair each SAG with a \emph{Performance} value, corresponding to its aggregated score in Eq.~(\ref{eq:sag_feedback}). These pairs provide feedback for revising which signals to emphasize and how to interpret them. The output is a natural-language guideline for subsequent State Analysis, rather than an executable optimization algorithm.

\begin{figure}[H]
    \centering
    \includegraphics[width=0.75\linewidth]{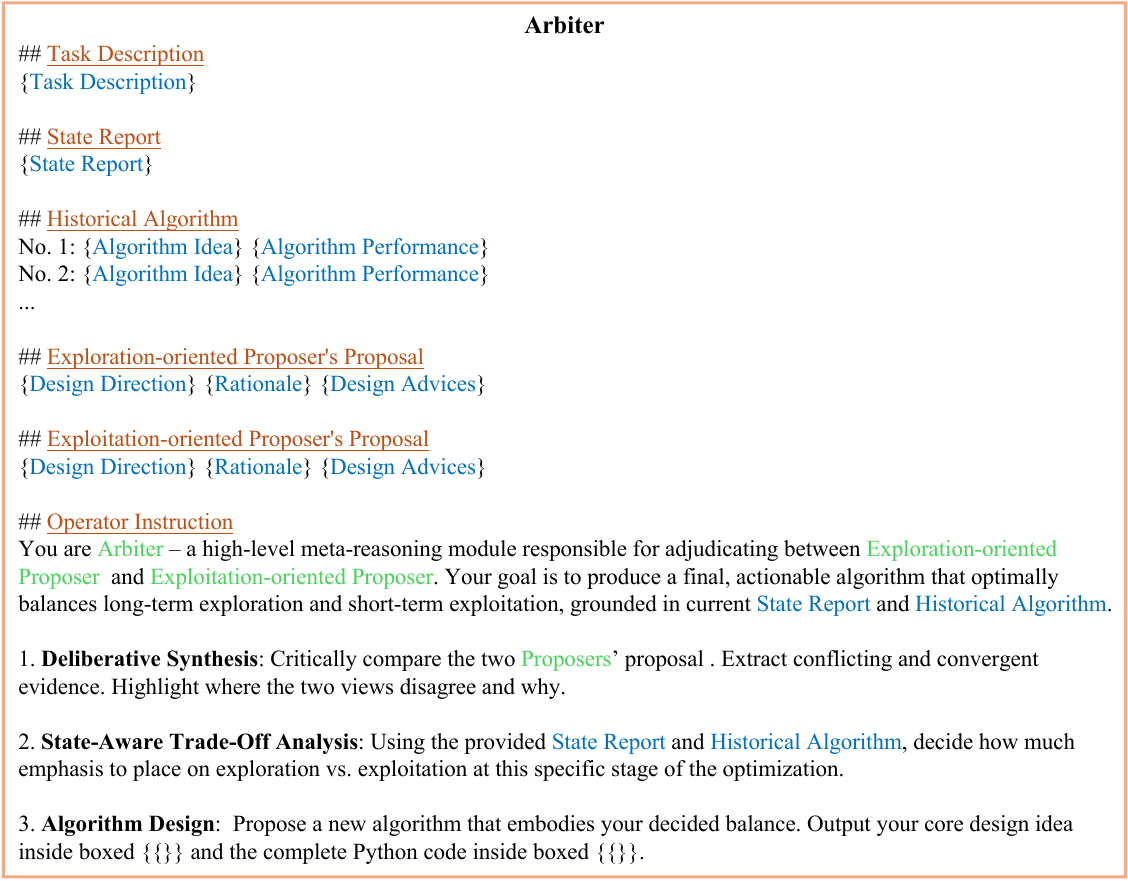}
    \caption{\emph{Arbiter} prompt for Multi-Agent Deliberation. The two proposals are synthesized into a task-specific executable algorithm.}
    \label{fig:prompt_arbiter}
\end{figure}

\begin{figure}[H]
    \centering
    \includegraphics[width=0.75\linewidth]{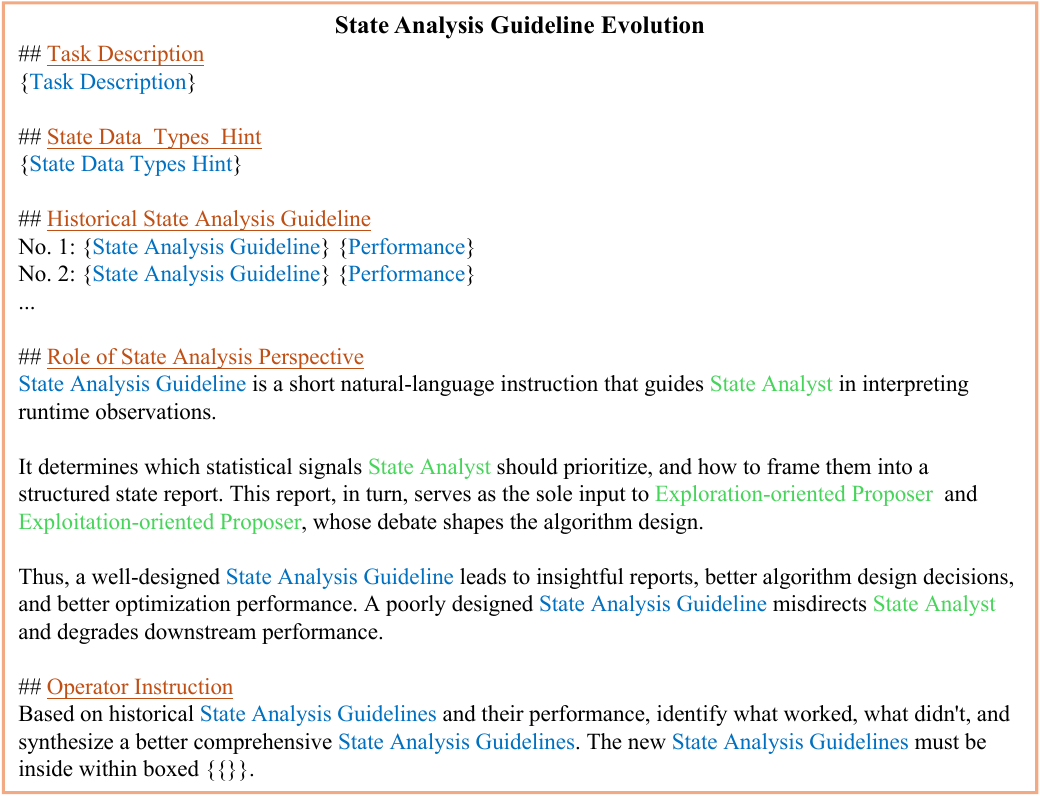}
    \caption{\emph{SAG Refiner} prompt, conditioned on available state-data types and scored SAG history.}
    \label{fig:prompt_sag_refiner}
\end{figure}

\clearpage
\section{Construction of the MV-BBOB Test Suite}
\label{app:mv_bbob}

We construct MV-BBOB from the 24 continuous BBOB functions by restricting selected coordinates to integer values. For a fixed instance of function $i\in\{1,\ldots,24\}$ in dimension $D\ge 2$, let $f_i$ denote the original objective, $[\ell_j,u_j]$ its bounds for coordinate $j$, and $\mathcal{J}_i$ the set of integer-coordinate indices. Evaluation applies a coordinate-wise mapping before calling the original objective:
\begin{equation}
\begin{aligned}
    f_i^{\mathrm{MV}}(\textbf{x}) &= f_i\!\left(R_i(\textbf{x})\right),\\
    [R_i(\textbf{x})]_j & =
    \begin{cases}
        \operatorname{clip}\!\left(\operatorname{round}(x_j),\ell_j,u_j\right),
        & j\in\mathcal{J}_i,\\
        x_j, & j\notin\mathcal{J}_i.
    \end{cases}
\end{aligned}
\label{eq:mv_bbob_mapping}
\end{equation}
Here, $\operatorname{round}$ rounds to the nearest integer, with ties to the even integer, and $\operatorname{clip}$ enforces the inherited bounds. Thus, integer coordinates use unit spacing, while the remaining coordinates retain their continuous domains. The underlying objective implementation and instance parameters are unchanged.

The target proportion $\rho_i$ of integer variables is fixed for each function, as listed in Table~\ref{tab:mv_bbob_ratios}. The actual number of integer coordinates is $k_i=\max\{1,\min\{D-1,\operatorname{round}(\rho_i D)\}\}$, ensuring that both variable types are present. Using one-based coordinate indices, the implementation selects
\begin{equation}
    \mathcal{J}_i=
    \left\{1+m\left\lfloor\frac{D}{k_i}\right\rfloor
    \;\middle|\;m=0,\ldots,k_i-1\right\}.
\label{eq:mv_bbob_indices}
\end{equation}
This deterministic rule keeps the variable-type assignment fixed across instances and runs for a given function and dimension.

\begin{table}[H]
    \centering
    \caption{Function-specific target proportions of integer variables in MV-BBOB. The actual counts follow the rounding rule for $k_i$.}
    \label{tab:mv_bbob_ratios}
    \small
    \begin{tabular}{cl}
        \toprule
        \textbf{Target proportion $\rho_i$} & \textbf{BBOB function indices $i$} \\
        \midrule
        $0.15$ & $9,21,22$ \\
        $0.20$ & $6,8,10,15,19,20$ \\
        $0.30$ & $12,18$ \\
        $0.40$ & $3,4,11,16,17,23,24$ \\
        $0.50$ & $2$ \\
        $0.60$ & $7,13,14$ \\
        $0.70$ & $1,5$ \\
        \bottomrule
    \end{tabular}
\end{table}

The resulting suite combines continuous variables with ordered integers. Varying the integer-variable proportion across functions yields diverse continuous--integer compositions, broadening the range of mixed-variable settings covered by the test suite. Because integer constraints restrict the feasible inputs, the optimum of the original continuous problem is not assumed to remain feasible.

\clearpage
\section{Experimental Details}
\label{app:experimental_details}

\subsection{Baselines}
\label{app:benchmarks_baselines}

We group the baselines into 1) fixed-logic optimizers, 2) adaptive optimizers, and 3) LLM-designed algorithms. The LLM-designed group contains algorithms reported in prior studies and scenario-specific algorithms generated offline for each setting. EoH~\citep{liu2024eoh} co-evolves natural-language design ideas and executable code; MEoH~\citep{yao2025meoh} formulates heuristic discovery as a multi-objective search and manages candidates using dominance and code dissimilarity; and ReEvo~\citep{ye2024reevo} combines evolutionary search with LLM-generated reflections that guide subsequent heuristic revisions. In our experiments, these methods perform offline algorithm design in advance on problems from the corresponding target setting, producing a scenario-specific algorithm for comparison with OnDesign.

\paragraph{Bayesian optimization.}
LogEI~\citep{ament2023unexpected} and LogPI~\citep{ament2023unexpected} are log-space reformulations of improvement-based acquisition functions that improve numerical stability. UCB~\citep{srinivas2009gaussian} balances the surrogate posterior mean and uncertainty through an upper confidence bound, while JES~\citep{hvarfner2022joint} selects evaluations by their information gain about the joint distribution of the optimal input and value; these four methods form the fixed-logic group. SETUP-BO~\citep{vasconcelos2022self}, the adaptive baseline, uses a portfolio of acquisition functions and Thompson sampling to tune its portfolio hyperparameters. Among the LLM-designed algorithms reported in prior studies, FunBO~\citep{aglietti2024funbo} discovers acquisition functions through FunSearch~\citep{romera2024mathematical}-based program search, ATRBO-LCB~\citep{li2026llamea} combines an LCB acquisition function with adaptive trust-region search, and TREvol~\citep{li2026llamea} integrates trust-region search with dynamic kernels, acquisition blending, and evolutionary candidate generation. EoH~\citep{liu2024eoh}, MEoH~\citep{yao2025meoh}, and ReEvo~\citep{ye2024reevo} provide the three scenario-specific baselines for Bayesian optimization.

\paragraph{Evolutionary continuous optimization.}
BSPGA~\citep{su2020non} uses a binary space-partition tree to record explored solutions, avoid revisits, and guide population refinement. CLPSO~\citep{liang2006comprehensive} constructs particle exemplars from the historical best positions of multiple particles to preserve swarm diversity. CoDE~\citep{wang2011differential} generates trial vectors by combining three mutation strategies with three parameter settings; these methods constitute the fixed-logic group. Within the adaptive group, SAHLPSO~\citep{tao2021self} assigns particles exploration and exploitation roles and adapts their learning parameters from success history, while SHADE~\citep{tanabe2013success} maintains memories of successful differential-evolution parameters. AutoEP~\citep{xu2026autoep} uses multi-LLM reasoning to derive hyperparameter adaptation strategies. ParEvo1 and ParEvo2 are two algorithms reported by PartEvo~\citep{hu2025partition}, whose design search combines feature-assisted niches with collaborative evolution. EoH~\citep{liu2024eoh}, MEoH~\citep{yao2025meoh}, and ReEvo~\citep{ye2024reevo} provide the scenario-specific baselines for evolutionary continuous optimization.

\paragraph{Evolutionary mixed-variable optimization.}
CoDE~\citep{wang2011differential} uses the composite differential-evolution mechanism described above. MDE-IHS~\citep{liao2010two} couples differential evolution with improved harmony search, whereas $\mathrm{DE}_{\mathrm{MV}}$~\citep{lin2018hybrid} co-evolves continuous variables with standard DE and discrete variables with set-based DE. Sig-DE~\citep{yu2016stock} employs a sigmoid transformation to map differential-evolution search into discrete decisions, and IMI-GWO~\citep{gupta2019efficient} augments grey-wolf optimization with opposition-based learning and chaotic local search; together with CoDE, they form the fixed-logic group. SHADE~\citep{tanabe2013success} uses success-history parameter memories as described above, while $\mathrm{PSO}_{\mathrm{MV}}$~\citep{wang2021particle} combines mixed-variable encoding and type-specific reproduction with adaptive parameter tuning; these are the adaptive baselines. EoH~\citep{liu2024eoh}, MEoH~\citep{yao2025meoh}, and ReEvo~\citep{ye2024reevo} provide the scenario-specific offline AAD baselines for evolutionary mixed-variable optimization.

\subsection{Experimental Configurations and Evaluation Protocol}
\label{app:evaluation_protocol}

\paragraph{Shared evaluation environment.}
OnDesign and the baselines use the same benchmark objectives, variable domains, and fitness-evaluation budgets within each scenario. All objectives are minimized. Table~\ref{tab:default_experimental_configurations} summarizes the common evaluation settings and run counts. The total evaluation budget includes initialization and all subsequent calls to the benchmark objective. Apart from the shared settings specified below, all baseline hyperparameters retain their default values in the corresponding implementations.

For OnDesign, we use $L_{\mathrm{SAG}}=10$ as the default SAG update interval across all three optimization scenarios, refining the guideline after every ten completed design stages. An illustrative sensitivity analysis of this interval is provided in Appendix~\ref{app:sag_sensitivity}.

\begin{table}[H]
    \centering
    \caption{Common evaluation settings and run counts. Evaluation budgets include initialization; the Bayesian optimization budget specifies evaluations after initialization.}
    \label{tab:default_experimental_configurations}
    \small
    \setlength{\tabcolsep}{4pt}
    \begin{tabular}{lccc}
        \toprule
        \textbf{Scenario} & \textbf{Initial setting} & \textbf{Evaluation budget} & \textbf{Runs} \\
        \midrule
        Bayesian optimization
        & $2D+1$ Sobol samples
        & 100 additional evaluations
        & 10 \\
        \makecell[l]{Evolutionary continuous\\optimization}
        & $N=100$
        & 5,000 total evaluations
        & 10 \\
        \makecell[l]{Evolutionary mixed-variable\\optimization}
        & $N=100$
        & 5,000 total evaluations
        & 10 \\
        \bottomrule
    \end{tabular}
\end{table}

\paragraph{Bayesian optimization.}
We evaluate the 24 noiseless BBOB functions~\citep{elhara2019coco} at $D\in\{10,30\}$, covering separable, ill-conditioned, and multimodal landscapes. CEC2020~\citep{yue2019cec2020} and CEC2022~\citep{kumar2021cec2022} contain 10 and 12 functions, respectively, including basic, hybrid, and composition functions with shifts and rotations; both are evaluated at $D\in\{10,20\}$. For CEC2026~\citep{Chen2026CEC}, we use the 29-function single-objective benchmark at $D=30$, which includes unimodal, multimodal, hybrid, and composition landscapes. Each method begins with $2D+1$ scrambled Sobol samples and performs 100 subsequent evaluations, giving total budgets of 121, 141, and 161 for $D=10$, 20, and 30, respectively. Search inputs are represented in $[0,1]^D$ and mapped to the benchmark bounds before objective evaluation. The common surrogate is a single-task Gaussian process with a scaled Mat\'ern-$5/2$ kernel, fitted to standardized objective observations. Evaluations after initialization are sequential, with one candidate evaluated at a time.

\paragraph{Evolutionary continuous optimization.}
We use the same four continuous benchmark suites and seven suite--dimension configurations as in Bayesian optimization, evaluating population-based search on the same range of landscape characteristics. The common population-size parameter is $N=100$, and the total budget is 5,000 fitness evaluations, including the initial population. All candidates are evaluated within the bounds specified by the corresponding benchmark. The stopping criterion is the number of objective evaluations rather than the number of generations, so methods with different evaluation requirements per generation are compared under the same total budget.

\paragraph{Evolutionary mixed-variable optimization.}
MV-BBOB extends the 24 BBOB functions by restricting function-specific proportions of coordinates to ordered integers, retaining continuous coordinates and the underlying landscape diversity. We evaluate it at $D\in\{10,30\}$. EOPCCV~\citep{9464165} contains 30 ten-dimensional problems combining continuous variables with unordered categories: each problem has 2, 5, or 8 categorical variables, each admitting 5 or 10 labels. These configurations vary both the categorical-variable proportion and the domain cardinality. As in continuous optimization, the common population-size parameter is $N=100$, and the total budget is 5,000 evaluations, including initialization. All methods use the same variable types and admissible domains for each problem. Fitness evaluation uses continuous values within the specified bounds and admissible integer values or category labels for discrete coordinates. The construction and integer-variable proportions of MV-BBOB are described in Appendix~\ref{app:mv_bbob}.

\paragraph{Repeated runs and performance summaries.}
OnDesign and all baselines are evaluated over ten independent runs per problem configuration. Independent random seeds are used across runs. For each problem, we record the best objective value reached by the end of the prescribed budget. Function-wise tables report means and standard deviations over independent runs. Because objective scales differ across functions, each method is ranked separately on each function and the resulting ranks are averaged within a suite--dimension configuration. Baseline entries in the summary tables additionally report $+/-/=$ counts, denoting the numbers of functions on which the baseline is better, worse, or tied relative to OnDesign. All ranks are interpreted within an optimization scenario because the compared method sets differ across scenarios.

%

\subsection{Sensitivity to the SAG Update Interval}
\label{app:sag_sensitivity}

The SAG update interval $L_{\mathrm{SAG}}$ specifies the number of completed design stages between guideline refinements; $L_{\mathrm{SAG}}=\infty$ denotes retaining the initial SAG throughout optimization. We use $L_{\mathrm{SAG}}=10$ consistently across the three optimization settings. Figure~\ref{fig:sag_sensitivity} illustrates different update frequencies in the Bayesian optimization setting on BBOB F22 at $D=30$. Its adjusted illustrative values visualize the scheduling trade-off rather than establish an optimal interval: frequent refinements can respond sooner to changes in the optimization state but draw on less new execution evidence per update, whereas infrequent refinements accumulate more feedback but may react later. The fixed, moderate interval of 10 balances these considerations in our experimental protocol.

\begin{figure}[H]
    \centering
    \includegraphics[width=0.65\linewidth]{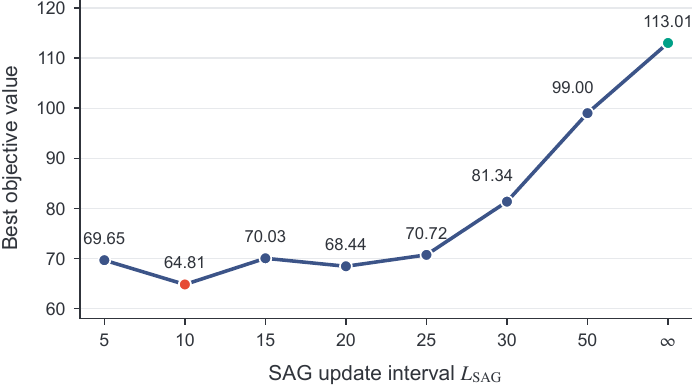}
    \caption{Illustration of SAG update frequencies in Bayesian optimization on BBOB F22 at $D=30$. The plotted objective values are adjusted illustrations, not a controlled comparison; $\infty$ denotes no SAG refinement. Lower is better.}
    \label{fig:sag_sensitivity}
\end{figure}

\subsection{Engineering Validation}
\label{app:engineering_problem}

\paragraph{Background.}
Frontal collisions require vehicles to dissipate impact energy while limiting the loads transmitted to the passenger compartment. The vehicle's front-end structure addresses this requirement through controlled deformation of its energy-absorbing members. Within this structure, thin-walled energy-absorbing boxes dissipate energy through progressive crushing. Multicell configurations distribute deformation across multiple walls, whose thickness allocation affects both energy absorption and component mass. This trade-off motivates maximizing specific energy absorption (SEA), the energy absorbed per unit mass. We consider the multicell energy-absorbing box studied by \citet{zhang2024two}, adopting its geometry and wall-thickness parameterization for a single-objective optimization task. Figure~\ref{fig:engineering_problem} shows its position in the vehicle front-end structure and the spatial distribution of its design variables.

\begin{figure}[H]
    \centering
    \includegraphics[width=\linewidth]{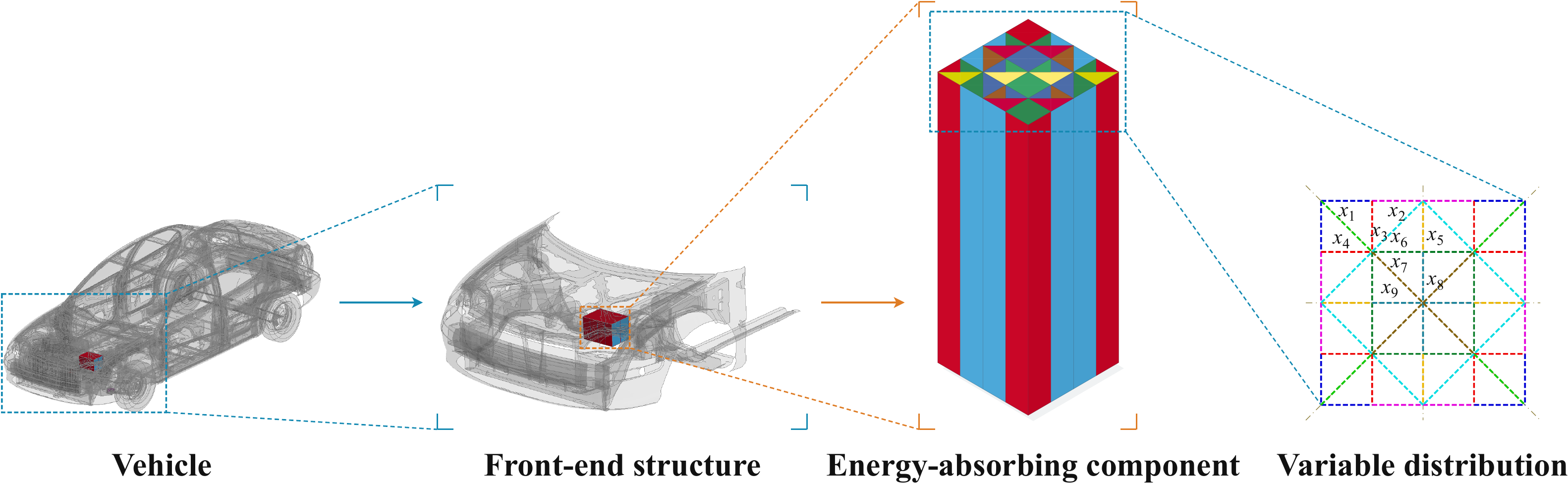}
    \caption{Multicell automotive energy-absorbing component and its nine wall-thickness variables. The cross-sectional layout identifies the independently parameterized wall groups under one-eighth symmetry.}
    \label{fig:engineering_problem}
\end{figure}

\paragraph{Design variables.}
The component has one-eighth symmetry, allowing its wall-thickness distribution to be specified by nine independent continuous variables,
\begin{equation}
    \mathbf{x}=(x_1,\ldots,x_9)^{\top},
    \qquad
    0.2 \le x_i \le 3.0~\mathrm{mm},
    \quad i=1,\ldots,9.
\end{equation}
Each $x_i$ controls the thickness of the corresponding wall group indicated in Figure~\ref{fig:engineering_problem}; symmetry-related walls share the same thickness. The optimization therefore reallocates material among the wall groups within a fixed geometric configuration. The remaining model settings, including material properties and loading conditions, are held fixed across candidate designs.

\paragraph{Problem definition.}
Let ${EA}(\mathbf{x})$ denote the energy absorbed during the simulated crushing response, measured in kJ, and let $M(\mathbf{x})$ denote the component mass, measured in kg. The specific energy absorption (SEA) is
\begin{equation}
    \emph{{SEA}}(\mathbf{x})
    =
    \frac{\emph{{EA}}(\mathbf{x})}{M(\mathbf{x})},
\end{equation}
with units of kJ/kg. Our objective is to maximize \emph{SEA}:
\begin{equation}
\begin{aligned}
    \max_{\mathbf{x}\in\mathbb{R}^{9}}\quad
    &{SEA}(\mathbf{x})\\
    \text{subject to}\quad
    &0.2 \le x_i \le 3.0~\mathrm{mm},
    \qquad i=1,\ldots,9.
\end{aligned}
\label{eq:engineering_sea}
\end{equation}
This objective favors designs that absorb more energy per unit mass, incorporating the mass cost of thickness changes directly into the performance measure. The adopted formulation is a single-objective, bound-constrained problem.

\paragraph{Objective evaluation.}
Each candidate thickness vector is inserted into the shell-section definitions of the finite-element model, and LS-DYNA is used to simulate the crushing response. The absorbed energy is obtained from the final internal-energy output, and the component mass is extracted from the solver's part-mass output. These quantities are converted to kJ and kg before computing SEA. Each such simulation counts as one fitness evaluation. The need to run the finite-element solver for every candidate makes this an expensive black-box optimization task.

\paragraph{Evaluation protocol.}
The search is conducted in normalized coordinates $\mathbf{u}\in[0,1]^9$, with physical wall thicknesses recovered by $x_i=0.2+2.8u_i$ in mm. The reported comparison starts with $2D+1=19$ initial evaluations and follows the methods through 40 subsequent sequential evaluations, ending at a total of 59 evaluations. Figure~\ref{fig:case1-convergence} reports the best observed SEA at each evaluation count from 19 to 59. OnDesign generates acquisition functions during the target run and is compared with the same Bayesian optimization and LLM-based AAD baselines used in the benchmark study.

\subsection{Tracing an Online Acquisition-Function Design Decision}
\label{app:bo_case_study_fe056}

To expose the mechanism underlying the run-level trajectory in Fig.~\ref{fig:algorithm_evolution}, we examine one acquisition-function design decision from the same 10-dimensional BBOB F1 run. The selected checkpoint follows 29 consecutive evaluations without improvement and immediately precedes a substantial reduction in the optimality gap. Here, FE56 denotes the 56th sequential evaluation after the 21 initialization evaluations. Before this decision, the optimization state contains 76 objective observations, and the incumbent is -92.526282, corresponding to an optimality gap of 0.123718. Executing the generated acquisition function produces the 77th observation, improves the incumbent to -92.612056, and reduces the gap to 0.037944, a reduction of 69.33\%. For readability, Figs.~\ref{fig:case_fe056_state} and~\ref{fig:case_fe056_deliberation} present representative excerpts from the complete Agent input and output records; bracketed ellipses mark omitted portions. Figure~\ref{fig:case_fe056_code} shows the generated algorithm and its execution feedback. These figures do not depict a SAG Evolution update, which is a separate periodic process. Common role instructions appear in Appendix~\ref{app:prompt_templates}.

\paragraph{Module 1: State Analysis.}
Figure~\ref{fig:case_fe056_state} shows the selected runtime observations and the State Report produced by the State Analyst. At FE56, 45 sequential evaluations remain, and the incumbent has not improved within the latest ten-step state window. The fitted GP has a lengthscale of 11.41, substantially larger than the maximum distance of 3.16 in the normalized search space, indicating a strongly smoothed posterior. Meanwhile, the candidate diagnostic reports zero minimum distance to the observed set, revealing a risk of duplicate sampling. The small local residual indicates local consistency, but the global posterior geometry provides no reliable basis for pure exploitation. Among the conditions discussed in the full State Report, three are particularly relevant to this design decision: prolonged stagnation, potentially over-smoothed posterior geometry, and duplicate-sampling risk.

\paragraph{Module 2: Multi-Agent Deliberation.}
Figure~\ref{fig:case_fe056_deliberation} shows how the two Proposers reason from the same State Report and Algorithm Design History but emphasize different evidence. \emph{Exploitation-oriented Proposer} retains LogEI as the primary ranking mechanism and recommends only a small auxiliary correction with $\beta\in[0.1,0.5]$. \emph{Exploration-oriented Proposer} places greater weight on the prolonged stagnation and ambiguous posterior geometry, advocating a stronger but bounded correction with $\beta\in[0.8,1.3]$. Rather than selecting either proposal directly, \emph{Arbiter} synthesizes a new conditional structure: an improvement-based backbone is retained, while an auxiliary term is activated according to posterior-standardized predicted improvement. The resulting design uses $\beta=1.0$ and a correction weight of 0.35, corresponding to a moderate design tendency.

\paragraph{Generated algorithm and execution feedback.}
Figure~\ref{fig:case_fe056_code} shows the executable output of Multi-Agent Deliberation and the feedback obtained after its execution.
With $b$ denoting the incumbent objective, $\mu$ the posterior mean, and $\sigma$ the clamped posterior standard deviation, the generated acquisition score is
\begin{equation}
\alpha=\log\!\bigl(1+\max\{\mathrm{EI},0\}\bigr)
+0.35\,\frac{\sigma}{\sigma+10^{-8}}\,
\operatorname{sigmoid}\!\left(\frac{b-\mu}{\sigma}\right).
\label{eq:case_fe056_acquisition}
\end{equation}
The first term preserves the ranking induced by expected improvement. The second introduces a bounded, posterior-standardized correction whose influence increases when the posterior assigns greater plausibility to improvement and vanishes as the posterior standard deviation approaches zero. Thus, the generated acquisition function changes the composition and conditional structure of the decision rule rather than merely adjusting a parameter of a predefined acquisition function. The selected evaluation reduces the optimality gap by 69.33\%. Its normalized gain of 0.030447, together with the executed code and updated optimization state, is subsequently added to the Algorithm Design History. These records become part of the context for the next design stage, completing the transition from state evidence to algorithm generation, execution, and renewed state-dependent design.

\clearpage
\begin{figure}[H]
    \centering
    \includegraphics[width=\linewidth]{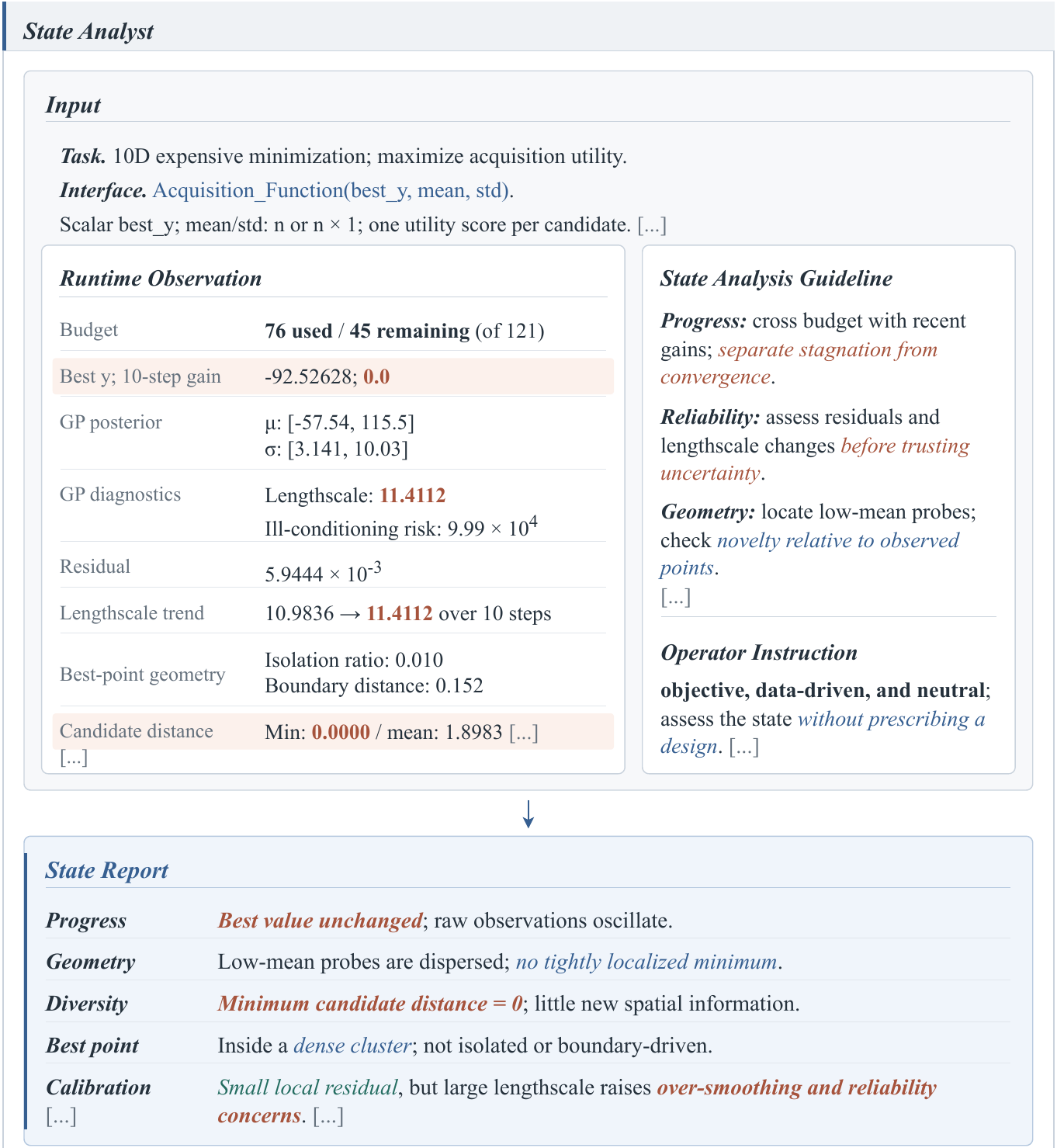}
    \caption{Module 1 (State Analysis) at FE56: from runtime observations to the State Report. The State Analyst organizes evidence concerning search progress, posterior geometry, sampling diversity, and model reliability. Bracketed ellipses denote omitted content; all displayed values are taken from the archived Agent records.}
    \label{fig:case_fe056_state}
\end{figure}

\clearpage
\begin{figure}[H]
    \centering
    \includegraphics[width=\linewidth]{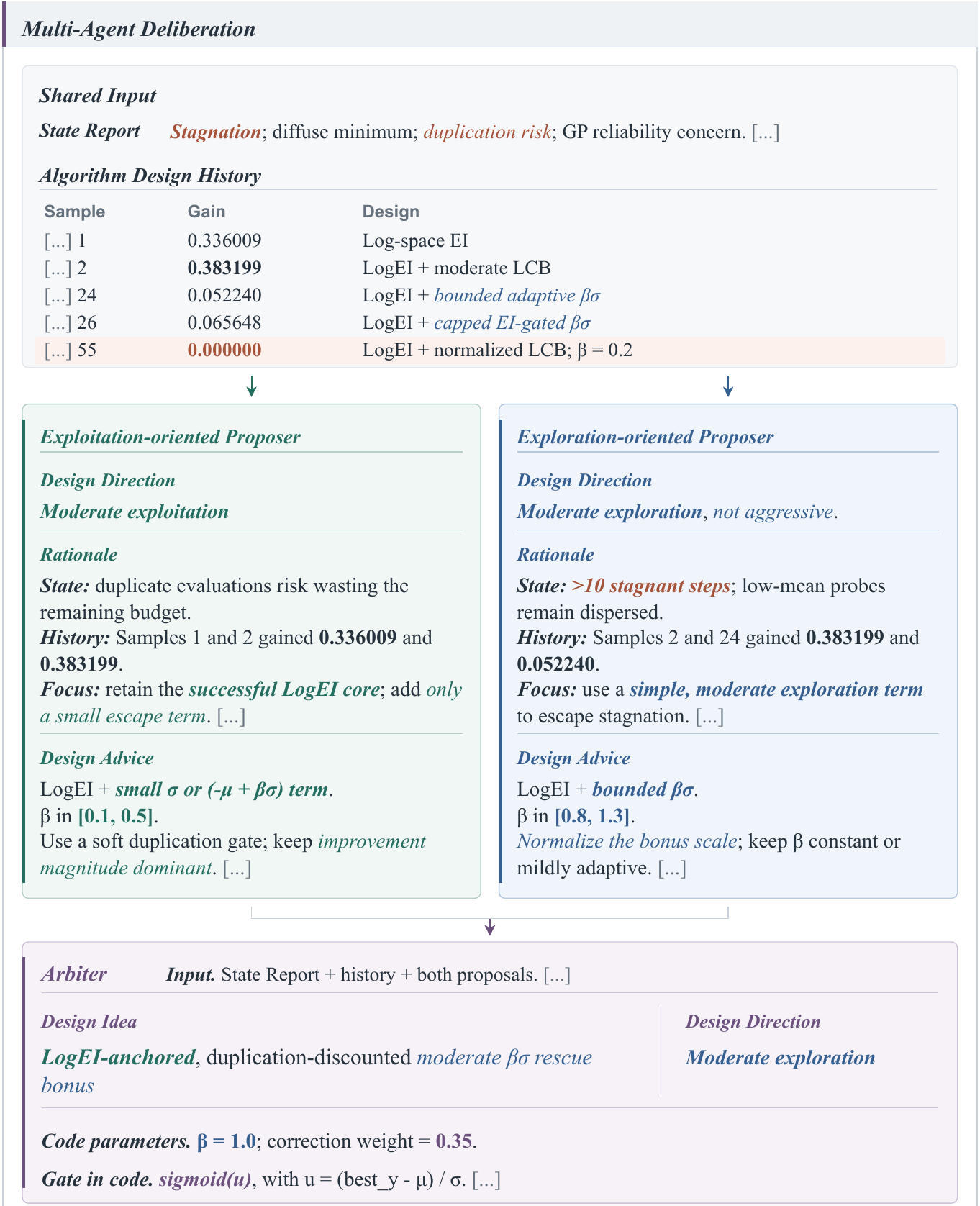}
    \caption{Module 2 (Multi-Agent Deliberation) at FE56: from alternative proposals to the arbitrated acquisition-function design. The two Proposers emphasize different evidence from the same State Report and Algorithm Design History, while \emph{Arbiter} synthesizes an improvement-based design with a gated auxiliary correction. Bracketed ellipses denote omitted content.}
    \label{fig:case_fe056_deliberation}
\end{figure}

\clearpage
\begin{figure}[H]
    \centering
    \includegraphics[width=0.97\linewidth]{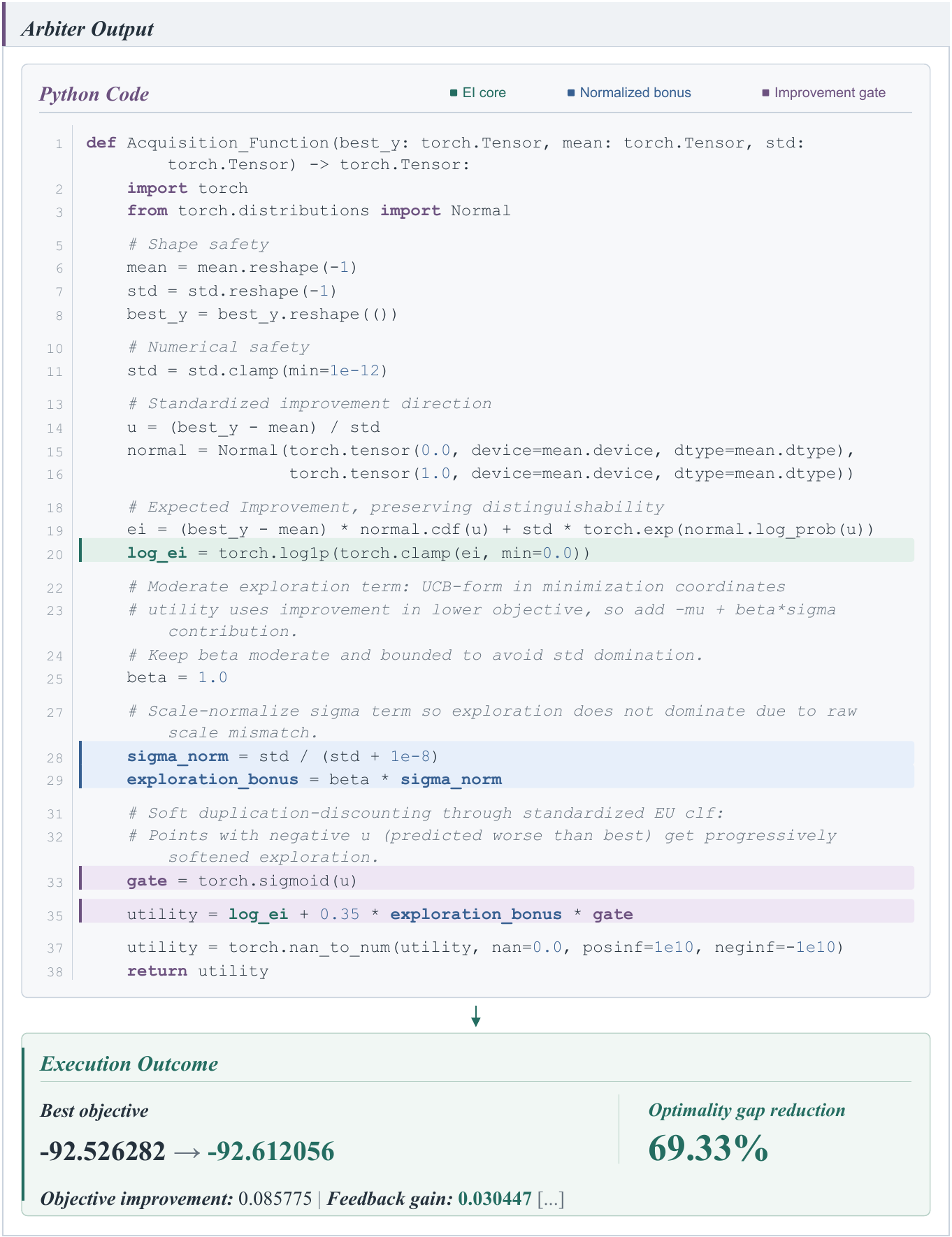}
    \caption{Output of Module 2 (Multi-Agent Deliberation) and its execution feedback at FE56. The highlighted statements implement Eq.~(\ref{eq:case_fe056_acquisition}). The resulting evaluation reduces the optimality gap by 69.33\%, and the normalized gain is returned to the Algorithm Design History for subsequent design stages.}
    \label{fig:case_fe056_code}
\end{figure}

\clearpage
\section{Additional Experimental Results}
\label{app:additional_results}

Average ranks and $+/-/=$ counts are computed from the means displayed in the following function-wise tables. Equal displayed means receive average ranks and count as ties; $+/-/=$ compares each baseline with OnDesign.

\begin{table*}[htbp]
  \centering
  \caption{Results on Bayesian optimization scenario: BBOB test suite with $D=10$.}
  \label{tab:bo-bbob-10d}
  \setlength{\tabcolsep}{2pt}
  \renewcommand{\arraystretch}{1.1}
  \arrayrulecolor{ODRule}
  \resizebox{\linewidth}{!}{%
\begin{tabular}{c|ccccccccccc>{\columncolor{ODOursBody}}c}
    \toprule
     & \cellcolor{ODGroupOne}\textbf{LogEI} & \cellcolor{ODGroupOne}\textbf{LogPI} & \cellcolor{ODGroupOne}\textbf{UCB} & \cellcolor{ODGroupOne}\textbf{JES} & \cellcolor{ODGroupTwo}\textbf{SETUP-BO} & \cellcolor{ODGroupThree}\textbf{FunBO} & \cellcolor{ODGroupThree}\textbf{ATRBO\_LCB} & \cellcolor{ODGroupThree}\textbf{TREvol} & \cellcolor{ODGroupThree}\textbf{EoH} & \cellcolor{ODGroupThree}\textbf{MEoH} & \cellcolor{ODGroupThree}\textbf{ReEvo} & \cellcolor{ODOursHeader}\textbf{OnDesign} \bigstrut\\
    \midrule
    F1 & -6.23E+01(8.56E+00) & -5.93E+01(9.60E+00) & -5.38E+01(8.02E+00) & -8.58E+01(1.66E+00) & -6.17E+01(6.94E+00) & -8.98E+01(3.91E+00) & -9.09E+01(5.32E+00) & -8.99E+01(2.67E+00) & -8.04E+01(2.09E+01) & -5.65E+01(9.73E+00) & -5.97E+01(9.33E+00) & \textbf{-9.26E+01(4.78E-03)} \bigstrut[t]\\
    F2 & 4.35E+05(2.59E+05) & 4.91E+05(2.20E+05) & 5.02E+05(1.97E+05) & 3.27E+05(1.68E+05) & 3.03E+05(1.41E+05) & 8.20E+05(4.98E+05) & 9.83E+05(6.38E+05) & 3.16E+05(2.33E+05) & 1.40E+05(1.44E+05) & 5.18E+05(2.98E+05) & 4.24E+05(2.05E+05) & \textbf{9.72E+04(3.16E+04)}\\
    F3 & 2.89E+02(5.11E+01) & 2.83E+02(3.96E+01) & 2.73E+02(6.74E+01) & 1.74E+02(2.56E+01) & 2.72E+02(6.55E+01) & 2.53E+02(2.38E+01) & 2.45E+02(2.35E+01) & \textbf{1.71E+02(2.32E+01)} & 1.82E+02(9.37E+01) & 2.54E+02(4.59E+01) & 1.73E+02(2.69E+01) & 1.80E+02(1.16E+01)\\
    F4 & 3.30E+02(6.10E+01) & 3.37E+02(6.90E+01) & 3.22E+02(6.46E+01) & 3.30E+02(5.35E+01) & 2.99E+02(5.43E+01) & 3.10E+02(7.00E+01) & 3.63E+02(9.41E+01) & 2.54E+02(6.22E+01) & \textbf{1.95E+02(5.61E+01)} & 3.31E+02(7.48E+01) & 3.44E+02(7.72E+01) & 2.79E+02(2.86E+01)\\
    F5 & 1.52E+02(1.20E+01) & 1.48E+02(1.38E+01) & 1.41E+02(1.37E+01) & 9.33E+01(6.26E+00) & 1.49E+02(1.55E+01) & \textbf{5.15E+01(0.00E+00)} & \textbf{5.15E+01(0.00E+00)} & 1.43E+02(1.86E+01) & 6.60E+01(1.75E+01) & 1.47E+02(1.43E+01) & 1.48E+02(9.20E+00) & \textbf{5.15E+01(0.00E+00)}\\
\bottomrule
    \end{tabular}%
  }
\end{table*}

\begin{table*}[htbp]
  \centering
  \caption{Results on Bayesian optimization scenario: BBOB test suite with $D=30$.}
  \label{tab:bo-bbob-30d}
  \setlength{\tabcolsep}{2pt}
  \renewcommand{\arraystretch}{1.1}
  \arrayrulecolor{ODRule}
  \resizebox{\linewidth}{!}{%
    \begin{tabular}{c|ccccccccccc>{\columncolor{ODOursBody}}c}
    \toprule
     & \cellcolor{ODGroupOne}\textbf{LogEI} & \cellcolor{ODGroupOne}\textbf{LogPI} & \cellcolor{ODGroupOne}\textbf{UCB} & \cellcolor{ODGroupOne}\textbf{JES} & \cellcolor{ODGroupTwo}\textbf{SETUP-BO} & \cellcolor{ODGroupThree}\textbf{FunBO} & \cellcolor{ODGroupThree}\textbf{ATRBO\_LCB} & \cellcolor{ODGroupThree}\textbf{TREvol} & \cellcolor{ODGroupThree}\textbf{EoH} & \cellcolor{ODGroupThree}\textbf{MEoH} & \cellcolor{ODGroupThree}\textbf{ReEvo} & \cellcolor{ODOursHeader}\textbf{OnDesign} \bigstrut\\
    \midrule
    F1 & 9.76E+01(2.86E+01) & 1.07E+02(1.80E+01) & 1.08E+02(2.42E+01) & 1.64E+01(1.60E+01) & 1.07E+02(2.96E+01) & -2.30E+01(4.47E+01) & -5.30E+01(1.04E+01) & 1.09E+02(2.09E+01) & 1.07E+02(2.47E+01) & 9.90E+01(1.91E+01) & 1.15E+02(1.94E+01) & \textbf{-9.26E+01(1.45E-02)} \bigstrut[t]\\
    F2 & 5.41E+06(1.70E+06) & 6.84E+06(2.07E+06) & 6.23E+06(1.51E+06) & 3.19E+06(8.61E+05) & 5.47E+06(1.65E+06) & 5.32E+06(1.31E+06) & 4.94E+06(9.15E+05) & 5.53E+06(2.19E+06) & 5.52E+06(2.12E+06) & 5.63E+06(2.06E+06) & 3.01E+06(7.05E+05) & \textbf{1.50E+06(4.62E+05)}\\
    F3 & 1.43E+03(1.65E+02) & 1.53E+03(2.06E+02) & 1.47E+03(1.83E+02) & 1.09E+03(1.20E+02) & 1.54E+03(2.82E+02) & 7.93E+02(8.69E+01) & 9.43E+02(1.17E+02) & 1.58E+03(3.13E+02) & 1.49E+03(1.89E+02) & 1.39E+03(2.50E+02) & 1.05E+03(1.49E+02) & \textbf{5.44E+02(2.70E+01)}\\
    F4 & 1.61E+03(1.57E+02) & 1.56E+03(1.54E+02) & 1.63E+03(1.47E+02) & 1.54E+03(2.06E+02) & 1.60E+03(1.64E+02) & 1.57E+03(1.59E+02) & 1.67E+03(2.25E+02) & 1.78E+03(1.83E+02) & 1.80E+03(1.79E+02) & 1.73E+03(2.40E+02) & 1.61E+03(1.74E+02) & \textbf{1.15E+03(1.11E+02)}\\
    F5 & 4.45E+02(2.24E+01) & 4.37E+02(3.26E+01) & 4.52E+02(3.00E+01) & 3.83E+02(6.79E+01) & 4.51E+02(2.52E+01) & 2.02E+02(1.86E+02) & \textbf{5.15E+01(0.00E+00)} & 4.17E+02(3.71E+01) & 4.53E+02(2.96E+01) & 4.40E+02(2.48E+01) & 4.54E+02(1.66E+01) & \textbf{5.15E+01(0.00E+00)}\\
\bottomrule
    \end{tabular}%
  }
\end{table*}

\begin{table*}[htbp]
  \centering
  \caption{Results on Bayesian optimization scenario: CEC2020 test suite with $D=10$.}
  \label{tab:bo-cec2020-10d}
  \setlength{\tabcolsep}{2pt}
  \renewcommand{\arraystretch}{1.1}
  \arrayrulecolor{ODRule}
  \resizebox{\linewidth}{!}{%
    \begin{tabular}{c|ccccccccccc>{\columncolor{ODOursBody}}c}
    \toprule
     & \cellcolor{ODGroupOne}\textbf{LogEI} & \cellcolor{ODGroupOne}\textbf{LogPI} & \cellcolor{ODGroupOne}\textbf{UCB} & \cellcolor{ODGroupOne}\textbf{JES} & \cellcolor{ODGroupTwo}\textbf{SETUP-BO} & \cellcolor{ODGroupThree}\textbf{FunBO} & \cellcolor{ODGroupThree}\textbf{ATRBO\_LCB} & \cellcolor{ODGroupThree}\textbf{TREvol} & \cellcolor{ODGroupThree}\textbf{EoH} & \cellcolor{ODGroupThree}\textbf{MEoH} & \cellcolor{ODGroupThree}\textbf{ReEvo} & \cellcolor{ODOursHeader}\textbf{OnDesign} \bigstrut\\
    \midrule
    F1 & 8.64E+07(3.06E+06) & 1.73E+10(5.42E+09) & 1.57E+10(4.16E+09) & 1.43E+10(4.64E+09) & 1.42E+10(6.88E+09) & 5.12E+09(1.60E+09) & 1.58E+10(5.40E+09) & 1.64E+10(5.11E+09) & 4.01E+08(5.02E+08) & 1.75E+10(5.73E+09) & 1.56E+10(5.67E+09) & \textbf{7.31E+07(3.47E+07)} \bigstrut[t]\\
    F2 & 2.00E+03(3.31E+02) & 2.09E+03(5.48E+02) & 2.08E+03(4.23E+02) & 2.05E+03(3.47E+02) & 2.13E+03(4.27E+02) & 2.13E+03(5.19E+02) & 2.01E+03(4.42E+02) & 2.14E+03(4.68E+02) & 1.95E+03(4.80E+02) & 1.92E+03(3.82E+02) & \textbf{1.91E+03(5.57E+02)} & 2.01E+03(3.90E+02)\\
    F3 & 4.26E+05(9.77E+04) & 4.00E+05(1.10E+05) & 2.03E+03(2.11E+02) & 3.97E+05(1.02E+05) & 3.05E+05(1.24E+05) & 6.78E+04(2.43E+04) & 9.04E+04(5.15E+04) & 4.05E+05(7.91E+04) & 3.94E+05(8.81E+04) & 3.95E+05(1.12E+05) & 3.53E+05(8.56E+04) & \textbf{8.92E+02(1.40E+02)}\\
    F4 & 1.16E+05(1.17E+05) & 1.29E+05(1.30E+05) & 3.17E+05(2.84E+05) & 9.83E+04(1.35E+05) & 1.66E+05(1.85E+05) & \textbf{3.90E+04(2.60E+04)} & 2.94E+05(3.07E+05) & 1.04E+05(1.14E+05) & 1.12E+06(9.12E+04) & 9.58E+04(8.02E+04) & 1.32E+05(1.41E+05) & 8.51E+05(6.89E+05)\\
    F5 & 1.21E+07(9.34E+06) & \textbf{1.74E+06(1.09E+06)} & 1.68E+07(1.93E+07) & 1.92E+07(1.99E+07) & 2.14E+07(2.12E+07) & 3.09E+07(3.27E+07) & 6.90E+07(4.86E+07) & 1.88E+07(1.34E+07) & 5.43E+06(2.54E+06) & 2.58E+07(1.88E+07) & 1.40E+07(1.24E+07) & 4.98E+06(1.54E+06)\\
\bottomrule
    \end{tabular}%
  }
\end{table*}

\begin{table*}[htbp]
  \centering
  \caption{Results in the Bayesian optimization scenario: CEC2020 test suite with $D=20$.}
  \label{tab:bo-cec2020-20d}
  \setlength{\tabcolsep}{2pt}
  \renewcommand{\arraystretch}{1.1}
  \arrayrulecolor{ODRule}
  \resizebox{\linewidth}{!}{%
    \begin{tabular}{c|ccccccccccc>{\columncolor{ODOursBody}}c}
    \toprule
     & \cellcolor{ODGroupOne}\textbf{LogEI} & \cellcolor{ODGroupOne}\textbf{LogPI} & \cellcolor{ODGroupOne}\textbf{UCB} & \cellcolor{ODGroupOne}\textbf{JES} & \cellcolor{ODGroupTwo}\textbf{SETUP-BO} & \cellcolor{ODGroupThree}\textbf{FunBO} & \cellcolor{ODGroupThree}\textbf{ATRBO\_LCB} & \cellcolor{ODGroupThree}\textbf{TREvol} & \cellcolor{ODGroupThree}\textbf{EoH} & \cellcolor{ODGroupThree}\textbf{MEoH} & \cellcolor{ODGroupThree}\textbf{ReEvo} & \cellcolor{ODOursHeader}\textbf{OnDesign} \bigstrut\\
    \midrule
    F1 & 4.99E+10(8.82E+09) & 4.93E+10(8.93E+09) & 4.76E+10(7.72E+09) & 5.26E+10(6.47E+09) & 4.18E+10(1.78E+10) & 5.15E+10(1.08E+10) & 5.14E+10(9.00E+09) & 5.03E+10(8.42E+09) & \textbf{3.25E+09(4.51E+09)} & 5.06E+10(6.79E+09) & 4.95E+10(5.72E+09) & 3.46E+09(1.92E+09) \bigstrut[t]\\
    F2 & 4.28E+03(6.43E+02) & 4.47E+03(8.50E+02) & 4.66E+03(6.03E+02) & 4.67E+03(8.48E+02) & 4.22E+03(7.87E+02) & 4.46E+03(4.27E+02) & 4.60E+03(6.49E+02) & 4.29E+03(6.59E+02) & 4.69E+03(5.76E+02) & \textbf{4.14E+03(8.41E+02)} & 4.80E+03(4.53E+02) & 4.42E+03(7.20E+01)\\
    F3 & 1.37E+06(2.18E+05) & 5.83E+04(5.34E+04) & 1.35E+06(2.00E+05) & 8.85E+05(1.66E+05) & 5.31E+05(1.99E+05) & 3.69E+05(2.17E+05) & 1.34E+06(2.05E+05) & 1.53E+06(2.52E+05) & 1.41E+06(1.79E+05) & 7.74E+05(1.11E+05) & 1.31E+06(2.01E+05) & \textbf{9.01E+02(1.82E+01)}\\
    F4 & 1.45E+06(6.76E+05) & 1.06E+06(7.58E+05) & 1.93E+05(2.91E+05) & 1.22E+06(1.19E+06) & 2.74E+05(3.20E+05) & 1.30E+05(1.44E+05) & 4.27E+05(7.75E+05) & \textbf{6.81E+04(4.64E+04)} & 1.25E+06(6.56E+05) & 1.16E+06(5.69E+05) & 3.33E+05(2.54E+05) & 4.79E+06(5.60E+06)\\
    F5 & 4.63E+07(2.41E+07) & 4.69E+07(2.41E+07) & 1.65E+08(1.01E+08) & 1.55E+08(1.18E+08) & 1.60E+08(9.20E+07) & 1.79E+08(8.99E+07) & 2.45E+08(1.55E+08) & 1.68E+08(1.00E+08) & 1.68E+08(9.84E+07) & 1.78E+08(7.04E+07) & 1.72E+08(8.62E+07) & \textbf{4.48E+07(3.68E+07)}\\
\bottomrule
    \end{tabular}%
  }
\end{table*}

\begin{table*}[htbp]
  \centering
  \caption{Results on Bayesian optimization scenario: CEC2022 test suite with $D=10$.}
  \label{tab:bo-cec2022-10d}
  \setlength{\tabcolsep}{2pt}
  \renewcommand{\arraystretch}{1.1}
  \arrayrulecolor{ODRule}
  \resizebox{\linewidth}{!}{%
    \begin{tabular}{c|ccccccccccc>{\columncolor{ODOursBody}}c}
    \toprule
     & \cellcolor{ODGroupOne}\textbf{LogEI} & \cellcolor{ODGroupOne}\textbf{LogPI} & \cellcolor{ODGroupOne}\textbf{UCB} & \cellcolor{ODGroupOne}\textbf{JES} & \cellcolor{ODGroupTwo}\textbf{SETUP-BO} & \cellcolor{ODGroupThree}\textbf{FunBO} & \cellcolor{ODGroupThree}\textbf{ATRBO\_LCB} & \cellcolor{ODGroupThree}\textbf{TREvol} & \cellcolor{ODGroupThree}\textbf{EoH} & \cellcolor{ODGroupThree}\textbf{MEoH} & \cellcolor{ODGroupThree}\textbf{ReEvo} & \cellcolor{ODOursHeader}\textbf{OnDesign} \bigstrut\\
    \midrule
    F1 & 2.23E+04(6.42E+03) & 3.72E+04(1.55E+04) & 2.27E+04(6.66E+03) & 2.14E+04(7.46E+03) & 3.40E+04(1.30E+04) & 3.48E+04(9.26E+03) & 4.02E+04(2.06E+04) & 3.10E+04(1.01E+04) & \textbf{1.95E+04(6.37E+03)} & 2.10E+04(6.56E+03) & 2.37E+04(8.37E+03) & 2.35E+04(2.08E+04) \bigstrut[t]\\
    F2 & 2.16E+03(7.87E+02) & 2.07E+03(6.55E+02) & 1.85E+03(4.86E+02) & 2.60E+03(8.73E+02) & 1.66E+03(1.17E+03) & 2.25E+03(8.84E+02) & 2.29E+03(7.77E+02) & 8.88E+02(1.90E+02) & 1.63E+03(6.78E+02) & 2.16E+03(6.37E+02) & 2.05E+03(8.63E+02) & \textbf{8.01E+02(2.25E+02)}\\
    F3 & 6.01E+02(2.01E-01) & 6.01E+02(1.55E-01) & 6.01E+02(1.83E-01) & \textbf{6.00E+02(3.50E-02)} & \textbf{6.00E+02(1.91E-01)} & 6.01E+02(1.84E-01) & 6.01E+02(2.02E-01) & 6.01E+02(1.78E-01) & 6.01E+02(1.95E-01) & 6.01E+02(1.88E-01) & 6.01E+02(1.81E-01) & 6.48E+02(1.81E+01)\\
    F4 & 1.00E+03(2.47E+01) & 9.91E+02(3.12E+01) & 9.40E+02(2.93E+01) & 9.92E+02(3.28E+01) & 1.01E+03(3.40E+01) & 1.01E+03(2.70E+01) & 9.93E+02(3.00E+01) & 9.94E+02(2.66E+01) & 9.91E+02(3.35E+01) & 9.99E+02(2.55E+01) & 9.85E+02(2.80E+01) & \textbf{8.72E+02(1.97E+01)}\\
    F5 & \textbf{9.06E+02(1.38E+00)} & 9.12E+02(3.76E+00) & 9.08E+02(2.86E+00) & 9.07E+02(2.86E+00) & 9.07E+02(1.92E+00) & 9.12E+02(3.37E+00) & 9.13E+02(3.92E+00) & 9.09E+02(1.93E+00) & 9.07E+02(1.89E+00) & 9.07E+02(2.34E+00) & 9.07E+02(3.22E+00) & 1.66E+03(8.39E+02)\\
\bottomrule
    \end{tabular}%
  }
\end{table*}

\begin{table*}[htbp]
  \centering
  \caption{Results on Bayesian optimization scenario: CEC2022 test suite with $D=20$.}
  \label{tab:bo-cec2022-20d}
  \setlength{\tabcolsep}{2pt}
  \renewcommand{\arraystretch}{1.1}
  \arrayrulecolor{ODRule}
  \resizebox{\linewidth}{!}{%
    \begin{tabular}{c|ccccccccccc>{\columncolor{ODOursBody}}c}
    \toprule
     & \cellcolor{ODGroupOne}\textbf{LogEI} & \cellcolor{ODGroupOne}\textbf{LogPI} & \cellcolor{ODGroupOne}\textbf{UCB} & \cellcolor{ODGroupOne}\textbf{JES} & \cellcolor{ODGroupTwo}\textbf{SETUP-BO} & \cellcolor{ODGroupThree}\textbf{FunBO} & \cellcolor{ODGroupThree}\textbf{ATRBO\_LCB} & \cellcolor{ODGroupThree}\textbf{TREvol} & \cellcolor{ODGroupThree}\textbf{EoH} & \cellcolor{ODGroupThree}\textbf{MEoH} & \cellcolor{ODGroupThree}\textbf{ReEvo} & \cellcolor{ODOursHeader}\textbf{OnDesign} \bigstrut\\
    \midrule
    F1 & 7.07E+04(1.33E+04) & 7.10E+04(1.31E+04) & 6.68E+04(1.61E+04) & 7.02E+04(1.36E+04) & 9.59E+04(3.30E+04) & 6.66E+04(8.71E+03) & 8.16E+04(1.88E+04) & 6.99E+04(1.36E+04) & 7.04E+04(1.39E+04) & 6.57E+04(1.71E+04) & 7.12E+04(1.51E+04) & \textbf{5.68E+04(5.74E+04)} \bigstrut[t]\\
    F2 & 5.90E+03(1.43E+03) & 6.96E+03(1.48E+03) & 6.64E+03(1.54E+03) & 6.54E+03(9.37E+02) & 2.87E+03(2.38E+03) & 1.84E+03(3.77E+02) & 6.75E+03(1.88E+03) & 6.79E+03(1.52E+03) & 6.65E+03(1.67E+03) & 2.81E+03(6.84E+02) & 7.29E+03(1.68E+03) & \textbf{1.36E+03(3.61E+02)}\\
    F3 & 6.02E+02(2.75E-01) & 6.03E+02(4.61E-01) & 6.03E+02(4.92E-01) & \textbf{6.01E+02(5.73E-01)} & \textbf{6.01E+02(2.55E-01)} & 6.03E+02(4.41E-01) & \textbf{6.01E+02(3.86E-01)} & 6.02E+02(3.99E-01) & 6.02E+02(4.22E-01) & 6.02E+02(5.15E-01) & 6.02E+02(3.20E-01) & 6.43E+02(8.06E+00)\\
    F4 & 1.34E+03(4.23E+01) & 1.32E+03(5.60E+01) & 1.41E+03(4.00E+01) & 1.38E+03(5.64E+01) & 1.29E+03(5.99E+01) & 1.33E+03(5.70E+01) & 1.35E+03(6.08E+01) & 1.29E+03(6.56E+01) & 1.35E+03(3.70E+01) & 1.34E+03(4.68E+01) & 1.35E+03(4.05E+01) & \textbf{9.70E+02(2.33E+01)}\\
    F5 & 9.31E+02(4.65E+00) & 9.30E+02(6.92E+00) & \textbf{9.06E+02(3.02E+00)} & 9.33E+02(4.23E+00) & 9.19E+02(6.49E+00) & 9.28E+02(7.33E+00) & 9.32E+02(6.69E+00) & 9.28E+02(4.46E+00) & 9.27E+02(6.06E+00) & 9.28E+02(5.08E+00) & 9.27E+02(4.73E+00) & 1.66E+03(3.01E+02)\\
\bottomrule
    \end{tabular}%
  }
\end{table*}

\begin{table*}[htbp]
  \centering
  \caption{Results on Bayesian optimization scenario: CEC2026 test suite with $D=30$.}
  \label{tab:bo-cec2026-30d}
  \setlength{\tabcolsep}{2pt}
  \renewcommand{\arraystretch}{1.1}
  \arrayrulecolor{ODRule}
  \resizebox{\linewidth}{!}{%
    \begin{tabular}{c|ccccccccccc>{\columncolor{ODOursBody}}c}
    \toprule
     & \cellcolor{ODGroupOne}\textbf{LogEI} & \cellcolor{ODGroupOne}\textbf{LogPI} & \cellcolor{ODGroupOne}\textbf{UCB} & \cellcolor{ODGroupOne}\textbf{JES} & \cellcolor{ODGroupTwo}\textbf{SETUP-BO} & \cellcolor{ODGroupThree}\textbf{FunBO} & \cellcolor{ODGroupThree}\textbf{ATRBO\_LCB} & \cellcolor{ODGroupThree}\textbf{TREvol} & \cellcolor{ODGroupThree}\textbf{EoH} & \cellcolor{ODGroupThree}\textbf{MEoH} & \cellcolor{ODGroupThree}\textbf{ReEvo} & \cellcolor{ODOursHeader}\textbf{OnDesign} \bigstrut\\
    \midrule
    F1 & 1.13E+11(1.27E+10) & 1.09E+11(9.71E+09) & 1.06E+11(1.20E+10) & 1.06E+11(1.15E+10) & 3.57E+10(8.85E+09) & 3.84E+10(1.60E+10) & 1.08E+11(8.58E+09) & 1.14E+11(1.04E+10) & 1.10E+11(1.88E+10) & 8.91E+10(1.17E+10) & 9.57E+10(1.04E+10) & \textbf{4.99E+09(3.33E+09)} \bigstrut[t]\\
    F2 & 1.94E+05(3.21E+04) & 1.88E+05(3.59E+04) & 2.10E+05(2.23E+04) & 1.78E+05(3.27E+04) & 1.83E+05(3.82E+04) & 1.93E+05(2.15E+04) & 2.04E+05(3.75E+04) & 1.91E+05(4.05E+04) & 1.90E+05(3.71E+04) & 1.80E+05(3.29E+04) & 1.68E+05(2.90E+04) & \textbf{1.45E+05(1.41E+04)}\\
    F3 & 2.79E+04(1.41E+04) & 3.80E+04(8.10E+03) & 3.72E+04(8.81E+03) & 2.51E+04(6.95E+03) & 3.11E+04(6.79E+03) & 3.29E+04(9.54E+03) & 3.56E+04(8.08E+03) & 3.63E+04(7.07E+03) & 3.88E+04(1.12E+04) & 4.26E+04(9.47E+03) & 2.11E+04(5.47E+03) & \textbf{5.76E+03(1.82E+03)}\\
    F4 & 1.23E+05(1.94E+04) & 1.18E+05(1.59E+04) & 1.32E+05(1.61E+04) & 1.25E+05(1.34E+04) & 1.25E+05(1.63E+04) & 1.31E+05(1.43E+04) & 1.29E+05(1.50E+04) & 1.26E+05(2.05E+04) & 1.26E+05(1.55E+04) & 1.30E+05(1.51E+04) & 1.30E+05(1.76E+04) & \textbf{8.08E+02(3.04E+01)}\\
\bottomrule
    \end{tabular}%
  }
\end{table*}

\begin{table*}[htbp]
  \centering
  \caption{Results on evolutionary continuous optimization scenario: BBOB test suite with $D=10$.}
  \label{tab:ec-bbob-10d}
  \setlength{\tabcolsep}{2pt}
  \renewcommand{\arraystretch}{1.1}
  \arrayrulecolor{ODRule}
  \resizebox{\linewidth}{!}{%
    \begin{tabular}{c|ccccccccccc>{\columncolor{ODOursBody}}c}
    \toprule
     & \cellcolor{ODGroupOne}\textbf{BSPGA} & \cellcolor{ODGroupOne}\textbf{CLPSO} & \cellcolor{ODGroupOne}\textbf{CoDE} & \cellcolor{ODGroupTwo}\textbf{SAHLPSO} & \cellcolor{ODGroupTwo}\textbf{SHADE} & \cellcolor{ODGroupTwo}\textbf{AutoEP} & \cellcolor{ODGroupThree}\textbf{ParEvo1} & \cellcolor{ODGroupThree}\textbf{ParEvo2} & \cellcolor{ODGroupThree}\textbf{EoH} & \cellcolor{ODGroupThree}\textbf{MEoH} & \cellcolor{ODGroupThree}\textbf{ReEvo} & \cellcolor{ODOursHeader}\textbf{OnDesign} \bigstrut\\
    \midrule
    F1 & -9.16E+01(7.75E-01) & -8.80E+01(1.96E+00) & -7.69E+01(3.94E+00) & -9.24E+01(6.53E-01) & -9.21E+01(1.95E-01) & -9.25E+01(1.91E-02) & -6.80E+01(4.96E+00) & \textbf{-9.26E+01(2.02E-02)} & -8.85E+01(6.43E-01) & \textbf{-9.26E+01(2.28E-03)} & \textbf{-9.26E+01(1.28E-03)} & -8.00E+01(1.79E+01) \bigstrut[t]\\
    F2 & 1.94E+03(1.01E+03) & 2.07E+04(8.96E+03) & 9.22E+04(5.10E+04) & 2.24E+03(6.65E+03) & 1.18E+03(3.29E+02) & 1.25E+03(4.16E+02) & 6.83E+05(3.59E+05) & 4.65E+02(1.04E+02) & 6.51E+04(2.99E+04) & 3.55E+02(7.65E+01) & 3.10E+02(1.50E+01) & \textbf{2.80E+02(3.88E+00)}\\
    F3 & 4.68E+01(4.49E+00) & 1.24E+02(2.05E+01) & 1.41E+02(1.72E+01) & 6.04E+01(1.74E+01) & 6.93E+01(7.73E+00) & \textbf{3.51E+01(5.51E-01)} & 3.41E+02(5.08E+01) & 4.18E+01(3.83E+00) & 9.19E+01(1.21E+01) & 4.40E+01(8.09E+00) & 4.46E+01(9.08E+00) & 4.70E+01(5.28E+00)\\
    F4 & 6.16E+01(1.01E+01) & 1.58E+02(2.24E+01) & 2.50E+02(4.48E+01) & 6.20E+01(1.60E+01) & 7.23E+01(5.65E+00) & \textbf{3.59E+01(4.45E+00)} & 1.97E+02(2.39E+01) & 4.53E+01(4.33E+00) & 1.31E+02(1.65E+01) & 5.17E+01(1.16E+01) & 4.98E+01(8.63E+00) & 4.97E+01(1.16E+01)\\
\bottomrule
    \end{tabular}%
  }
\end{table*}

\begin{table*}[htbp]
  \centering
  \caption{Results on evolutionary continuous optimization: BBOB test suite with $D=30$.}
  \label{tab:ec-bbob-30d}
  \setlength{\tabcolsep}{2pt}
  \renewcommand{\arraystretch}{1.1}
  \arrayrulecolor{ODRule}
  \resizebox{\linewidth}{!}{%
    \begin{tabular}{c|ccccccccccc>{\columncolor{ODOursBody}}c}
    \toprule
     & \cellcolor{ODGroupOne}\textbf{BSPGA} & \cellcolor{ODGroupOne}\textbf{CLPSO} & \cellcolor{ODGroupOne}\textbf{CoDE} & \cellcolor{ODGroupTwo}\textbf{SAHLPSO} & \cellcolor{ODGroupTwo}\textbf{SHADE} & \cellcolor{ODGroupTwo}\textbf{AutoEP} & \cellcolor{ODGroupThree}\textbf{ParEvo1} & \cellcolor{ODGroupThree}\textbf{ParEvo2} & \cellcolor{ODGroupThree}\textbf{EoH} & \cellcolor{ODGroupThree}\textbf{MEoH} & \cellcolor{ODGroupThree}\textbf{ReEvo} & \cellcolor{ODOursHeader}\textbf{OnDesign} \bigstrut\\
    \midrule
    F1 & -6.72E+01(4.97E+00) & -1.18E+01(1.13E+01) & 8.26E+01(4.14E+01) & -7.98E+01(7.34E+00) & -8.60E+01(1.22E+00) & -8.73E+01(9.94E-01) & 1.36E+01(5.04E+00) & -8.21E+01(2.19E+00) & \textbf{-9.20E+01(7.58E-02)} & -9.05E+01(7.14E-01) & -9.04E+01(9.76E-01) & -9.18E+01(2.81E-01) \bigstrut[t]\\
    F2 & 1.72E+05(6.03E+04) & 1.08E+06(3.32E+05) & 1.66E+06(4.76E+05) & 2.41E+05(2.62E+05) & 4.86E+04(7.62E+03) & 6.28E+04(1.54E+04) & 7.12E+06(1.88E+06) & 4.81E+04(8.60E+03) & 6.55E+04(1.61E+04) & 1.83E+04(1.09E+04) & 1.26E+04(5.83E+03) & \textbf{7.33E+03(2.56E+03)}\\
    F3 & 3.43E+02(2.70E+01) & 7.66E+02(9.09E+01) & 8.75E+02(1.04E+02) & 3.42E+02(6.60E+01) & 3.00E+02(1.34E+01) & \textbf{1.44E+02(1.86E+01)} & 8.98E+02(5.17E+01) & 2.58E+02(2.65E+01) & 6.61E+02(1.58E+02) & 2.71E+02(2.87E+01) & 2.64E+02(4.34E+01) & 2.15E+02(5.63E+01)\\
    F4 & 5.39E+02(6.83E+01) & 9.44E+02(9.56E+01) & 1.70E+03(1.92E+02) & 4.39E+02(8.77E+01) & 3.48E+02(2.38E+01) & \textbf{2.54E+02(2.75E+01)} & 8.00E+02(2.10E+01) & 3.43E+02(2.03E+01) & 8.56E+02(1.54E+02) & 3.56E+02(3.97E+01) & 3.38E+02(3.49E+01) & 3.43E+02(1.18E+01)\\
    F5 & 1.04E+02(5.77E+00) & 2.29E+02(2.74E+01) & 1.27E+02(3.11E+01) & 7.30E+01(2.54E+01) & 6.40E+01(3.55E+00) & \textbf{6.02E+01(1.68E-01)} & 5.29E+02(1.40E+01) & 6.40E+01(2.62E+00) & 1.11E+02(2.09E+01) & 1.65E+02(3.65E+01) & 1.70E+02(4.06E+01) & 7.79E+01(1.32E+01)\\
\bottomrule
    \end{tabular}%
  }
\end{table*}

\begin{table*}[htbp]
  \centering
  \caption{Results on evolutionary continuous optimization: CEC2020 test suite with $D=10$.}
  \label{tab:ec-cec2020-10d}
  \setlength{\tabcolsep}{2pt}
  \renewcommand{\arraystretch}{1.1}
  \arrayrulecolor{ODRule}
  \resizebox{\linewidth}{!}{%
    \begin{tabular}{c|ccccccccccc>{\columncolor{ODOursBody}}c}
    \toprule
     & \cellcolor{ODGroupOne}\textbf{BSPGA} & \cellcolor{ODGroupOne}\textbf{CLPSO} & \cellcolor{ODGroupOne}\textbf{CoDE} & \cellcolor{ODGroupTwo}\textbf{SAHLPSO} & \cellcolor{ODGroupTwo}\textbf{SHADE} & \cellcolor{ODGroupTwo}\textbf{AutoEP} & \cellcolor{ODGroupThree}\textbf{ParEvo1} & \cellcolor{ODGroupThree}\textbf{ParEvo2} & \cellcolor{ODGroupThree}\textbf{EoH} & \cellcolor{ODGroupThree}\textbf{MEoH} & \cellcolor{ODGroupThree}\textbf{ReEvo} & \cellcolor{ODOursHeader}\textbf{OnDesign} \bigstrut\\
    \midrule
    F1 & 3.81E+08(2.33E+08) & 2.45E+09(8.42E+08) & 3.75E+09(1.19E+09) & 7.23E+05(6.26E+05) & 3.26E+08(1.09E+08) & 1.04E+08(5.80E+07) & 6.52E+09(1.81E+09) & 2.14E+07(1.14E+07) & 1.46E+09(3.84E+08) & 1.69E+06(1.40E+06) & 1.26E+06(9.37E+05) & \textbf{1.53E+05(9.01E+04)} \bigstrut[t]\\
    F2 & 1.98E+03(2.50E+02) & 2.64E+03(2.11E+02) & 2.72E+03(1.82E+02) & 2.39E+03(4.18E+02) & 2.59E+03(1.38E+02) & 1.98E+03(1.39E+02) & 2.81E+03(1.78E+02) & \textbf{1.71E+03(1.89E+02)} & 2.97E+03(1.70E+02) & 2.03E+03(3.77E+02) & 2.05E+03(3.90E+02) & 2.31E+03(3.83E+02)\\
    F3 & 1.56E+04(7.35E+03) & 5.10E+04(1.85E+04) & 1.13E+05(4.07E+04) & \textbf{8.09E+02(2.82E+01)} & 7.47E+03(2.30E+03) & 2.94E+03(7.18E+02) & 1.74E+05(4.26E+04) & 1.11E+03(8.77E+01) & 4.48E+04(9.29E+03) & 8.24E+02(2.64E+01) & 8.27E+02(2.32E+01) & 8.85E+02(1.48E+02)\\
    F4 & 1.91E+03(2.02E+00) & 2.23E+03(3.93E+02) & 3.57E+03(1.79E+03) & 1.91E+03(2.86E+00) & 1.91E+03(1.44E+00) & \textbf{1.90E+03(9.90E-01)} & 7.97E+03(4.85E+03) & \textbf{1.90E+03(8.33E-01)} & 1.93E+03(1.12E+01) & \textbf{1.90E+03(7.65E-01)} & \textbf{1.90E+03(1.04E+00)} & \textbf{1.90E+03(7.28E-01)}\\
    F5 & 2.18E+04(1.44E+04) & 7.66E+04(3.91E+04) & 4.30E+05(4.01E+05) & 4.23E+04(8.70E+04) & 3.13E+04(9.28E+03) & 3.38E+04(1.67E+04) & 6.34E+05(3.88E+05) & 1.07E+04(3.83E+03) & 9.50E+05(7.54E+05) & 1.83E+04(8.72E+03) & 2.06E+04(9.63E+03) & \textbf{4.13E+03(1.06E+03)}\\
\bottomrule
    \end{tabular}%
  }
\end{table*}

\begin{table*}[htbp]
  \centering
  \caption{Results on evolutionary continuous optimization: CEC2020 test suite with $D=20$.}
  \label{tab:ec-cec2020-20d}
  \setlength{\tabcolsep}{2pt}
  \renewcommand{\arraystretch}{1.1}
  \arrayrulecolor{ODRule}
  \resizebox{\linewidth}{!}{%
    \begin{tabular}{c|ccccccccccc>{\columncolor{ODOursBody}}c}
    \toprule
     & \cellcolor{ODGroupOne}\textbf{BSPGA} & \cellcolor{ODGroupOne}\textbf{CLPSO} & \cellcolor{ODGroupOne}\textbf{CoDE} & \cellcolor{ODGroupTwo}\textbf{SAHLPSO} & \cellcolor{ODGroupTwo}\textbf{SHADE} & \cellcolor{ODGroupTwo}\textbf{AutoEP} & \cellcolor{ODGroupThree}\textbf{ParEvo1} & \cellcolor{ODGroupThree}\textbf{ParEvo2} & \cellcolor{ODGroupThree}\textbf{EoH} & \cellcolor{ODGroupThree}\textbf{MEoH} & \cellcolor{ODGroupThree}\textbf{ReEvo} & \cellcolor{ODOursHeader}\textbf{OnDesign} \bigstrut\\
    \midrule
    F1 & 3.81E+09(1.22E+09) & 1.43E+10(3.33E+09) & 2.01E+10(4.93E+09) & 2.67E+08(5.52E+08) & 1.56E+09(3.31E+08) & 5.50E+08(6.77E+07) & 3.57E+10(4.63E+09) & 6.11E+08(1.10E+08) & 7.78E+08(1.38E+08) & 1.16E+08(7.16E+07) & 1.11E+08(6.39E+07) & \textbf{8.90E+07(5.58E+07)} \bigstrut[t]\\
    F2 & 3.56E+03(4.27E+02) & 5.58E+03(2.84E+02) & 5.71E+03(2.81E+02) & 3.97E+03(8.12E+02) & 4.92E+03(3.11E+02) & \textbf{2.36E+03(2.18E+02)} & 5.64E+03(2.15E+02) & 3.19E+03(2.58E+02) & 5.82E+03(2.63E+02) & 4.10E+03(6.78E+02) & 4.28E+03(5.57E+02) & 4.13E+03(7.12E+02)\\
    F3 & 1.34E+05(3.72E+04) & 4.39E+05(1.09E+05) & 5.85E+05(1.30E+05) & \textbf{2.10E+03(1.09E+03)} & 4.04E+04(7.29E+03) & 2.00E+04(8.22E+03) & 1.08E+06(1.88E+05) & 1.97E+04(4.29E+03) & 2.43E+04(3.63E+03) & 4.65E+03(2.03E+03) & 4.61E+03(2.17E+03) & 2.81E+03(3.44E+02)\\
    F4 & 2.56E+03(7.32E+02) & 3.63E+04(2.00E+04) & 6.91E+04(5.47E+04) & 1.92E+03(6.46E+00) & 1.95E+03(2.05E+01) & \textbf{1.91E+03(2.35E+00)} & 2.21E+05(9.50E+04) & 1.93E+03(6.64E+00) & 1.92E+03(2.40E+00) & 1.92E+03(1.87E+00) & 1.92E+03(2.23E+00) & \textbf{1.91E+03(3.36E-01)}\\
    F5 & 3.20E+06(2.69E+06) & 1.04E+07(4.53E+06) & 2.07E+07(1.10E+07) & 3.00E+06(4.55E+06) & 3.52E+06(1.90E+06) & 1.40E+06(5.38E+05) & 2.91E+07(1.28E+07) & 1.10E+06(6.37E+05) & 3.36E+06(2.57E+06) & 1.34E+06(1.07E+06) & 1.30E+06(1.01E+06) & \textbf{2.66E+05(9.95E+04)}\\
\bottomrule
    \end{tabular}%
  }
\end{table*}

\begin{table*}[htbp]
  \centering
  \caption{Results on evolutionary continuous optimization: CEC2022 test suite with $D=10$.}
  \label{tab:ec-cec2022-10d}
  \setlength{\tabcolsep}{2pt}
  \renewcommand{\arraystretch}{1.1}
  \arrayrulecolor{ODRule}
  \resizebox{\linewidth}{!}{%
    \begin{tabular}{c|ccccccccccc>{\columncolor{ODOursBody}}c}
    \toprule
     & \cellcolor{ODGroupOne}\textbf{BSPGA} & \cellcolor{ODGroupOne}\textbf{CLPSO} & \cellcolor{ODGroupOne}\textbf{CoDE} & \cellcolor{ODGroupTwo}\textbf{SAHLPSO} & \cellcolor{ODGroupTwo}\textbf{SHADE} & \cellcolor{ODGroupTwo}\textbf{AutoEP} & \cellcolor{ODGroupThree}\textbf{ParEvo1} & \cellcolor{ODGroupThree}\textbf{ParEvo2} & \cellcolor{ODGroupThree}\textbf{EoH} & \cellcolor{ODGroupThree}\textbf{MEoH} & \cellcolor{ODGroupThree}\textbf{ReEvo} & \cellcolor{ODOursHeader}\textbf{OnDesign} \bigstrut\\
    \midrule
    F1 & 3.57E+03(1.66E+03) & 1.01E+04(3.28E+03) & 7.90E+03(2.67E+03) & 3.36E+02(1.01E+02) & 2.29E+03(7.24E+02) & 8.94E+02(1.43E+02) & 1.63E+04(4.84E+03) & 2.24E+03(8.00E+02) & 3.06E+03(8.64E+02) & 3.16E+02(1.62E+01) & 3.18E+02(1.90E+01) & \textbf{3.05E+02(3.10E+00)} \bigstrut[t]\\
    F2 & 4.52E+02(2.42E+01) & 5.66E+02(6.56E+01) & 7.12E+02(1.55E+02) & 4.23E+02(3.13E+01) & 4.41E+02(9.30E+00) & 4.16E+02(6.05E-01) & 6.78E+02(8.26E+01) & 4.20E+02(8.59E+00) & 4.80E+02(2.55E+01) & 4.27E+02(2.91E+01) & 4.23E+02(2.66E+01) & \textbf{4.11E+02(1.40E+00)}\\
    F3 & \textbf{6.00E+02(9.11E-03)} & \textbf{6.00E+02(3.63E-02)} & \textbf{6.00E+02(9.01E-02)} & \textbf{6.00E+02(6.21E-05)} & \textbf{6.00E+02(3.59E-03)} & \textbf{6.00E+02(2.24E-03)} & \textbf{6.00E+02(7.50E-02)} & \textbf{6.00E+02(1.41E-04)} & \textbf{6.00E+02(1.54E-02)} & \textbf{6.00E+02(2.57E-05)} & \textbf{6.00E+02(5.19E-05)} & \textbf{6.00E+02(1.49E-05)}\\
    F4 & 8.60E+02(1.80E+01) & 8.99E+02(2.10E+01) & 9.08E+02(1.85E+01) & 8.66E+02(1.93E+01) & 8.63E+02(1.83E+01) & 8.67E+02(9.67E+00) & 9.18E+02(1.83E+01) & 8.43E+02(9.36E+00) & 8.77E+02(2.11E+01) & 8.50E+02(1.78E+01) & 8.49E+02(1.70E+01) & \textbf{8.34E+02(1.96E+01)}\\
\bottomrule
    \end{tabular}%
  }
\end{table*}

\begin{table*}[htbp]
  \centering
  \caption{Results on evolutionary continuous optimization: CEC2022 test suite with $D=20$.}
  \label{tab:ec-cec2022-20d}
  \setlength{\tabcolsep}{2pt}
  \renewcommand{\arraystretch}{1.1}
  \arrayrulecolor{ODRule}
  \resizebox{\linewidth}{!}{%
    \begin{tabular}{c|ccccccccccc>{\columncolor{ODOursBody}}c}
    \toprule
     & \cellcolor{ODGroupOne}\textbf{BSPGA} & \cellcolor{ODGroupOne}\textbf{CLPSO} & \cellcolor{ODGroupOne}\textbf{CoDE} & \cellcolor{ODGroupTwo}\textbf{SAHLPSO} & \cellcolor{ODGroupTwo}\textbf{SHADE} & \cellcolor{ODGroupTwo}\textbf{AutoEP} & \cellcolor{ODGroupThree}\textbf{ParEvo1} & \cellcolor{ODGroupThree}\textbf{ParEvo2} & \cellcolor{ODGroupThree}\textbf{EoH} & \cellcolor{ODGroupThree}\textbf{MEoH} & \cellcolor{ODGroupThree}\textbf{ReEvo} & \cellcolor{ODOursHeader}\textbf{OnDesign} \bigstrut\\
    \midrule
    F1 & 1.97E+04(3.96E+03) & 4.08E+04(1.01E+04) & 3.63E+04(5.74E+03) & 4.05E+03(5.15E+03) & 1.18E+04(2.38E+03) & 7.24E+03(2.33E+03) & 4.82E+04(6.33E+03) & 1.47E+04(3.67E+03) & 1.72E+03(3.85E+02) & 2.57E+03(1.41E+03) & 3.16E+03(1.20E+03) & \textbf{1.24E+03(3.16E+02)} \bigstrut[t]\\
    F2 & 6.99E+02(6.86E+01) & 1.62E+03(4.19E+02) & 2.35E+03(9.78E+02) & 4.99E+02(2.65E+01) & 5.81E+02(2.27E+01) & 5.12E+02(3.11E+01) & 4.08E+03(7.23E+02) & 5.27E+02(3.02E+01) & 5.21E+02(4.67E+01) & 5.16E+02(6.43E+01) & \textbf{4.97E+02(4.52E+01)} & 5.45E+02(7.18E+01)\\
    F3 & \textbf{6.00E+02(5.37E-02)} & 6.01E+02(1.61E-01) & 6.01E+02(2.27E-01) & \textbf{6.00E+02(1.47E-02)} & \textbf{6.00E+02(1.88E-02)} & \textbf{6.00E+02(1.88E-03)} & 6.01E+02(2.14E-01) & \textbf{6.00E+02(6.43E-03)} & \textbf{6.00E+02(4.11E-03)} & \textbf{6.00E+02(3.16E-03)} & \textbf{6.00E+02(3.95E-03)} & \textbf{6.00E+02(9.70E-04)}\\
    F4 & 1.04E+03(3.39E+01) & 1.14E+03(5.55E+01) & 1.19E+03(4.19E+01) & 1.00E+03(4.95E+01) & 1.06E+03(3.14E+01) & 1.03E+03(2.21E+01) & 1.22E+03(4.84E+01) & 9.95E+02(2.51E+01) & 1.14E+03(6.23E+01) & 1.04E+03(3.85E+01) & 1.04E+03(4.67E+01) & \textbf{9.85E+02(4.00E+01)}\\
\bottomrule
    \end{tabular}%
  }
\end{table*}

\begin{table*}[htbp]
  \centering
  \caption{Results on evolutionary continuous optimization: CEC2026 test suite with $D=30$.}
  \label{tab:ec-cec2026-30d}
  \setlength{\tabcolsep}{2pt}
  \renewcommand{\arraystretch}{1.1}
  \arrayrulecolor{ODRule}
  \resizebox{\linewidth}{!}{%
    \begin{tabular}{c|ccccccccccc>{\columncolor{ODOursBody}}c}
    \toprule
     & \cellcolor{ODGroupOne}\textbf{BSPGA} & \cellcolor{ODGroupOne}\textbf{CLPSO} & \cellcolor{ODGroupOne}\textbf{CoDE} & \cellcolor{ODGroupTwo}\textbf{SAHLPSO} & \cellcolor{ODGroupTwo}\textbf{SHADE} & \cellcolor{ODGroupTwo}\textbf{AutoEP} & \cellcolor{ODGroupThree}\textbf{ParEvo1} & \cellcolor{ODGroupThree}\textbf{ParEvo2} & \cellcolor{ODGroupThree}\textbf{EoH} & \cellcolor{ODGroupThree}\textbf{MEoH} & \cellcolor{ODGroupThree}\textbf{ReEvo} & \cellcolor{ODOursHeader}\textbf{OnDesign} \bigstrut\\
    \midrule
    F1 & 1.41E+10(3.68E+09) & 4.27E+10(5.99E+09) & 5.75E+10(7.85E+09) & 1.25E+09(8.95E+08) & 3.69E+09(6.11E+08) & 2.74E+09(3.17E+08) & 9.15E+10(9.64E+09) & 4.49E+09(7.86E+08) & \textbf{3.96E+08(5.90E+07)} & 1.28E+09(5.29E+08) & 1.30E+09(5.57E+08) & 9.60E+08(1.08E+08) \bigstrut[t]\\
    F2 & 1.64E+05(2.88E+04) & 1.79E+05(2.28E+04) & 1.82E+05(4.43E+04) & 9.50E+04(2.77E+04) & 1.29E+05(2.07E+04) & 1.64E+05(1.10E+04) & 2.15E+05(2.62E+04) & 1.64E+05(3.29E+04) & \textbf{9.36E+04(2.32E+04)} & 1.22E+05(2.63E+04) & 1.34E+05(4.15E+04) & 9.46E+04(2.07E+04)\\
    F3 & 1.84E+03(5.64E+02) & 9.62E+03(3.00E+03) & 1.30E+04(3.16E+03) & 8.44E+02(3.20E+02) & 9.87E+02(1.23E+02) & 7.47E+02(1.25E+02) & 1.79E+04(3.63E+03) & 1.06E+03(1.02E+02) & \textbf{5.67E+02(4.52E+01)} & 7.06E+02(6.54E+01) & 7.62E+02(1.15E+02) & 6.08E+02(3.78E+01)\\
    F4 & 7.70E+02(1.88E+01) & 9.29E+02(2.33E+01) & 9.62E+02(4.23E+01) & 7.43E+02(4.40E+01) & 7.60E+02(2.11E+01) & 7.11E+02(1.68E+01) & 1.05E+03(3.21E+01) & 7.27E+02(1.77E+01) & 9.32E+02(5.49E+01) & 7.65E+02(2.65E+01) & 7.56E+02(3.02E+01) & \textbf{7.10E+02(2.00E+01)}\\
    F5 & 6.43E+02(6.99E+00) & 6.93E+02(9.03E+00) & 6.97E+02(7.73E+00) & 6.60E+02(7.04E+00) & 6.39E+02(2.66E+00) & 6.19E+02(4.82E+00) & 7.00E+02(7.41E+00) & 6.23E+02(2.51E+00) & 6.92E+02(1.50E+01) & 6.27E+02(7.30E+00) & 6.28E+02(1.12E+01) & \textbf{6.17E+02(3.42E+00)}\\
\bottomrule
    \end{tabular}%
  }
\end{table*}

\begin{table*}[htbp]
  \centering
  \caption{Results on evolutionary mixed-variable optimization: MV-BBOB test suite with $D=10$.}
  \label{tab:mix-mv-bbob-10d}
  \setlength{\tabcolsep}{2pt}
  \renewcommand{\arraystretch}{1.1}
  \arrayrulecolor{ODRule}
  \resizebox{\linewidth}{!}{%
    \begin{tabular}{c|cccccccccc>{\columncolor{ODOursBody}}c}
    \toprule
     & \cellcolor{ODGroupOne}\textbf{CoDE} & \cellcolor{ODGroupOne}\textbf{MDE-IHS} & \cellcolor{ODGroupOne}\textbf{DE$_{\mathrm{MV}}$} & \cellcolor{ODGroupOne}\textbf{Sig-DE} & \cellcolor{ODGroupOne}\textbf{IMI-GWO} & \cellcolor{ODGroupTwo}\textbf{SHADE} & \cellcolor{ODGroupTwo}\textbf{PSO$_{\mathrm{MV}}$} & \cellcolor{ODGroupThree}\textbf{EoH} & \cellcolor{ODGroupThree}\textbf{MEoH} & \cellcolor{ODGroupThree}\textbf{ReEvo} & \cellcolor{ODOursHeader}\textbf{OnDesign} \bigstrut\\
    \midrule
    F1 & -8.28E+01(3.16E+00) & -9.21E+01(2.57E-02) & -8.91E+01(8.41E-01) & -8.05E+01(3.17E+00) & -9.20E+01(2.26E-01) & -9.22E+01(8.08E-04) & -8.72E+01(2.05E+00) & -9.22E+01(1.11E-04) & -9.22E+01(3.84E-05) & \textbf{-9.22E+01(8.33E-06)} & -9.22E+01(3.42E-05) \bigstrut[t]\\
    F2 & 5.24E+04(2.41E+04) & 4.84E+03(9.69E+00) & 6.57E+03(9.17E+02) & 2.50E+04(7.61E+03) & 8.88E+03(2.52E+03) & 4.86E+03(2.39E+01) & 8.64E+04(3.84E+04) & 4.91E+03(6.83E+01) & 4.91E+03(5.17E+01) & 4.90E+03(6.74E+01) & \textbf{4.83E+03(6.92E+00)}\\
    F3 & 1.47E+02(1.58E+01) & 7.37E+01(3.85E+00) & 8.95E+01(9.63E+00) & 1.44E+02(1.21E+01) & 6.48E+01(9.82E+00) & 5.95E+01(3.30E+00) & 1.56E+02(1.82E+01) & 5.26E+01(3.16E+00) & 6.24E+01(9.27E+00) & 9.53E+01(1.92E+01) & \textbf{5.14E+01(4.66E+00)}\\
    F4 & 1.84E+02(2.50E+01) & 9.06E+01(4.71E+00) & 1.35E+02(1.09E+01) & 1.62E+02(1.63E+01) & 8.60E+01(1.30E+01) & 7.40E+01(4.49E+00) & 1.56E+02(2.16E+01) & 7.34E+01(5.91E+00) & 7.91E+01(8.20E+00) & 1.10E+02(2.07E+01) & \textbf{7.00E+01(1.16E+01)}\\
    F5 & 6.47E+01(8.15E+00) & 6.22E+01(3.27E+00) & 5.19E+01(4.46E-01) & 6.77E+01(2.91E+00) & 5.33E+01(2.34E+00) & \textbf{5.15E+01(7.29E-15)} & 5.68E+01(3.26E+00) & \textbf{5.15E+01(7.29E-15)} & \textbf{5.15E+01(7.29E-15)} & \textbf{5.15E+01(7.29E-15)} & \textbf{5.15E+01(0.00E+00)}\\
\bottomrule
    \end{tabular}%
  }
\end{table*}

\begin{table*}[htbp]
  \centering
  \caption{Results on evolutionary mixed-variable optimization: MV-BBOB test suite with $D=30$.}
  \label{tab:mix-mv-bbob-30d}
  \setlength{\tabcolsep}{2pt}
  \renewcommand{\arraystretch}{1.1}
  \arrayrulecolor{ODRule}
  \resizebox{\linewidth}{!}{%
    \begin{tabular}{c|cccccccccc>{\columncolor{ODOursBody}}c}
    \toprule
     & \cellcolor{ODGroupOne}\textbf{CoDE} & \cellcolor{ODGroupOne}\textbf{MDE-IHS} & \cellcolor{ODGroupOne}\textbf{DE$_{\mathrm{MV}}$} & \cellcolor{ODGroupOne}\textbf{Sig-DE} & \cellcolor{ODGroupOne}\textbf{IMI-GWO} & \cellcolor{ODGroupTwo}\textbf{SHADE} & \cellcolor{ODGroupTwo}\textbf{PSO$_{\mathrm{MV}}$} & \cellcolor{ODGroupThree}\textbf{EoH} & \cellcolor{ODGroupThree}\textbf{MEoH} & \cellcolor{ODGroupThree}\textbf{ReEvo} & \cellcolor{ODOursHeader}\textbf{OnDesign} \bigstrut\\
    \midrule
    F1 & 4.60E+00(1.86E+01) & -6.73E+01(5.81E+00) & 9.92E+00(1.21E+01) & 5.01E+01(1.08E+01) & -8.24E+01(5.86E+00) & -9.06E+01(3.48E-01) & -3.21E+01(1.16E+01) & -9.04E+01(1.62E-04) & -9.13E+01(5.45E-02) & \textbf{-9.14E+01(5.20E-04)} & -9.08E+01(4.52E-01) \bigstrut[t]\\
    F2 & 1.85E+06(7.06E+05) & 1.24E+05(1.49E+04) & 4.83E+05(7.50E+04) & 1.08E+06(1.88E+05) & 2.25E+05(1.18E+05) & 8.56E+04(4.91E+03) & 1.29E+06(2.96E+05) & 8.96E+04(8.95E+01) & 8.08E+04(7.72E+03) & 7.71E+04(3.29E+03) & \textbf{7.56E+04(6.26E+03)}\\
    F3 & 8.70E+02(1.16E+02) & 4.03E+02(2.82E+01) & 6.64E+02(4.65E+01) & 8.89E+02(9.20E+01) & 3.24E+02(4.80E+01) & 2.34E+02(1.20E+01) & 1.03E+03(1.85E+02) & 3.09E+02(9.87E+00) & 2.59E+02(3.24E+01) & 4.11E+02(6.98E+01) & \textbf{2.21E+02(2.92E+01)}\\
    F4 & 9.22E+02(9.86E+01) & 6.17E+02(4.01E+01) & 1.23E+03(1.07E+02) & 1.25E+03(1.42E+02) & 4.19E+02(6.13E+01) & 3.28E+02(1.89E+01) & 1.17E+03(1.70E+02) & 3.68E+02(8.72E+00) & 3.61E+02(4.51E+01) & 5.40E+02(1.02E+02) & \textbf{3.09E+02(1.50E+01)}\\
    F5 & 2.54E+02(3.98E+01) & 2.49E+02(1.27E+01) & 1.79E+02(2.25E+01) & 2.04E+02(1.07E+01) & 1.38E+02(2.75E+01) & 5.61E+01(2.06E+00) & 1.49E+02(1.72E+01) & 5.65E+01(7.29E-15) & \textbf{5.37E+01(4.49E+00)} & 5.58E+01(5.72E+00) & 6.03E+01(1.15E+01)\\
\bottomrule
    \end{tabular}%
  }
\end{table*}

\begin{table*}[htbp]
  \centering
  \caption{Results on evolutionary mixed-variable optimization: EOPCCV test suite with $D=10$.}
  \label{tab:mix-eopccv-10d}
  \setlength{\tabcolsep}{2pt}
  \renewcommand{\arraystretch}{1.1}
  \arrayrulecolor{ODRule}
  \resizebox{\linewidth}{!}{%
    \begin{tabular}{c|cccccccccc>{\columncolor{ODOursBody}}c}
    \toprule
     & \cellcolor{ODGroupOne}\textbf{CoDE} & \cellcolor{ODGroupOne}\textbf{MDE-IHS} & \cellcolor{ODGroupOne}\textbf{DE$_{\mathrm{MV}}$} & \cellcolor{ODGroupOne}\textbf{Sig-DE} & \cellcolor{ODGroupOne}\textbf{IMI-GWO} & \cellcolor{ODGroupTwo}\textbf{SHADE} & \cellcolor{ODGroupTwo}\textbf{PSO$_{\mathrm{MV}}$} & \cellcolor{ODGroupThree}\textbf{EoH} & \cellcolor{ODGroupThree}\textbf{MEoH} & \cellcolor{ODGroupThree}\textbf{ReEvo} & \cellcolor{ODOursHeader}\textbf{OnDesign} \bigstrut\\
    \midrule
    F1 & 3.40E+03(1.06E+03) & 1.95E+02(7.29E+01) & 8.09E+02(1.97E+02) & 4.61E+02(1.22E+02) & 5.76E+01(4.84E+01) & 8.56E+01(1.38E+01) & 3.96E+03(1.50E+03) & 5.35E+01(9.08E+01) & 1.02E+02(8.83E+01) & 1.02E+02(1.26E+02) & \textbf{7.87E-01(6.53E-01)} \bigstrut[t]\\
    F2 & 8.61E+01(8.15E+00) & 5.27E+01(9.19E+00) & 5.56E+01(8.95E+00) & 5.60E+01(5.97E+00) & 2.55E+01(1.06E+01) & 4.28E+01(5.45E+00) & 7.85E+01(9.72E+00) & 3.22E+01(1.31E+01) & 4.09E+01(1.49E+01) & 5.00E+01(1.67E+01) & \textbf{2.54E+01(7.19E+00)}\\
    F3 & 1.41E+01(8.78E-01) & 5.30E+00(5.65E-01) & 1.13E+01(8.11E-01) & 7.83E+00(5.94E-01) & 3.81E+00(1.32E+00) & 5.04E+00(7.48E-01) & 1.61E+01(1.21E+00) & 3.81E+00(1.47E+00) & 4.71E+00(1.44E+00) & 4.64E+00(1.37E+00) & \textbf{5.64E-01(6.63E-01)}\\
    F4 & 6.23E+01(1.98E+01) & 6.10E+00(4.32E+00) & 7.09E+00(2.69E+00) & 2.74E+01(7.66E+00) & 1.76E-01(1.49E-01) & 2.63E+00(4.60E+00) & 6.17E+01(1.53E+01) & 1.97E+00(4.67E+00) & 6.51E+00(7.69E+00) & 1.37E+01(9.22E+00) & \textbf{3.87E-02(5.84E-02)}\\
    F5 & 4.72E+01(1.59E+01) & 7.13E+00(3.56E+00) & 9.46E+00(3.33E+00) & 2.09E+01(4.44E+00) & 1.17E+00(5.54E-01) & 2.96E+00(2.65E+00) & 3.87E+01(1.24E+01) & 4.03E+00(3.28E+00) & 1.06E+00(2.68E-01) & 5.87E+00(3.26E+00) & \textbf{7.49E-01(1.04E-01)}\\
\bottomrule
    \end{tabular}%
  }
\end{table*}

\begin{table*}[htbp]
  \centering
  \caption{Ablation results on Bayesian optimization on BBOB at $D=10$ and $D=30$. OnDesign denotes the complete framework.}
  \label{tab:bo-ablation-details}
  \setlength{\tabcolsep}{2pt}
  \renewcommand{\arraystretch}{1.1}
  \arrayrulecolor{ODRule}
  \resizebox{\linewidth}{!}{%
    \begin{tabular}{>{\centering\arraybackslash}m{66pt}|ccc>{\columncolor{ODOursBody}}c|ccc>{\columncolor{ODOursBody}}c}
    \toprule
    \cellcolor{ODNeutralHeader} & \multicolumn{4}{c|}{\cellcolor{ODNeutralHeader}{\fontsize{12}{14}\selectfont\bfseries\boldmath$D=10$}} & \multicolumn{4}{c}{\cellcolor{ODNeutralHeader}{\fontsize{12}{14}\selectfont\bfseries\boldmath$D=30$}} \bigstrut\\
    \cline{2-9}
    \cellcolor{ODNeutralHeader}\multirow{-2}[4]{=}{\centering\textbf{Problem}} & \cellcolor{ODGroupThree}\textbf{w/o SA} & \cellcolor{ODGroupThree}\textbf{w/o MAD} & \cellcolor{ODGroupThree}\textbf{w/o SAGE} & \cellcolor{ODOursHeader}\textbf{OnDesign} & \cellcolor{ODGroupThree}\textbf{w/o SA} & \cellcolor{ODGroupThree}\textbf{w/o MAD} & \cellcolor{ODGroupThree}\textbf{w/o SAGE} & \cellcolor{ODOursHeader}\textbf{OnDesign} \bigstrut\\
    \midrule
    F1 & \textbf{-9.26E+01(2.79E-03)} & \textbf{-9.26E+01(6.54E-03)} & \textbf{-9.26E+01(4.29E-03)} & \textbf{-9.26E+01(4.78E-03)} & -9.17E+01(4.56E-01) & -9.17E+01(9.44E-01) & -9.19E+01(4.45E-01) & \textbf{-9.26E+01(1.45E-02)} \bigstrut[t]\\
    F2 & 2.19E+05(4.29E+04) & 3.12E+05(1.54E+05) & 1.09E+05(1.08E+04) & \textbf{9.72E+04(3.16E+04)} & 1.82E+06(9.49E+05) & 2.08E+06(2.44E+05) & 1.87E+06(3.10E+05) & \textbf{1.50E+06(4.62E+05)}\\
    F3 & 2.00E+02(4.90E+01) & \textbf{1.57E+02(3.04E+01)} & 1.60E+02(2.67E+01) & 1.80E+02(1.16E+01) & 5.53E+02(5.64E+01) & 5.89E+02(3.37E+01) & 6.40E+02(8.16E+01) & \textbf{5.44E+02(2.70E+01)}\\
    F4 & 3.42E+02(8.78E+01) & \textbf{2.79E+02(6.23E+01)} & 3.36E+02(7.25E+01) & \textbf{2.79E+02(2.86E+01)} & 1.28E+03(2.55E+02) & 1.35E+03(2.12E+02) & 1.30E+03(1.25E+02) & \textbf{1.15E+03(1.11E+02)}\\
    F5 & \textbf{5.15E+01(0.00E+00)} & \textbf{5.15E+01(0.00E+00)} & \textbf{5.15E+01(0.00E+00)} & \textbf{5.15E+01(0.00E+00)} & \textbf{5.15E+01(0.00E+00)} & \textbf{5.15E+01(0.00E+00)} & \textbf{5.15E+01(0.00E+00)} & \textbf{5.15E+01(0.00E+00)}\\
\bottomrule
    \end{tabular}%
  }
\end{table*}

\end{document}